\documentclass[11pt]{article}
\usepackage{jmlr2e}
\usepackage{amsmath}

\DeclareMathOperator*{\argmax}{arg\,max}
\DeclareMathOperator*{\argmin}{arg\,min}

\newcommand{\g}{\,\vert\,}
\newcommand{\s}{\,;\,}

\newcommand{\kl}[2]{D_{\textsc{kl}}\left( #1 \parallel #2\right)}

\newcommand{\rmp}{\mathrm{p}}
\newcommand{\rmq}{\textrm{q}}

\newcommand{\KL}{D_{\text{KL}}}
\newcommand{\C}{\mathcal{C}}

\newcommand{\Q}{\mathbb{Q}}

\newcommand{\E}{\mathbb{E}}
\renewcommand{\P}{\mathbb{P}}

\newcommand{\M}{\mathbb{M}}

\newcommand{\iid}{\stackrel{iid}{\sim}}

\newcommand\numberthis{\addtocounter{equation}{1}\tag{\theequation}}

\newcommand{\cC}{\mathcal{C}}

\newcommand{\cN}{\mathcal{N}}

\newcommand{\cS}{{\mathcal{S}}}
\newcommand{\cT}{{\mathcal{T}}}

\newcommand{\cZ}{\mathcal{Z}}

\newcommand{\mbx}{\mathbf{x}}

\newcommand{\bl}{\bm{l}}

\newcommand{\bw}{\bm{w}}
\newcommand{\bx}{\bm{x}}
\newcommand{\by}{\bm{y}}

\newcommand{\bB}{\bm{B}}
\newcommand{\bC}{\bm{C}}

\newcommand{\bE}{\bm{E}}
\newcommand{\bF}{\bm{F}}

\renewcommand{\H}{\mathbb{H}}

\newcommand{\bM}{\bm{M}}

\newcommand{\bQ}{\bm{Q}}
\newcommand{\bR}{\bm{R}}
\newcommand{\bS}{\bm{S}}

\newcommand{\bX}{\bm{X}}

\newcommand{\CC}{\mathbb{C}}

\makeatother

\newcommand{\BW}{\text{BW}}

\renewcommand{\P}{\mathbb{P}}

\RequirePackage[ruled,vlined]{algorithm2e}

\usepackage{etoolbox}
\providecommand{\doi}[1]{}
\providecommand{\isbn}[1]{}
\providecommand{\eprint}[1]{}
\providecommand{\url}[1]{}
\AtBeginEnvironment{thebibliography}{
  \renewcommand{\doi}[1]{}
  \renewcommand{\isbn}[1]{}
  \renewcommand{\eprint}[1]{}
  \renewcommand{\url}[1]{}
}

\usepackage{dsfont}\usepackage{array}\usepackage{mathrsfs}\usepackage{comment}\onecolumn

\usepackage{color}\usepackage{babel}

\newcommand{\R}{\mathbb{R}}

\usepackage{enumitem}
\setlist[itemize]{leftmargin=1em}
\setlist[enumerate]{leftmargin=1em}

\usepackage{babel}
\usepackage{bm}

\usepackage{tikz}
\usetikzlibrary{bayesnet}
\usetikzlibrary{trees}
\usepackage{lastpage}

\usepackage{amsmath}
\usepackage{aliascnt}

\newtheorem{assumption}{Assumption}
\newaliascnt{proposition}{theorem}
\newtheorem{proposition}[proposition]{Proposition}
\aliascntresetthe{proposition}
\newaliascnt{lemma}{theorem}
\newtheorem{lemma}[lemma]{Lemma}
\aliascntresetthe{lemma}
\newaliascnt{remark}{theorem}
\newtheorem{remark}[remark]{Remark}
\aliascntresetthe{remark}
\newaliascnt{corollary}{theorem}
\newtheorem{corollary}[corollary]{Corollary}
\aliascntresetthe{corollary}

\usepackage{cleveref}
\crefname{assumption}{Assumption}{Assumptions}
\crefname{proposition}{Proposition}{Propositions}
\Crefname{proposition}{Proposition}{Propositions}
\crefname{theorem}{Theorem}{Theorems}
\Crefname{theorem}{Theorem}{Theorems}
\crefname{corollary}{Corollary}{Corollaries}
\Crefname{corollary}{Corollary}{Corollaries}
\crefname{equation}{eq.}{eqs.}
\Crefname{equation}{Eq.}{Eqs.}
\Crefname{algocf}{Algorithm}{Algorithms}
\crefname{algocf}{algorithm}{algorithms}

\newcounter{problem}

\crefname{problem}{problem}{problems}
\Crefname{problem}{Problem}{Problems}

\usepackage{booktabs}
\usepackage{subcaption}
\numberwithin{equation}{section}

\firstpageno{1}

\begin{document}
\jmlrheading{27}{2026}{1-\pageref{LastPage}}{7/24; Revised
12/25}{1/26}{24-1057}{Bohan Wu and David M. Blei}
\ShortHeadings{Beyond Mean-Field}{Wu and Blei}

\title{Extending Mean-Field Variational Inference via Entropic Regularization: Theory and Computation}

\author{\name Bohan Wu \email bw2766@columbia.edu\\
  \addr Department of Statistics\\
  Columbia University\\
  New York, NY 10027, USA
  \AND
  \name David M. Blei \email  david.blei@columbia.edu\\
  \addr Department of Computer Science and Department of Statistics \\
  Columbia University\\
  New York, NY 10027, USA}

\editor{Pierre Alquier}
\maketitle

\begin{abstract}
  Variational inference (VI) has emerged as a popular method for approximate inference for high-dimensional Bayesian models. In this paper, we propose a novel VI method that extends the naive mean field via entropic regularization, referred to as $\Xi$-variational inference ($\Xi$-VI). $\Xi$-VI has a close connection to the entropic optimal transport problem and benefits from the computationally efficient Sinkhorn algorithm. We show that $\Xi$-variational posteriors effectively recover the true posterior dependency, where the likelihood function is downweighted by a regularization parameter. We analyze the role of dimensionality of the parameter space on the accuracy of $\Xi$-variational approximation and the computational complexity of computing the approximate distribution, providing a rough characterization of the statistical-computational trade-off in $\Xi$-VI, where higher statistical accuracy requires greater computational effort. We also investigate the frequentist properties of $\Xi$-VI and establish results on consistency, asymptotic normality, high-dimensional asymptotics, and algorithmic stability. We provide sufficient criteria for our algorithm to achieve polynomial-time convergence.  Finally, we show the inferential benefits of using $\Xi$-VI over mean-field VI and other competing methods, such as normalizing flow, on simulated and real datasets.

\end{abstract}


\section{Introduction}

Variational inference (VI) is a widely used method for approximate probabilistic inference. VI approximates a difficult-to-compute distribution by positing a family of simpler distributions and minimizing the KL divergence between the family and the target. In Bayesian modeling, the target is a posterior distribution of latent variables given observations $\rmp(\theta \g \mbx)$ and the variational family is of distributions of the latent variables $\rmq(\theta) \in \mathbb{Q}(\Theta)$. VI approximates the posterior with
\begin{align}\label{def:VI}
  \rmq^*(\theta) = \arg \min_{\rmq \in \mathbb{Q}}
  \kl{\rmq(\theta)}{\rmp(\theta \g \mbx)}.
\end{align}

To set up variational inference, we need to select the family of distributions over which to optimize. In many applications, practitioners use the \textit{mean-field} or \textit{fully factorized} family. This is the family of product distributions, where each variable is independent and endowed with its own distributional factor.  Consider a model with $D$ latent variables $\theta = \{\theta_1, \ldots, \theta_D\}$. The corresponding mean-field family is
\begin{align}
  \rmq(\theta) = \prod_{i=1}^{D} \rmq_i(\theta_i),
\end{align}
where each $q_i(\theta_i)$ is the variational factor for $\theta_i$.  Thanks to this simple family, the variational optimization is computationally efficient (to a local optimum). But this efficiency comes at a cost. Mean-field VI suffers in accuracy because it cannot capture posterior dependencies between the elements of $\theta$~\citep{Blei2017}.

In this paper, we develop a new way of doing variational inference.  The idea is to optimize over \textit{all} distributions of the latent variables, i.e., $\rmq \in \mathbb{P}(\Theta)$, but to regularize the variational objective function to encourage simpler distributions that are ``more like the mean-field.'' At one end of the regularization path we effectively optimize over the mean-field family, providing traditional mean-field VI (MFVI). At the other end we optimize over all distributions, providing exact inference (but at prohibitive cost).  Between these extremes, our method smoothly trades off efficiency and accuracy.

In detail, consider a probabilistic model $\rmp(\theta, \mbx) = \rmp(\theta) \rmp(\mbx \g \theta)$ and the goal to approximate the posterior $\rmp(\theta \g \mbx)$. Denote the prior $\pi(\theta) := \rmp(\theta)$ and the log likelihood $\ell(\mbx \s \theta) := \log \rmp(\mbx \g \theta)$. We propose to approximate the posterior by optimizing an \textit{expressivity-regularized variational objective} over the entire space of distributions $\rmq \in \P(\Theta)$.

Take an arbitrary distribution $\rmq(\theta)$ with marginal distributions denoted $\rmq_j(\theta_j)$. We define the \textit{expressivity functional} as the KL divergence between $\rmq(\theta)$ and the product of its marginals:
\begin{align}
  \Xi(\rmq) = \kl{\rmq(\theta)}{\prod_{i = 1}^D \rmq_i(\theta_i)}.
\end{align}
Expressivity measures the dependence among the $D$ latent variables
under $\rmq(\theta)$. In the language of information theory, it is the
multivariate mutual information of
$\theta \sim \rmq(\theta)$~\citep{Cover2006}. Intuitively, it
quantifies how "un-mean-field" the distribution $\rmq$ is. A larger
$\Xi(\rmq)$ indicates that the distribution is further from a
factorized distribution.

We define \textit{$\Xi$-variational inference} ($\Xi$-VI) as an expressivity-regularized optimization problem:
\begin{align}
  \label{eq:xi_vi}
  \rmq_{\lambda}^*(\theta) =
  \arg \max_{\rmq \in \mathbb{P}(\Theta)} \,
  \underbrace{\E_{\rmq} \left[ \ell(\mbx \s \theta) \right]
  - \kl{\rmq(\theta)}{\pi(\theta)}}_{\text{ELBO}(\rmq)}
  - \underbrace{\lambda \, \Xi(\rmq)}_{\text{Expressivity penalty}}.
\end{align}
The first two terms comprise the \textit{evidence lower bound}
(ELBO)~\citep{Jordan1999a,Blei2017}, which is the usual objective
functon for variational inference. When optimized relative to the full
set of distributions of $\theta$, maximizing the ELBO recovers the
exact posterior~\citep{Zellner1988,Knoblauch2022}.  The third term,
however, is a penalty term. It encourages the optimal $\rmq$ to
resemble a product distribution, i.e., a member of the mean-field
family. By varying $\lambda > 0$, we interpolate between the exact
posterior and its mean-field approximation.

We will study the theory and application of \Cref{eq:xi_vi}, which we call $\Xi$-VI (pronounced ``ksee VI''). First we show that we can solve this optimization by iterating between (1) calculating approximate posterior marginals for each variable and (2) solving a problem of \textit{entropic optimal transport} (EOT) with a multi-marginal Sinkhorn algorithm~\citep{Cuturi2013,Lin2022}. We then develop \textit{expressivity-corrected mean field}. It first approximates marginals using traditional VI (e.g., black-box VI~\citep{Ranganath2014} or expectation propagation~\citep{Minka2013}), and then optimizes \Cref{eq:xi_vi} with the Sinkhorn algorithm to model dependencies in the variational approximation.

We prove that $\Xi$-VI gives frequentist guarantees including posterior consistency and a Bernstein-von Mises theorem. Further, we theoretically characterize how to choose the regularization parameter $\lambda$ to balance accuracy and efficiency. Specifically, we characterize the regions of possible $\lambda$ values where the resulting variational approximation is either mean-field or Bayes-optimal.

Empirically, we apply $\Xi$-VI correction to multivariate Gaussians, linear regression with a Laplace prior, and hierarchical Bayesian modeling. The results demonstrate the competitive performance of $\Xi$-VI over other variational inference methods, including mean-field and full-rank ADVI \citep{Kucukelbir2017}, normalizing flow \citep{Rezende2015NFVI}, and Stein variational gradient descent \citep{Liu2016SVGD}. To set the regularization strength $\lambda$, our empirical findings suggest that $\lambda = D$ is a reasonable choice, marking a phase transition between the computationally efficient and statistically accurate regimes.

The rest of the paper is organized as follows.  \Cref{sect-xi-vi} introduces $\Xi$-VI and the $\Xi$-VI correction algorithm. \Cref{sect-examples} provides an empirical study. \Cref{sect-theory} establishes theoretical guarantees for the $\Xi$-variational posterior, including posterior consistency, Bernstein-von Mises theorem, high-dimensional bounds, finite-sample convergence, and algorithmic stability. \Cref{sect-discussion} concludes the paper with a discussion of limitations and further research.

\paragraph{Related Work.} This paper proposes $\Xi$-VI, a new way to
relax the mean-field assumption in variational inference. With this
new algorithm, we also add to two existing areas of VI research:
statistical guarantees and computational guarantees.

Mean-field VI is efficient, but it also has limitations. It poorly approximates posteriors in settings such as multivariate Gaussian models \citep{Blei2017}, state-space models \citep{Wang2004Lack}, piecewise-constant models \citep{ZhangGao2020Rates}, and spike covariance models \citep{Ghorbani2019}. To address these shortcomings, researchers have proposed a variety of solutions, including structural VI \citep{Xing2012,Ranganath2016Hierarchical}, copula-based methods \citep{Tran2015Copula,Tran2017}, linear response corrections \citep{Giordano2018,Raymond2017}, TAP corrections \citep{Opper2001,Fan2021,Celentano2023TAPMF,Celetano2023TAPZ2}, and variational boosting \citep{Miller2017Boosting2,Locatello2018}. Our method makes a contribution to this landscape of research, providing a principled and theoretically supported approach to capture dependencies among latent variables and to manage the statistical-computation tradeoff.

Several lines of recent research examine the statistical properties of VI approximations. This work includes results on asymptotic normality \citep{Hall2011GaussianVI,Hall2011PoissonMixture,Bickel2013,Wang2019}, posterior contraction rates \citep{ZhangGao2020Rates,ZhangZhou2020}, finite-sample bounds \citep{Alquier2016,Alquier2020,Yang2020Alpha}, and performance in high-dimensional settings \citep{Basak2017,Ray2020,Ray2021,Mukherjee2022,Mukherjee2023,Qiu2024}. We contribute to this research by proving frequentist guarantees---posterior consistency and a Bernstein-von Mises theorem---for our proposed class of variational approximations.

Other research examines computational aspects of VI, including convergence rates for coordinate ascent methods \citep{Mukherjee2018SBM,Plummer2020,ZhangZhou2020,Xu2022,bhattacharya2023convergence}, black-box optimization \citep{Kim2023BBVI}, and the trade-off between statistical accuracy and computational complexity \citep{Bhatia2022}. Related work also analyzes VI through gradient flow techniques \citep{Yao2022,Lambert2022,Diao2023,Jiang2023}. Our paper contributes to these computational analyses by explicitly characterizing the trade-off of accuracy for computational simplicity. We also expand the interface between VI and optimal transport, in using entropic optimal transport methods \citep{Cuturi2013,Lin2022} in VI optimization.

\section{$\Xi$-variational inference} \label{sect-xi-vi}

\label{sect-xi}

Again, we consider a general probabilistic model
\begin{align}
  \rmp(\theta, \mbx) = \pi(\theta) \exp\{\ell(\mbx \s \theta)\},
\end{align}
where $\pi(\theta)$ is the prior of the unknown parameter and $\ell(\mbx \s \theta)$ is the log likelihood of the data under $\theta$. While we use the log-likelihood throughout our methods and experiments (Section3), all proposed techniques and theoretical results extend readily to generalized posteriors in which $\ell(\bx \mid \theta)$ is a loss function \citep{Bissiri2016,Knoblauch2019,Miller2021}.  We consider the prior $\pi$ that is a product distribution of the form $\pi(\theta) = \prod_{i = 1}^d \pi_i(\theta_i)$. Our goal is to approximate the posterior $\rmp(\theta \g \mbx)$.

In this section, we formally define $\Xi$-VI and analyze its structure. We reformulate $\Xi$-VI as a nested optimization, separating the problem into an outer optimization over marginals and an inner optimization over their couplings, i.e., a representation of the dependency structure in the variational approximation. We focus on solving the inner optimization, showing how to correct mean-field (factorized) solutions using entropic optimal transport (EOT). We present a computationally practical algorithm and discuss its interpretation.

\subsection{The $\Xi$ variational objective and its nested formulation}

We aim to find the distribution $q_\lambda^*$ that solves the problem in \Cref{eq:xi_vi}. In this problem, $\lambda \geq 0$ is a user-defined regularization parameter, and the optimal $\rmq_\lambda^*(\theta)$ is called the \textit{$\Xi$-variational posterior}. When it is not unique, $\rmq^*_\lambda$ is one of the optimizers of \Cref{eq:xi_vi}.

We make two observations about the $\Xi$-VI problem: (1) When $\lambda = 0$, $\rmq^*_0$ is the exact posterior. When $\lambda = \infty$, $\rmq^*_\infty$ is a mean-field variational posterior. (2) By the standard duality theory, the $\Xi$-VI problem is equivalent to optimizing the standard ELBO over a neighborhood of the mean-field family:
\begin{align*}
  \rmq^*_\lambda =
  \argmax_{\rmq \in \P(\Theta): \Xi(\rmq)  \leq \delta}
  \text{ELBO}\left(\rmq \right).
\end{align*}
The $\Xi$-VI posterior is the distribution over the latent variables
closest to the posterior, but within the neighborhood of expressivity.

We can rewrite $\Xi$-VI as a \textit{nested minimization} problem.  Let $m_i(\theta_i)$ denote a marginal distribution of $\theta_i$ and let $\M(\Theta)$ denote the space of product distributions over $\Theta$,
\begin{align}
  \M(\Theta) = \left\{m(\theta) : m(\theta) = \prod_{i=1}^{D}
  m_i(\theta_i)\right\}.
\end{align}

Given a set of $D$ marginals let $\C\left(m_1, \ldots, m_D\right)$ denote the set of $D$-dimensional joint distributions where $m_j(\theta_j)$ is the $j^{th}$ marginal,
\begin{align}
  \C\left(m_1, \ldots, m_D\right) = \{\rmq(\theta_1, \ldots, \theta_D) : \rmq_j(\theta_j) =
  m_j(\theta_j) \, , \, j = 1, \ldots, D\}.
\end{align}
The set $\C(m_1, \cdots, m_D)$ is called the set of \textit{couplings} over the distributions $\{m_1, \cdots, m_D\}$. As shorthand, we write $\C(m)$ as the set of couplings over the marginal distributions of $m(\theta)$. Note the set $\C(m)$ is convex and closed in the Wasserstein distance \citep{Nutz2021Notes}, and we assume that there exists $\rmq \in \C(m)$ with finite (Boltzmann) entropy.

With these definitions in place, we write $\Xi$-VI as a double minimization problem,
\begin{align}
  \label{VI-EOT-decomposition}
  \min\limits_{m \in
  \M(\Theta)} \min \limits_{\rmq \in \C(m)}
  \E_\rmq[-\ell(\bx; \theta)] + \lambda \KL(\rmq \parallel m) + \KL(\rmq
  \parallel \pi).
\end{align}
The equation follows from expressing the minimization set $\P(\Theta)$ as $\{\rmq \in \C(m), m \in \M(\Theta)\}$, while the objective stays the same.

In \Cref{VI-EOT-decomposition}, the \textit{outer variational problem} minimizes the objective with respect to the space of marginal distributions. Given a set of marginals, the \textit{inner variational problem} minimizes the objective over its set of couplings. Here we will focus on the inner variational problem. Given fixed marginal distributions---such as those produced by mean-field VI---the inner problem finds the optimal coupling that corrects these marginals, i.e., the $\lambda$-regularized optimal dependencies between the variables. We show this problem is solvable using entropic optimal transport tools. Our method improves a given mean-field solution to capture dependencies in the latent variables and better approximate the posterior.

\subsection{Expressivity-corrected mean-field VI}
\label{sect-eot-derivation}

We now derive an algorithm to correct mean-field variational inference using $\Xi$-VI\@. We fix the solution to the outer variational problem with a product distribution $m(\theta)$, obtained from mean-field VI or another approximate method. We then solve \Cref{VI-EOT-decomposition} for the optimal distribution $\rmq(\theta)$ that matches these marginals, i.e., by solving the inner variational problem with respect to the coupling $\rmq \in \mathcal{C}(m)$. As we will see, we can optimize over the set of couplings using tools from entropic optimal transport (EOT)~\citep{Villani2009,Nutz2021Notes}.

\paragraph{Solving the inner variational problem with the Sinkhorn
  algorithm.} We view the inner variational problem as an EOT problem
\citep{Nutz2021Notes}. Again, we fix $m$ and optimize $\rmq$. A simple
calculation (in \Cref{sect-proof1}) shows that
\begin{align}\label{eqn-VI-inner}
  \rmq^*_\lambda(\theta) = \argmin\limits_{\rmq \in \C(m)} \E_\rmq[-\ell(\bx; \theta)] + (\lambda +1) \KL(\rmq \parallel m).
\end{align}
\Cref{thm:MEOT-struct} (\Cref{sect-proof1}) shows the unique solution to \Cref{eqn-VI-inner} has the following form:
\begin{align} \label{eqn-EOT-solution}
      \rmq^*_\lambda(\theta) =\exp\left(\sum_{i = 1}^D \phi_i^*(\theta_i) +
      \frac{1}{\lambda + 1} \ell\left(\bx; \theta\right) \right)
      m(\theta),
\end{align}
where each $\phi_i^*: \Theta_i \to \mathbb{R}$ is a measurable
function, called an \textit{EOT potential}.

The set of EOT potentials $\phi^* := (\phi_1^*, \cdots, \phi_D^*)$ are identifiable up to an additive constant. So we identify the solution by imposing $D-1$ constraints that each one has mean-zero under the marginal,
\begin{align}
  \label{eqn-EOT-constraints}
  \E_{m_1} \phi_1^* (\theta_1) = \cdots = \E_{m_{D-1}}
  \phi_{D-1}^*(\theta_D) = 0.
\end{align}
Let $\bE(m)$ denote the space of $\phi$ such that \Cref{eqn-EOT-constraints} holds. We find the optimal potentials from \Cref{eqn-EOT-solution} by maximizing the Lagrangian dual problem,
\begin{align}
  \label{eqn-VI-inner-dual}
  \phi^* = \arg \max\limits_{\phi \in \bE(m)}\sum_{i = 1}^D \E_{m_i}[\phi_i(\theta_i)]- \E_{m} \left[\exp\left( \sum_{i = 1}^D \phi_i(\theta_i) + \frac{1}{\lambda + 1} \ell\left(\bx; \theta\right)\right) \right].
\end{align}
See \Cref{sect-proof1} for the derivation.

Finally, we solve \Cref{eqn-VI-inner-dual} with a block coordinate
ascent algorithm called the \textit{Sinkhorn algorithm}
\citep{Cuturi2013}. Given the marginals $m^t$ at time $t$, the
Sinkhorn algorithm iteratively updates each $\phi_i$,
\begin{align}
  \label{sinkhorn-problem}
  \phi_i^{t+1} = \argmax_{\phi_i \in L^1_0(m_i^t)} \E_{m_i^t}
  \phi_i(\theta_i) + \sum_{j = 1}^{i-1} \E_{m_j^t}
  \phi_j^{t+1}(\theta_j) + \sum_{j = i+1}^{D} \E_{m_j^t}
  \phi_j^t(\theta_j)-\E_{m^t} \left[\exp(\Lambda^{t+1}(\theta_i,
  \theta_{-i})) \right],
\end{align}
where
\begin{align*}
  \Lambda^{t+1}(\theta_i, \theta_{-i}) := \sum_{j = 1}^{i-1}
  \phi_j^{t+1}(\theta_j) + \sum_{j = i + 1}^D \phi_j^t(\theta_j) +
  \frac{1}{\lambda + 1} \ell(\bx; \theta_i, \theta_{-i}).
\end{align*}
To solve for \Cref{sinkhorn-problem}, the update has an explicit formula:
\begin{align}\label{sinkhorn-update}
  \phi_i^{t+1}(\theta_i) = -\log \E_{m_{-i}^t}
  \exp(\Lambda^{t+1}(\theta_i, \theta_{-i}))  + \eta_i^t , \quad
  \forall \theta_i \in \Theta_i,
\end{align}
where
\begin{align*}
  \eta_i^t =
  \begin{cases}
    \E_{m_i^t} \log \E_{m_{-i}^t} \exp(\Lambda^{t+1}(\theta_i,
    \theta_{-i})) , \quad \mbox{for } i \leq D-1, \\
    0 \quad \mbox{ otherwise. }
    \end{cases}
\end{align*}
The updated EOT potentials satisfy the identifiability constraints~\eqref{eqn-EOT-constraints}. The solution $\rmq^*_\lambda$ calculated with these EOT potentials is a valid probability distribution.

In practice, the expectations required of \Cref{sinkhorn-update} might be difficult to compute. In our algorithm, we approximate them with Monte Carlo.

Specifically, given samples $\theta_i^1, \cdots, \theta_i^N\sim m_i^t$ for $i \in [D]$, we approximate $ \phi_i^{t+1}$ with the Monte Carlo estimate $\hat \phi_i^{t+1}$:
\begin{equation} \label{Monte Carlo step}
    \hat  \phi_i^{t+1}(\theta_i) =  -\log\left(\frac{1}{N^{D-1}} \sum_{J \in [N]^{D-1}}
  \exp(\hat \Lambda^{t+1}(\theta_i, \theta_{-i}^{J})) \right) + \hat \eta_i^T,
\end{equation}
where
\begin{align*}
\hat \Lambda^{t+1}(\theta_i, \theta_{-i}) := \sum_{j = 1}^{i-1}
\hat \phi_j^{t+1}(\theta_j) + \sum_{j = i + 1}^D  \hat \phi_j^t(\theta_j) +
\frac{1}{\lambda + 1} \ell(\bx; \theta_i, \theta_{-i}).
\end{align*}
and
\begin{equation*}
 \hat \eta_i^t =
  \begin{cases}
   \frac{1}{N} \sum_{k = 1}^N\log\left(\frac{1}{N^{D-1}} \sum_{J \in [N]^{D-1}}
  \exp(\hat \Lambda^{t+1}(\theta_i, \theta_{-i}^{J})) \right), \quad \mbox{for } i \leq D-1, \\
    0 \quad \mbox{ otherwise. }
    \end{cases}
\end{equation*}
\paragraph{One-step expressivity-corrected mean-field VI.}
\Cref{alg-approx} implements the entropic correction in a single round of updates. In the first stage, it computes a set of pseudomarginals $\{\tilde m_i\}_{i \in [D]}$, and draws samples from them. In the second stage, it uses those samples in a multi-marginal Sinkhorn algorithm to compute the optimal EOT coupling.

\begin{algorithm}[h]
\SetAlgoLined
\DontPrintSemicolon
\BlankLine
\KwIn{
  \BlankLine
  \begin{itemize}[noitemsep]
  \item Data $\bx$
  \item Likelihood function $\ell(x \s \theta)$
  \item Prior distribution $\pi(\theta)$
  \item Expressivity regularization parameter $\lambda > 0$
  \end{itemize}
} \BlankLine
Compute marginals from a mean-field algorithm
\begin{align*}
  \tilde m_1, \cdots, \tilde m_D \leftarrow \text{mean field inference}(\bx, \ell, \pi).
\end{align*}

Iteratively compute EOT potentials
\begin{align*}
  \tilde \phi_{\lambda, 1}, \cdots, \tilde \phi_{\lambda, D} \leftarrow \text{Sinkhorn}(\bx, \ell, \tilde m_{1:D}, \lambda) && \text{\Cref{sinkhorn-problem,sinkhorn-update}}
\end{align*}

\KwOut{$\tilde \rmq_\lambda(\theta) = \exp\left(\sum_{i=1}^D \tilde \phi_{\lambda, i}(\theta_i) + \frac{1}{\lambda + 1} \ell (\bx; \theta) \right) \prod_{i = 1}^D \tilde m_i(\theta_i)$}

\caption{Expressivity-corrected mean-field variational inference}
\label{alg-approx}
\end{algorithm}

Note that in the first stage, we can use any algorithm for
approximating the posterior marginals, e.g., variational inference
\citep{Blei2017}, expectation propagation (EP) \citep{Minka2013}, or
MCMC \citep{Robert2004}.  Ideally, the first step of
\Cref{alg-approx} would produce accurate estimates of the marginals of the exact posterior.

In practice, we recommend using an approximate method that yields overdispersed marginals, such as EP, because the additional variability often improves downstream coupling approximations.  Intuitively, it produces more variation in the initial samples of $\{\theta_i^1, \cdots, \theta_i^N\}_{i \in [D]}$ for the Monte Carlo step~\eqref{Monte Carlo step}, which leads to better downstream approximations. We demonstrate this empirically in \Cref{sect-examples}.

\Cref{alg-approx} only outputs an approximate solution to the full $\Xi$-VI problem in \Cref{VI-EOT-decomposition}. However, by coupling the marginals $\tilde m_1, \cdots, \tilde m_D$, the final estimate $\tilde \rmq_\lambda(\theta)$ is guaranteed to be \textit{at least as good as} the initial approximation $\tilde m(\theta) := \prod_{j = 1}^D m_j(\theta_j)$ in terms of KL divergence to the exact posterior. The reason is that $\tilde \rmq_\lambda$ maximizes the regularized ELBO in \Cref{eq:xi_vi} over the coupling $\cC(\tilde m)$.

With a large number of variables, \Cref{alg-approx} is computationally challenging because the cost of averaging over $N^{D-1}$ term in step~\eqref{Monte Carlo step} scales exponentially in $D$. For this reason, a stochastic (minibatch) approximation to step~\eqref{Monte Carlo step} is necessary in practice when either $N$ or $D$ is large. In \Cref{sect-implementation}, we outline conditions on the likelihood for the algorithm to be polynomial-time solvable.  Specifically, we provide polynomial-time complexity guarantees in two settings: (i) graphical models with bounded treewidth, and (ii) models in which the likelihood evaluated at the sample points $\{\theta_i^1, \cdots, \theta_i^N\}_{i \in [D]}$ forms a low-rank and sparse tensor. In the first setting, we show that the algorithm converges in time polynomial in the dimension $D$, but exponential in the treewidth and inversely proportional to the regularization parameter $\lambda$.

\subsection{$\Xi$-VI Solution and Connection to Generalized Bayes}

In this section, we discuss the structure of the $\Xi$-VI solution and its connection to existing theories of generalized Bayesian methods \citep{Knoblauch2022}.

As shown in \Cref{eqn-EOT-solution}, the $\Xi$-VI solution consists of three components: (i) a scaled log-likelihood term, (ii) a set of potential functions $\{\phi^*_{\lambda, i}\}_{i=1}^D$, and (iii) a product of marginals $m^*_\lambda(\theta) = \prod_{i=1}^D m^*_{\lambda, i}(\theta_i)$. The regularization parameter $\lambda$ controls the temperature of the likelihood term $\ell(\bx; \theta)$.

Intuitively, $\lambda$ divides a sample size of $n$ between the true posterior and the mean-field solution by a factor of $1/(\lambda+1)$ and $\lambda/(\lambda + 1)$, respectively. It thus quantifies the tradeoff between the likelihood and a product distribution. Higher $\lambda$ allows the variational posterior to be close to the mean field, while lower $\lambda$ allows the solution to better approximate the exact posterior (but at computational cost). When $\lambda = 0$, the likelihood term is untempered---the variational solution is the exact posterior. When $\lambda = \infty$, the solution matches the mean-field variational posterior. The curve of measures $\{\rmq^*_\lambda, \lambda \in \bar \R_+\}$ smoothly interpolates between the mean-field variational posterior and the true posterior.

We can view $\rmq^*_\lambda$ as a nonlinear tilt of a $1/(\lambda+1)$-tempered posterior \citep{Miller2018,Bhattacharya2019}. Posteriors of this form have been studied extensively in statistical learning theory for addressing model misspecification \citep{grunwald2012safe,Grunwald2017}, and in machine learning they have been studied for their predictive performance where they are called “cold posteriors’’ \citep{Aitchison2021,Wenzel2020,mclatchie2025predictive}. Write
$$f_{\lambda, i}^*(\theta_i) := \phi_{\lambda, i}^*(\theta_i) + \log m_\lambda^*(\theta_i)-\log
\pi_i(\theta_i).$$ Then we can represent $\rmq^*_\lambda$ as a nonlinear tilt of the tempered posterior ${\rmq^*_0}^\lambda$,
\begin{align}
  \label{eqn-posterior-tempering}
  \rmq^*_\lambda (\theta)
  \propto
  \exp\left( \sum_{i = 1}^D f_{\lambda, i}^*(\theta_i) \right)
  {\rmq^*_0}^\lambda,
  \quad \text{where} \quad
  {\rmq^*_0}^\lambda(\theta)
  \propto \exp\left( \frac{1}{\lambda + 1}
  \ell (\bx; \theta) \right) \pi(\theta).
\end{align}
\citet{Wainwright2008} shows that the mean-field variational posterior of the quadratic interaction model amounts to a linear tilting of the prior.  \Cref{eqn-posterior-tempering} extends this result, where $f_{\lambda, i}^*(\theta_i)$ is the tilting function.

\section{Examples}\label{sect-examples}

We apply $\Xi$-VI to three statistical models: a multivariate Gaussian model, a high-dimensional Bayesian linear regression, and a hierarchical Bayesian model on the 8-schools data (\cite{Gelman1995}, Section 5.5).
\begin{itemize}
\item In the multivariate Gaussian example, $\Xi$-VI is explicitly solvable. This example illustrates the limitations of mean-field VI \citep{Blei2017}, and demonstrates how $\Xi$-VI improves it.

\item In high-dimensional Bayesian linear regression, mean-field VI produces valid inference under weak covariate interactions \citep{Mukherjee2022, Mukherjee2023}, but fail when the interaction among the covariates is strong \citep{Qiu2024, Celentano2023TAPMF}.  Again, $\Xi$-VI improves on the classical approach.

\item Our analysis of the Bayesian hierarchical model shows how $\Xi$-VI provides more accurate posterior inferences on a real-world dataset.
\end{itemize}

\subsection{Multivariate Gaussian distributions}
\label{example:multivariate-Gaussian}

We first apply $\Xi$-VI to approximating a multivariate Gaussian with the family of all Gaussian distributions.  In this demonstration, no algorithm is needed because $\Xi$-VI admits a closed form solution. In general, it is well known that mean-field VI underestimates the marginal variance of its target posterior~\citep{Blei2017}. Here we show how $\Xi$-VI interpolates between the mean-field and the target posterior, and strictly outperforms mean-field VI in covariance estimation.

Assume that the exact posterior is a multivariate normal, $\rmq_0^*:= N(\mu_0, \Sigma_0)$ with $D$-dimensional mean vector $\mu_0$ and a $D \times D$ full-rank covariance matrix $\Sigma_0$. The $\Xi$-VI formulation is
\begin{equation} \label{eqn-mGaussian-1} \rmq^*_\lambda =
  \argmin_{\rmq = \cN(\mu, \Sigma)} \KL(\rmq \parallel \rmq_0^*) +
  \lambda \Xi(\rmq).
\end{equation}

The next result establishes the self-consistency equations for the $\Xi$-VI solution and establishes upper and lower bounds for the approximating covariance:
\begin{proposition}
  \label{prop:mGaussian}
  Suppose we solve the Gaussian $\Xi$-VI problem~\eqref{eqn-mGaussian-1} with
$\cN(\mu_0, \Sigma_0)$ the exact posterior and  $\lambda > 0$. Then the minimizer $\rmq_\lambda^* = \cN(\mu^*, \Sigma^*)$ where $\mu^*, \Sigma^*$ satisfy the following fixed point equations:
\begin{equation*}
    \mu^* = \mu_0, \quad \left(\Sigma^{*} \right)^{-1}  =  \frac{1}{\lambda + 1} \Lambda_0  +\frac{\lambda}{\lambda + 1} \left(\Sigma_{\text{diag}}^{*} \right)^{-1}.
\end{equation*}
For any matrix norm $\|.\|$, the following bounds
  hold:
\begin{align*}
  \left\| \left[\frac{1}{\lambda + 1} \Lambda_0 + \frac{\lambda}{\lambda + 1} \Sigma_{0, \text{diag}}^{-1}\right]^{-1} -\Sigma_0 \right\| \leq \|\Sigma^*-\Sigma_0 \|  \leq \left\| \left[\frac{1}{\lambda + 1} \Lambda_0 + \frac{\lambda}{\lambda + 1}   \Lambda_{0, \text{diag}}\right]^{-1} -\Sigma_0\right\|.
\end{align*}
\end{proposition}
The proof can be found in \Cref{sect-proof-Gaussian}. Our result shows that $\Lambda^*$ is a convex combination
of the true precision $\Lambda_0$ and the inverse of the variational
marginal variances. As
the regularizer $\lambda \to \infty$, the off-diagonal elements of
$\Lambda^*$ converge to $0$ while the diagonal elements approach those
of $\Lambda_0$.

The weight $\lambda$ controls the approximation error of a variational posterior covariance by combining the marginal precisions of the exact posterior and the mean-field precision with weights \( \frac{1}{\lambda + 1} \) and \( \frac{\lambda}{\lambda + 1} \), respectively. For any $\lambda < \infty$, the $\Xi$-variational posterior offers a tighter approximation than the naive mean field. To see this, we note that $\kl{\rmq_\lambda^*}{\rmq_0^*} \leq \kl{\rmq_\lambda^*}{\rmq_0^*} + \lambda \Xi(\rmq_\lambda^*) \leq \kl{\rmq_\infty^*}{\rmq_0^*}$, where $\Xi(\rmq_\lambda^*) \geq 0$ implies the first inequality and the optimality of $\rmq_\lambda^*$ implies the second inequality.

As a concrete demonstration of these ideas, we study a bivariate Gaussian posterior. Here, the $\Xi$-variational posterior has an analytical solution that can be exactly computed (see \Cref{prop:mGaussian2} in the \Cref{sect-proof-Gaussian}).

\Cref{fig:mGaussian-1}illustrates the interpolation, where the regularization downweights the off-diagonal entries of the precision matrix by a factor of $1/(\lambda+1)$. It shows
$\hat{\rmq}_\lambda$ fitted to a bivariate Gaussian, for different
values of $\lambda$. The left panel shows $\hat
\rmq_\lambda$ as a smooth interpolation between the true posterior and
the mean-field variational posterior. Increasing $\lambda$ smoothly reduces posterior dependence, with a sharp structural change only at $\lambda = \infty$. The right panel paints a
quantitative picture of this interpolation: when
$\lambda \leq 10^{-1}$, the
$\Xi$-variational posterior closely approximates the covariance values
of the exact bivariate Gaussian posterior. For $\lambda \geq
10^{1}$, the covariance is close to zero, which indicates proximity to
the mean-field variational posterior. Both plots suggest that
$\hat
\rmq_\lambda$ undergoes a "phase transition" phenomenon at
$\lambda \in [10^{-1}, 10^1]$.

\begin{figure}[t]
\centering
\includegraphics[width=0.8\textwidth, height=0.4\textwidth]{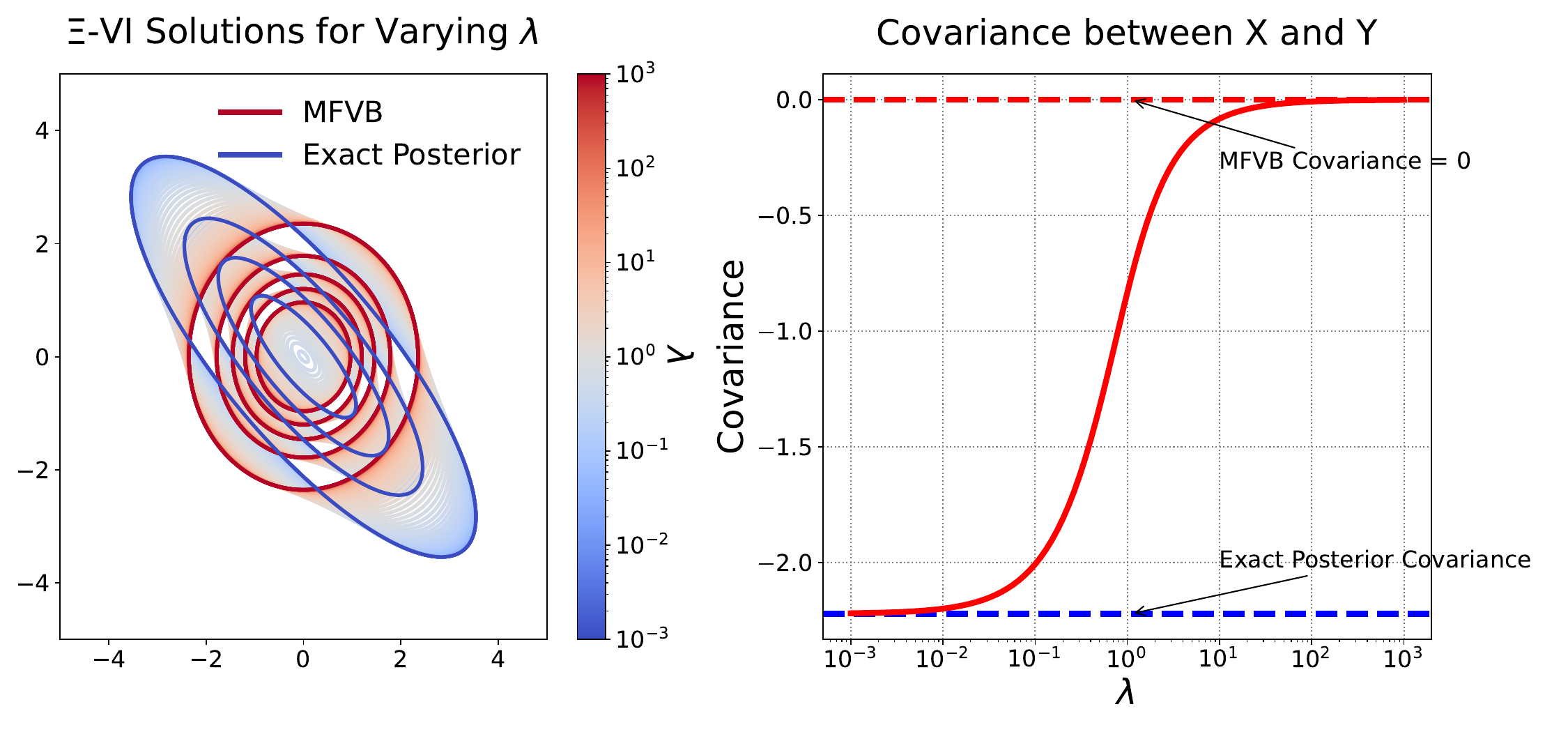}
\caption{$\Xi$-VI solutions for a bivariate Gaussian posterior for varying
  $\lambda$. The left panel illustrates the transition of the
  variational posterior $\rmq^*_\lambda$ from closely
  approximating the exact posterior (at low $\lambda$) to
  resembling the mean-field approximation (at high $\lambda$). The
  right panel shows the covariance between the two normal
  coordinates versus $\lambda$ on a log scale. Note that the
  $\Xi$-variational approximation to the covariance is very
  accurate up to a critical $\lambda$ ($\approx 10^{-1}$), after
  which it degrades rapidly to $0$.}
\label{fig:mGaussian-1}
\end{figure}

\subsection{Bayesian Linear Regression with Laplace Prior}
\label{example:linear-model}

$\Xi$-VI involves a tradeoff between statistical accuracy and computational complexity: as the regularization increases away from the mean-field solution, the quality of VI approximation improves at the cost of increased computational complexity.

To study this,
we consider a a Bayesian linear model with Laplace prior,
\begin{equation}  \label{eqn-linear-experiment}
    \by = \bX \theta + \epsilon, \quad \epsilon \sim N(0, \sigma^2 I_n), \quad \theta_i \sim \text{Laplace}(0, 1).
\end{equation}
The Laplace prior has density $\pi(\theta_i) = \frac{1}{2b}\exp\left(-\frac{|\theta_i|}{b}\right)$.

We simulate a dataset consisting of \(n = 100\) observations and \(d = 12\) features. The true regression coefficients is drawn randomly from a $12$-dimensional standard Gaussian distribution, and $\sigma^2 = 1$. Columns $(1, 2, 3, 8, 9)$ of $\bX$ are generated from a standard Gaussian distribution. Then we set each of features $(4, 5, 6, 11, 12)$ equal to each of features $(1, 2, 3, 8, 9)$ plus a standard Gaussian noise. This setup aims to simulate realistic multicollinearity. Finally, we generate the response $\by$ using model~\eqref{eqn-linear-experiment}. With this simulated data, we calculate an ``exact'' posterior with a long-run MCMC algorithm of 3,000 iterations. The MCMC draws produce an $\hat R$ of below $1.01$ across coefficients \citep{Gelman1995}, which is below the typical threshold of $1.1$ for satisfactory mixing.

Since coupling all $12$ coefficients is computationally expensive, we couple groups of coefficients in the EOT step. We adopt a naive grouping approach where features $(1, 2, 3)$, $(4, 5, 6)$, $(7, 8, 9)$, $(10, 11,12)$ are grouped together. This effectively reduces the computational cost by reducing a twelve-dimensional coupling problem into a four-dimensional one. While it is beneficial to use an informed grouping, any choice of grouping will improve the approximation accuracy of MFVI.

For each dimension, we use \( N = 20 \) points to represent each pseudomarginal. Recent work by \cite{frazier2024exact} shows that generalized Bayes posteriors can be sensitive to the number of pseudo-samples used in stochastic likelihood estimation. In our setting, the posterior $\hat{\rmq}_\lambda$ depends on the EOT potentials computed via the Sinkhorn algorithm (\Cref{eqn-EOT-solution}). To assess this sensitivity, \Cref{tab:xi-vi-sensitivity} in \Cref{sect-additional-simulation} reports the approximation quality of $\Xi$-VI for various $N$ under $\lambda = 10$. Larger $N$ offers only modest accuracy gains while the memory cost scales as $O(N^{D})$ (see \Cref{Monte Carlo step}) and runtime increases sharply. Empirically, $N = 20$ is sufficient to accurately compute $\Xi$-VI under a reasonable computational budget.

With this simulated data, we use \Cref{alg-approx} to compute the $\Xi$-VI approximation. In the first step, we use expectation propagation (EP)  to compute the pseudomarginals. For the analysis, we chose 100 \(\lambda\) values on a logarithmic scale from $10^{-4}$ to $10^6$, and represented the variational posterior for each \(\lambda\) by $2,000$ sample points.
\begin{figure}[t]
\centering
\begin{subfigure}{0.53\textwidth}
\includegraphics[width=\textwidth, height=5cm]{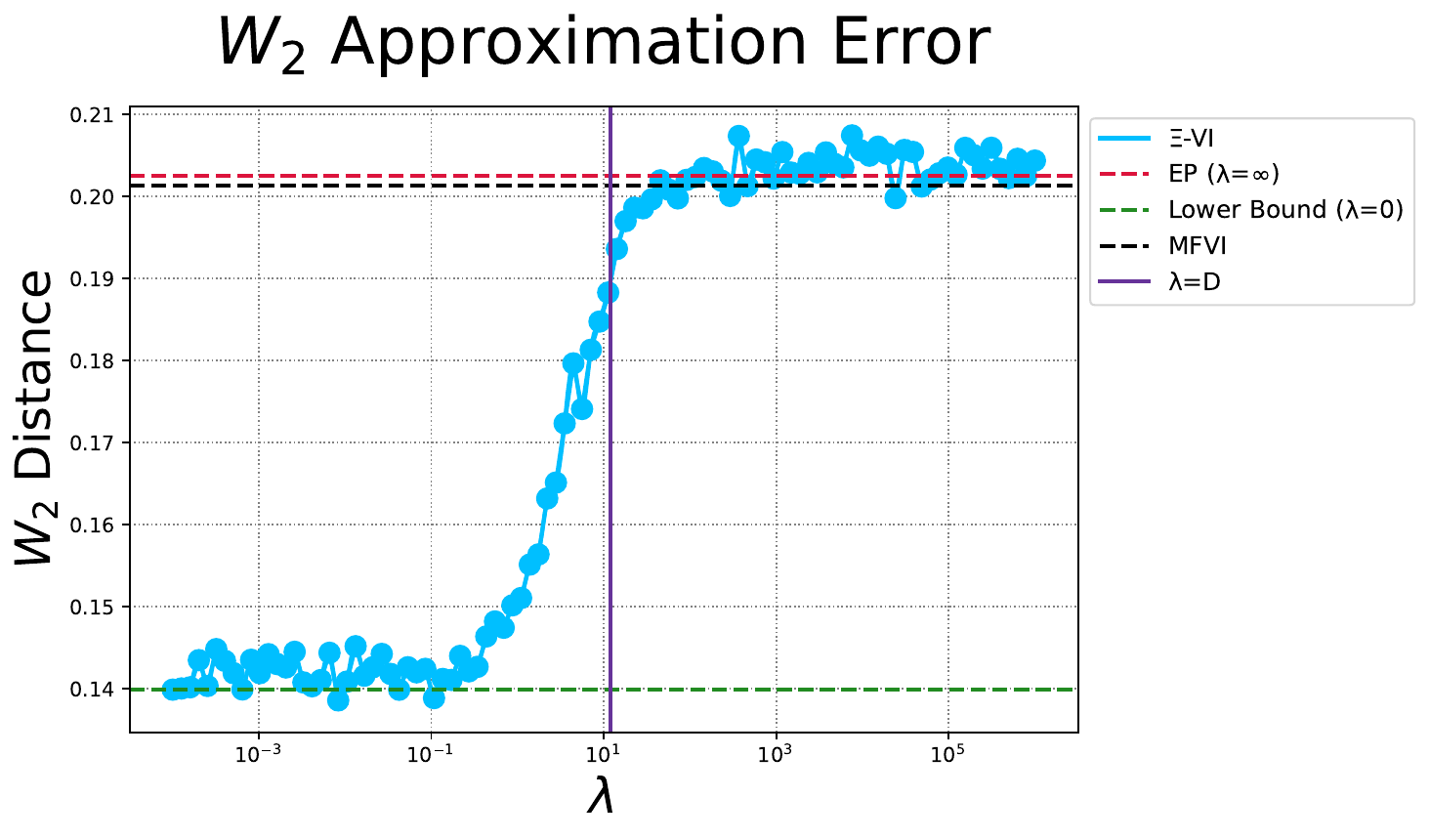}
\label{fig:laplace_W2}
\caption{}
\end{subfigure}
\hfill
\begin{subfigure}{0.46\textwidth}
\centering
\includegraphics[width=\textwidth, height=5cm]{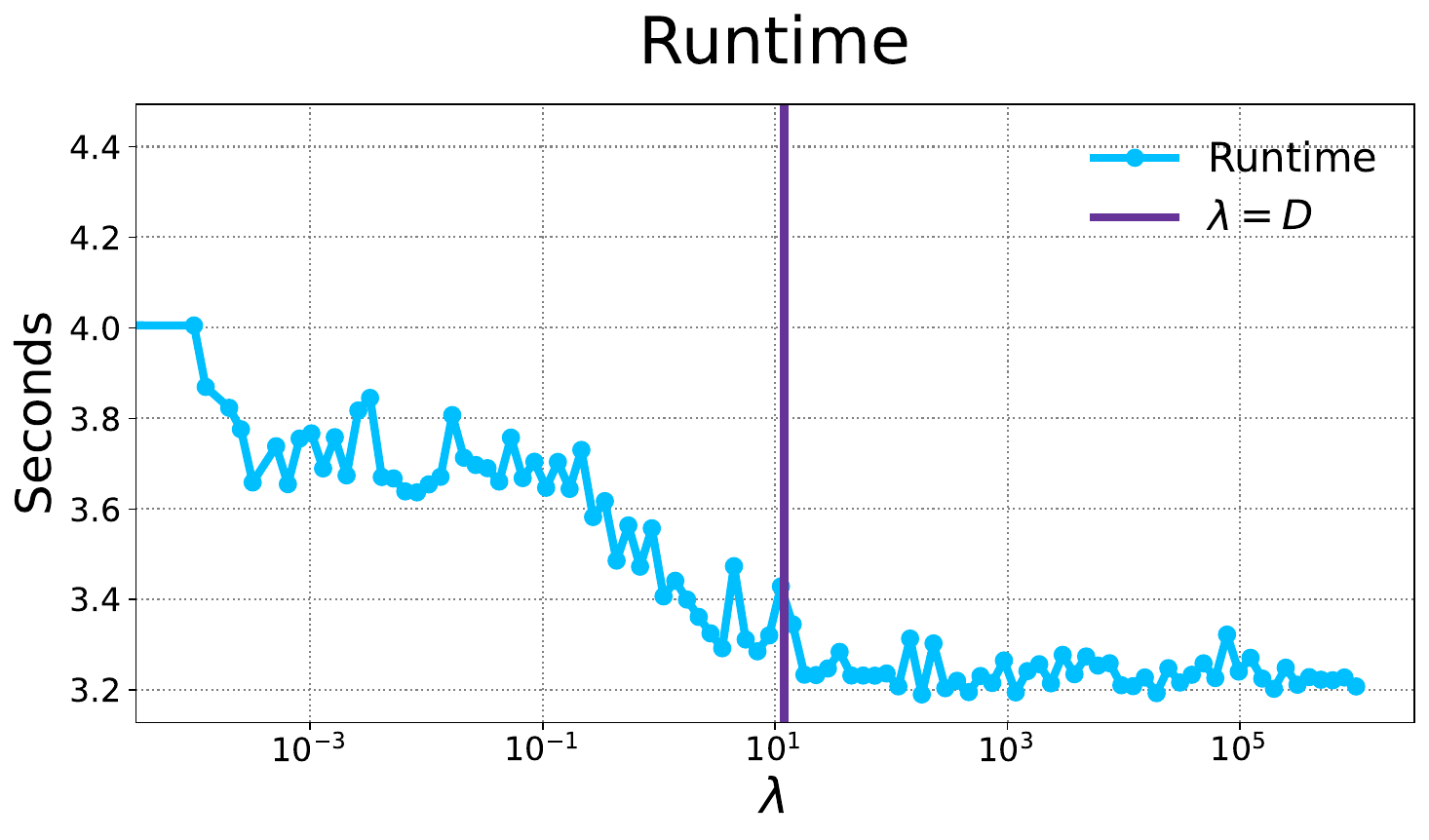}
\label{fig:laplace_runtime}
\caption{}
\end{subfigure}
\caption{\textbf{Left.} accuracy of $\Xi$-VI for Laplace linear regression, measured in $W_2$ across values of $\lambda$. \textbf{Right.} runtime of $\Xi$--VI for Laplace linear regression with a Sinkhorn error of $10^{-5}$, measured in seconds across values of $\lambda$.}
\label{fig:laplace_comparison}
\end{figure}
\Cref{fig:laplace_comparison}(a) shows the approximation errors of $\Xi$-VI as a
function of \(\lambda\), measured using the Wasserstein distance
(\(W_2\)). These distances are computed between the posterior
distributions sampled via MCMC and those obtained from $\Xi$-VI. The
$\Xi$-VI approximation errors are benchmarked against the baseline
errors of EP at \(\lambda = \infty\), mean-field VI, and the
theoretical lower bound at \(\lambda = 0\). A vertical line at
\(\lambda = D\), the number of features, marks an inflection point
where the posterior variational approximation error transitions from
rapidly converging to the EP error (\(\lambda \leq D\)) to relatively stable (\(\lambda > D\)).

\Cref{fig:laplace_comparison}(b) reports the runtime of the approximate
coordinate-ascent algorithm for Laplace linear regression, measured in seconds until
convergence. The $\lambda$ values are displayed on a logarithmic scale to highlight the performance over several orders of magnitude. The runtime decreases sharply for
$\lambda \le D$ and stabilizes once $\lambda > D$. The inflection in both the approximation error and
runtime plots suggests that a regularization strength around \(\lambda = D\) offers an balanced tradeoff between
approximation accuracy and computational complexity.

\subsection{Hierarchical Model}\label{example:Eight-school}

The 8-schools model (\cite{Gelman1995}, Section 5.5) is a classical example of a hierarchical Bayesian model. Each of the 8 schools run a randomized trial to assess the effect of tutoring on a standardized test. Each school provides separate estimates for the mean $y_i$ and standard deviation $\sigma_i$ of their respective treatment effects.

Let $\theta_j$ be the treatment effect in school $j$. We treat the outcomes from each school as independent:
\begin{equation} \label{eqn:school-1}
\begin{aligned}
  y_j | \theta_j &\sim N(\theta_j, \sigma_j^2), \quad \theta_j | \mu, \tau \sim \mathcal{N}(\mu, \tau^2), \quad 1 \leq j \leq 8, \\
  \mu&\sim N(0, 5), \quad \tau \sim \text{halfCauchy}(0, 5).
\end{aligned}
\end{equation}
where \( \mu \) and \( \tau \) are the global parameters common to all schools, $\theta_j$ is a local parameter specific to school $j$. The target of posterior inference are $\{\theta_j\}_{j=1}^{8}$, $\mu$ and $\tau^2$.

To match the $\Xi$-VI formulation in \Cref{sect-eot-derivation}, we define $ z_j := (\theta_j-\mu)/\tau$ and rewrite the model as follows:
\begin{equation}
  \begin{aligned}
    y_j &\mid \mu,  z_j, \tau \sim N(\mu + \tau z_j, \sigma_j^2), \\
    z_j &\sim N(0,1), \quad \mu \sim N(0, 5), \quad \tau \sim \text{halfCauchy}(0, 5).
  \end{aligned}
\end{equation}
This reparameterization transforms the joint prior of $z_j$'s, $\mu$,
and $\tau$ into a product distribution.

We apply \Cref{alg-approx} to solve the $\Xi$-VI problem for this model, expressed as:
\begin{equation} \label{Eight-school-vi}
\rmq_\lambda^* \in  \argmin_{\rmq(z, \mu, \tau)} \E_{\rmq}\left[\sum_{j = 1}^8 \frac{(y_j-\tau z_j-\mu)^2}{2\sigma_j^2} \right] + \KL(\rmq \parallel \pi) + \lambda \Xi(\rmq).
\end{equation}

In the first step, we use automatic differentiation variational inference (ADVI, \citep{Kucukelbir2017,Carpenter2015}) to compute a set of pseudomarginals. In the second step, we use the Sinkhorn algorithm to solve the EOT problem:
\begin{equation} \label{Eight-school-eot}
 \rmq_\lambda^*  =  \argmin_{\rmq(z, \mu, \tau) \in \C(\rmq^*_\infty)} \E_\rmq \left[\sum_{j = 1}^8 \frac{\tau^2 z_j^2 + 2(\mu-y_j)\tau z_j}{2\sigma_j^2}  \right] + (\lambda + 1) \KL(\rmq \parallel \hat m_\lambda).
\end{equation}
The problem~\eqref{Eight-school-eot} is solvable in polynomial time using the Sinkhorn algorithm, as detailed in \Cref{prop:Sinkhorn1} of \Cref{sect-implementation}. Assumption 1 of \Cref{prop:Sinkhorn1} is upheld due to efficient storage of the cost tensor as third-order tensors. Ultimately, we derive the joint distribution $\rmq_\lambda^*(\theta_1, \cdots, \theta_8, \mu, \tau)$ by setting $\theta_j = \mu + \tau z_j$ based on the optimal coupling in \eqref{Eight-school-eot}.

To benchmark the performance of our VI methods, we compute the true posterior using MCMC draws with 4 chains for 1000 tune and 5000 draw iterations.  For each of the VI methods, we represent the approximate posterior with $10,000$ sample points.

$\Xi$-VI captures the dependency among the variables in the posterior.
\Cref{fig:school-joint} compares the strength of association between $\theta_1$ and $\theta_7$ under the true posterior, mean-field variational posterior and $\Xi$-variational posteriors when $\lambda \in \{0, 1,10, 1000\}$. The true posterior shows a strong positive correlation between $\theta_1$ and $\theta_7$, which is effectively captured by $\Xi$-VI at small $\lambda$. As $\lambda$ increases, the correlation decreases to the MFVI level that attains a slope estimate of $0.19$.
\begin{figure}[t]
\centering
\includegraphics[width=0.8\textwidth]{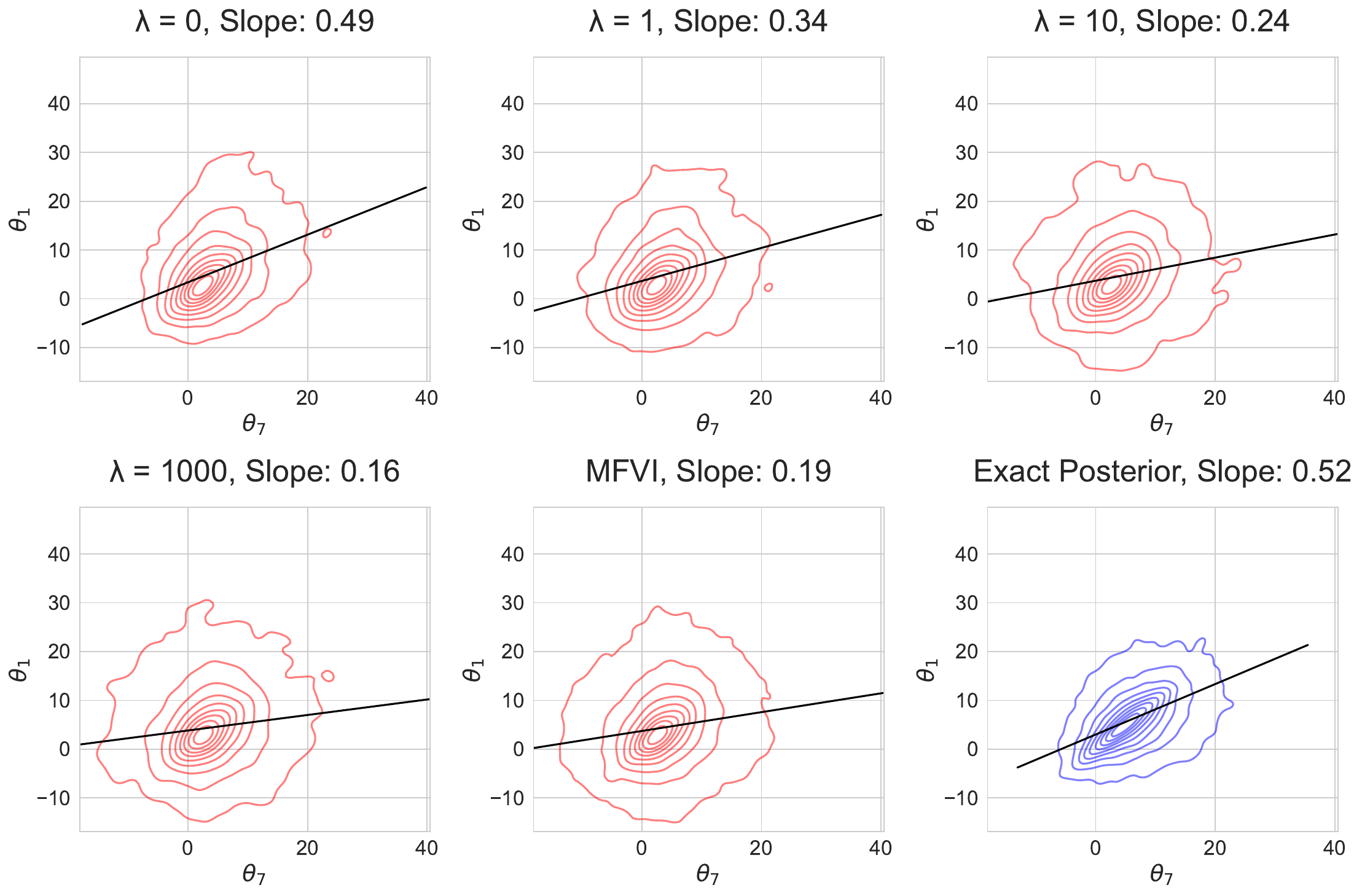}
\caption{
Contour plots for the joint distribution of \(\theta_1\) and $\theta_7$ across various variational approximation of the Eight School model. The subplots compare the exact posterior distribution with $\Xi$-variational posteriors for varying \(\lambda\) values, and the MFVI approximation. A linear regression fitted slope of $\theta_7$ over $\theta_1$ is provided for each subplot. Each subplot includes a linear regression line showing the fitted slope of \(\theta_7\) over \(\theta_1\).}
\label{fig:school-joint}
\end{figure}

$\Xi$-VI excels in inference that involves multiple variables in the posterior. \Cref{fig:max_min_treatment} illustrates credible intervals for maximum and minimum treatment effects across schools, comparing $\Xi$-VI with MFVI, normalizing flow variational inference (NFVI), Stein variational gradient descent (SVGD), and full-rank ADVI. $\Xi$-VI achieves more accurate interval width and coverage accuracy for both  max and min treatment effects compared to other VI methods. Specifically, for the maximum treatment effect, while MFVI, NFVI, and full-rank ADVI produce overly large or small intervals, SVGD results in overly small intervals. In contrast, $\Xi$-VI closely approximates the true 95\% posterior credible interval. For the minimum treatment effect, none of the VI methods precisely capture the true posterior interval. MFVI, NFVI, and full-rank ADVI produce intervals with a downward-shifted center, SVGD offers considerably undersized intervals, and $\Xi$-VI generates reasonably-sized intervals with less downward shift compared to MFVI.
\begin{figure}[t]
\centering
\includegraphics[width=1\textwidth]{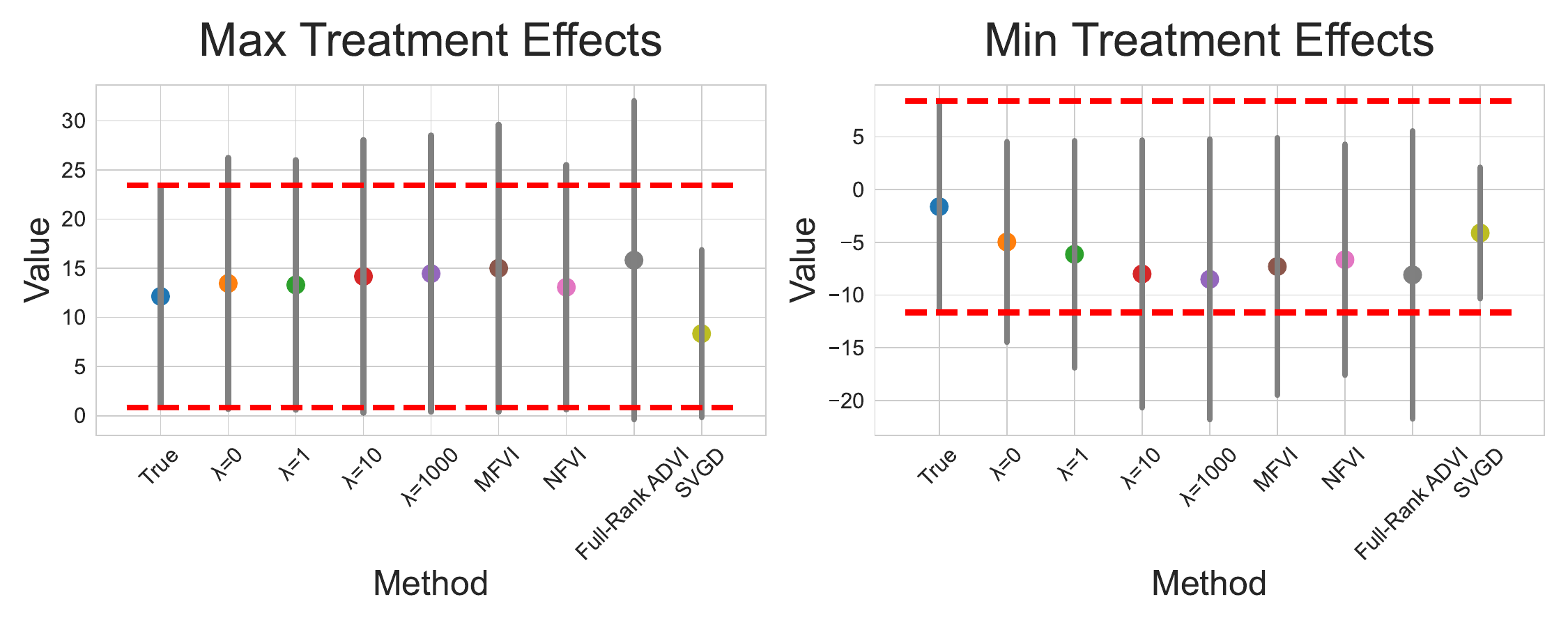}
\caption{Comparison of the 95\% posterior credible intervals for the maximum and minimum treatment effects across schools in the Eight School model. The sequence from left to right includes the exact posterior, $\Xi$-VI with $\lambda \in \{0,1,10,1000\}$, MFVI, normalizing flow (NFVI), full-rank ADVI (Full-rank ADVI), and Stein variational gradient descent (SVGD).}
\label{fig:max_min_treatment}
\end{figure}

Now we show the computation--statistical tradeoff of $\Xi$-VI in the 8-schools model.
We evaluate our procedure on 100 $\lambda$ values logarithmically spaced from $10^{-3}$ to $10^5$.
\Cref{fig:eight-school-comparison}(a) illustrates the approximation errors of the $\Xi$-variational posterior relative to the exact posterior, measured using KL divergence and $W_2$ distance.
These errors are benchmarked against those of MFVI at $\lambda = \infty$ and a theoretical lower bound at $\lambda = 0$.
A vertical line at $\lambda = D = 10$ marks a critical transition: errors remain relatively stable for $\lambda < 1$ and approach MFVI for $\lambda \geq 100$.
Notably, the normalizing-flow VI also performs reasonably well for this model and matches the performance of $\Xi$-VI at $\lambda = 1$ in $W_2$ distance and at $\lambda = D$ in KL distance.
\Cref{fig:eight-school-comparison}(b) shows the runtime of
\Cref{alg-approx} for the 8-schools model, measured in seconds to reduce the Sinkhorn error (\Cref{alg-Sinkhorn1}) below $10^{-5}$. The regularization strength $\lambda$ is plotted on a logarithmic scale. The plot shows a sharp decline right before and right after $\lambda = D$. The phase transition in both plots confirms that a choice of \(\lambda = D\) offers a balance in the tradeoff between computational efficiency and approximation accuracy.  However, a computational-statistical gap exists in this model: while \(\lambda < 1\) yields a closer approximation to the exact posterior, optimal runtime is only achieved for \(\lambda > 10\).

\begin{figure}[t]
\centering
\begin{subfigure}{0.6\textwidth}
\includegraphics[width=\linewidth, height=5cm]{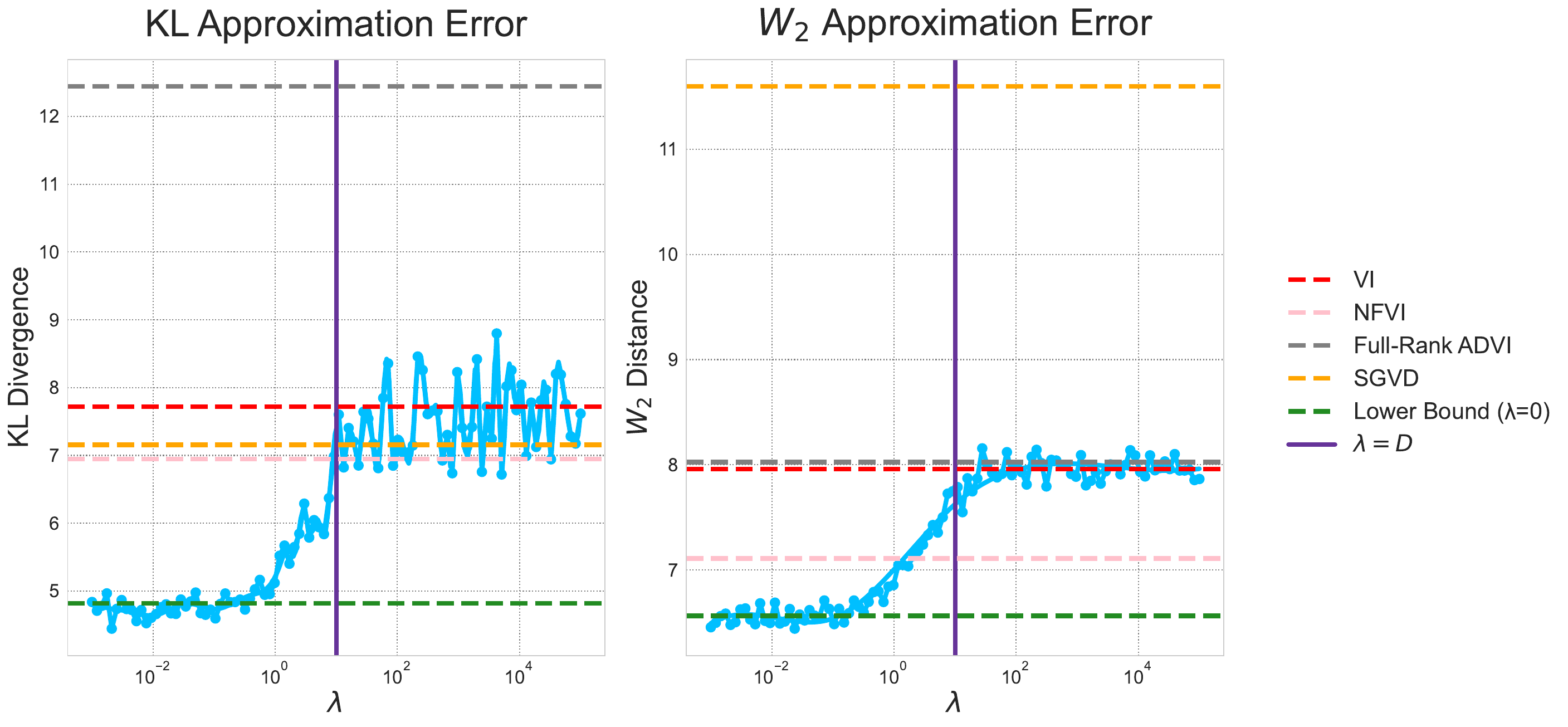}
\caption{}
\label{fig:school_klw2}
\end{subfigure}
\hfill
\begin{subfigure}{0.39\textwidth}
\centering
\includegraphics[width=\linewidth, height=5cm]{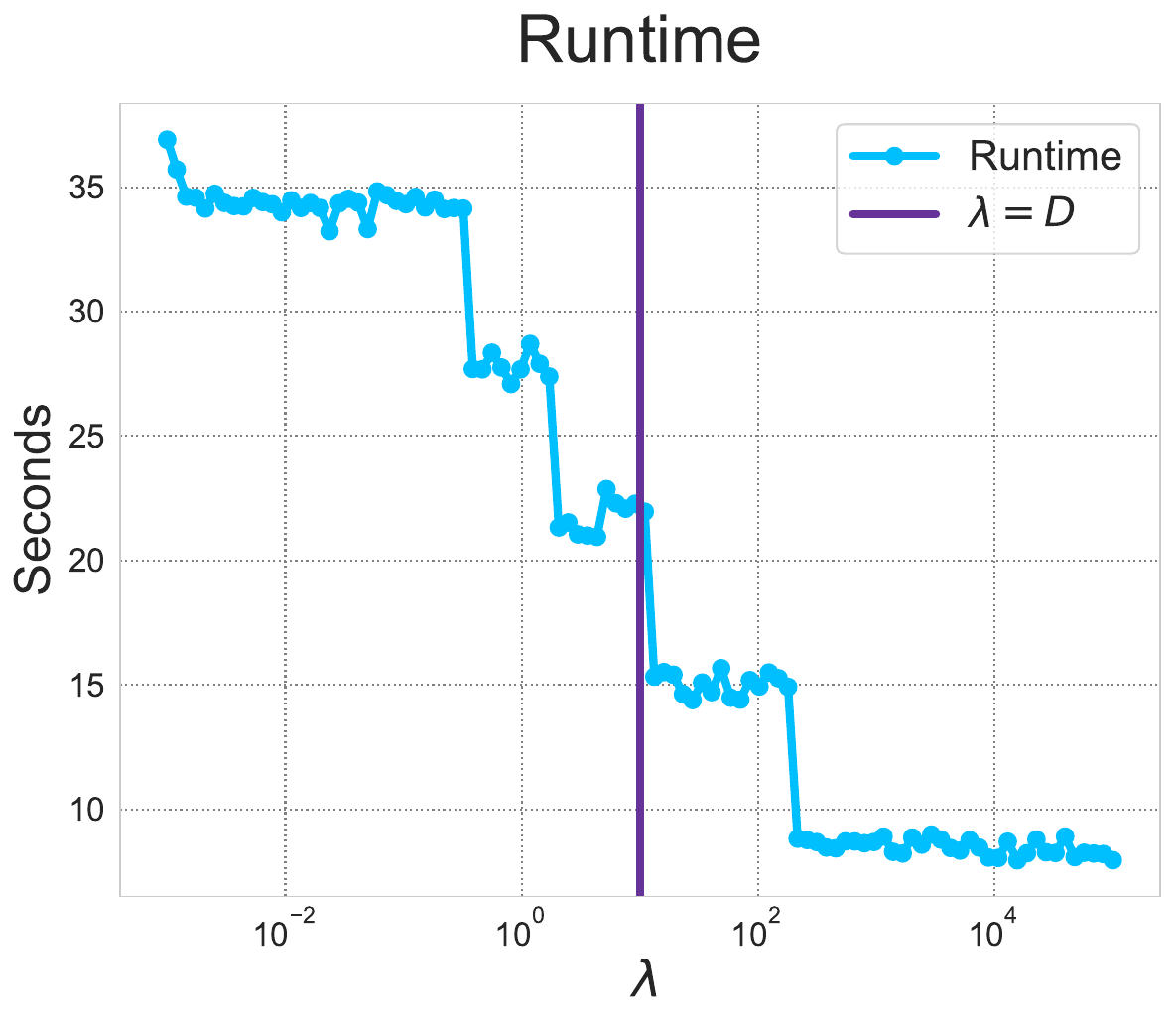}
\caption{}
\label{fig:school_runtime}
\end{subfigure}
\caption{\textbf{Left.} approximation accuracy for the Eight School model of $\Xi$-VI across varying $\lambda$ compared with other VI methods, measured in KL divergence and $W_2$ distance. \textbf{Right.} runtime for the Eight School model as a function of varying $\lambda$  with a Sinkhorn error of $10^{-5}$, measured in seconds.}
\label{fig:eight-school-comparison}
\end{figure}

Finally, note that in the 8-schools model, the MFVI produces overdispersed results after we apply the reparametrization. Generally, we recommend using overdispersed pseudomarginals in \Cref{alg-approx}. The advantage comes from an intuitive understanding of the one-step EOT correction: it seeks overlaps between the pseudomarginals and the exact posterior to effectively capture the dependency information present in the exact posterior. When the pseudomarginals are underdispersed, the one-step EOT correction still leads to underdispersed samples. With overdispersed pseudomarginals, the one-step EOT coupling compensates for the overdispersion by subsampling points from the marginals that reflect the dependency structure of the exact posterior distributions, as seen in \Cref{fig:school-joint}.

\section{Asymptotic Theory}
\label{sect-theory}
In this section, we study the asymptotic theory of $\Xi$-variational posterior $\rmq^*_{\lambda_n}$ in two regimes.
\begin{itemize}
\item In the regime where the parameter dimension $D$ grows with $n$
(\Cref{sect-theory-hd}), we provide asymptotic approximation guarantees for
$\rmq^*_{\lambda_n}$ and identify the range of $\lambda_n$ for which $\Xi$-VI matches either exact posterior inference or a product distribution, under the assumption of a \emph{compact parameter space}.
\item In the finite-dimensional regime (\Cref{sect-theory-finite-D}), where $D$ is held fixed, we establish the full set of frequentist guarantees for $\Xi$-VI, including posterior consistency and a Bernstein--von Mises (BvM) theorem. Under standard regularity conditions and a potentially unbounded parameter space, we show that
$\rmq^*_{\lambda_n}$ converges to a Gaussian distribution as $n \to \infty$.
\end{itemize}
We set some notations for the theory.
For two positive sequences $a_n$ and $b_n$, we write
$a_n \lesssim b_n$ or $a_n = O(b_n)$ or $b_n \gtrsim a_n$ if there
exists a constant \(C > 0\) such that \(a_n \leq C b_n\) for all
\(n\). The constant $C$ does not depend on \(n\). The relation
$a_n \asymp b_n$ holds if both \(a_n \lesssim b_n\) and
\(b_n \lesssim a_n\) are true. We write \(a_n \prec b_n\) or
\(a_n = o(b_n)\) if \(a_n \leq c_n b_n\) for all \(n\), for some
sequence \(c_n\) that converges to zero, \(c_n \to 0\). We write
$a_n \succ b_n$ if $b_n = o(a_n)$.

Let $\P(\Theta)$ denote the set of probability measures on the Euclidean metric space
$(\Theta,\|\cdot\|)$ that admit Lebesgue densities and $\P_p(\Theta) := \{\rmq \in \P(\Theta): \E_\rmq[\| \theta \|^p]<
\infty \}$.  For $p \geq 1$, the ($p^{\text{th}}$)-Wasserstein distance is defined as $W_p(\rmq_0, \rmq_1) :=  (\inf_{\rmq \in \C(\rmq_0, \rmq_1)} \E_\rmq[\|X-Y\|^p])^{1/p}$. The space $(\P_2(\Theta), W_2)$ forms a metric space
\citep{Villani2009}. We use $\BW(\R^D)$ to denote the subspace of
$\P_2(\R^D)$ consisting of Gaussian distributions, known as the
Bures-Wasserstein space \citep{Bhatia2019}. We use $\cC^2(\Theta)$ to denote the space of twice continuously differentiable functions on $\Theta$.
\begin{assumption}[Standing Assumptions]\label{assumption:regularity}
Let $(\Theta,\|\cdot\|)$ be a metric space where $\Theta \subseteq \R^D$ is equipped with the Euclidean norm.
The prior $\pi$ admits a Lebesgue density of the form $\pi(\theta) = \exp(\nu_0(\theta))$ with
$\nu_0 \in \cC^2(\Theta)$, and the exact posterior $\rmq_0^*$ lies in the Wasserstein space
$(\P_2(\Theta), W_2)$.
\end{assumption}

We make explicit the dependence on $n$ of the regularizer $\lambda_n$
and the data $\bx^{(n)}$. Under this setup, the $\Xi$-variational
posterior is given by
\begin{equation}
  \label{def:theory-setup}
  \rmq^*_{\lambda_n} = \argmin_{\rmq \in \P_2(\Theta)} \E_\rmq\left[-\ell(\bx^{(n)}; \theta) \right] + \KL (\rmq \parallel \pi) + \lambda_n \Xi(\rmq).
\end{equation}

\subsection{Asymptotic Approximation Guarantees with Growing Dimension}\label{sect-theory-hd}
To explain the changes in approximation accuracy observed in
\Cref{fig:laplace_comparison}(a) and \Cref{fig:eight-school-comparison}(a), we develop
asymptotic approximation guarantees for the $\Xi$–variational posterior
$\rmq^*_{\lambda_n}$ in the regime where $D_n \to \infty$ as $n \to \infty$. Our primary goal is to
characterize the scaling behavior of the regularization parameter $\lambda_n$ under which
$\rmq^*_{\lambda_n}$ either behaves like a product distribution or like the exact posterior. We begin with a theorem for general models with compact parameter space and then focus on
the high-dimensional linear model.
\paragraph{General Case.}
For any function $f:\Theta \to \R$, let
$\omega_\Theta(f) := \sup_{\theta\in\Theta} f(\theta) - \inf_{\theta\in\Theta} f(\theta)$
denote its oscillation. We now state the main result.

\begin{theorem}\label{thm:theory-hd-1}
Suppose that \Cref{assumption:regularity} holds, and let $\Theta = [-1,1]^D$. Assume further that $\ell(\bx^{(n)};\cdot) \in \cC^2(\Theta)$. Define
\begin{equation}
a := \omega_\Theta(\ell(\bx^{(n)}; \theta)), \quad b_i := \omega_\Theta([\nabla \ell(\bx^{(n)}; \theta)]_i), \quad c_{ij} := \begin{cases}
\omega_\Theta\left([\nabla^2  \ell(\bx^{(n)}; \theta)]_{ij} \right) & \text{for } i = j, \\
\sup_{\theta \in \Theta} \left|[\nabla^2  \ell(\bx^{(n)}; \theta)]_{ij} \right| & \text{for } i \neq j.
\end{cases}
\end{equation}
When $\lambda_n \succ D^{-1/2}\max\left(\sqrt{a \sum_{i = 1}^D c_{ii}}, \sqrt{\sum_{i =1}^D b_i^2}, \sqrt{\sum_{i = 1}^D \sum_{j = 1}^D c_{ij}^2 },  D^{1/2}\right)$, there exists a sequence of product distributions $m_{\lambda_n}^*$ such that, for any $1$-Lipschitz function $\psi: \R \mapsto \R$, as $n \to \infty$, the $\Xi$-VI optimizer $\rmq_{\lambda_n}^*$ satisfies
\begin{equation}\label{eq-thm:theory-hd-1}
\sup_{\bx^{(n)} \in \mathbb{X}^n} \left|\frac{1}{D} \sum_{i = 1}^D \left( \E_{\rmq^*_{\lambda_n}}[\psi(\theta_i)]-\E_{m_{\lambda_n}^*}[\psi(\theta_i)] \right)\right|\overset{P_{\theta_0}}{\to} 0.
\end{equation}

When $\lambda_n \prec D \Xi^{-1}(\rmq^*_0)$, for any $1$-Lipschitz function $\psi: \R \mapsto \R$, as $n \to \infty$,
\begin{equation}\label{eq-thm:theory-hd-2}
\sup_{\bx^{(n)} \in \mathbb{X}^n}\E_{\rmq^*_{\lambda_n}} \left[ \left(\frac{1}{D} \sum_{i = 1}^D \psi(\theta_i)-\frac{1}{D} \sum_{i = 1}^D \E_{\rmq^*_0}[\psi(\theta_i)] \right)^2\right] \overset{P_{\theta_0}}{\to} 0.
\end{equation}
\end{theorem}

The compactness of the parameter space is essential for establishing approximation guarantees for variational
inference in high dimensions. This assumption and the proofs in
\Cref{proof-theory-hd} build on the theory of nonlinear large deviations
\citep{Chatterjee2014,Yan2020}. The choice $\Theta = [-1,1]^D$ and corresponding tools have been
adopted in the theoretical analysis of mean-field
variational inference for Potts models \citep{Basak2017}, linear models
\citep{Mukherjee2022,Mukherjee2023}, and, in a related form, for latent variable
models \citep{zhong2025variational}. Our results also extend to any compact subset $\Theta$ of $\R^D$.

As $\theta \in [-1,1]^D$, all first- and second-order derivatives of $\ell(\bx^{(n)}; \cdot)$ are
uniformly bounded, which controls the oscillation quantities $a$,
$b_i$, and $c_{ij}$. For unbounded parameter spaces, alternative results
are available under stronger shape constraints—when both the likelihood and
the prior are log-concave \citep{Lacker2022}. We do not pursue this direction here to keep the presentation focused.

When $D$ is fixed, the convergence in \Cref{eq-thm:theory-hd-1} and \Cref{eq-thm:theory-hd-2} imply that the
expectations of the averages $\frac{1}{D}\sum_{i=1}^D \psi(\theta_i)$ under the pairs of
posteriors $(\rmq^*_{\lambda_n}, m_{\lambda_n}^*)$ and $(\rmq^*_{\lambda_n}, \rmq^*_0)$
are the same up to an asymptotically negligible error, respectively. When $D$ is fixed and $\Theta$ is compact, the results follow directly whenever the variational posteriors
$\rmq^*_{\lambda_n}$, $m_{\lambda_n}^*$, and $\rmq^*_0$ are \emph{uniformly (weakly)
consistent} at the true parameter $\theta_0$, since weak consistency is equivalent to
convergence in the bounded–Lipschitz metric (see Remark 6.3 of \cite{Ghosal2017}).

\Cref{eq-thm:theory-hd-1} defines a \textit{mean-field regime}, where a product measure matches $\rmq^*_{\lambda_n}$ in any $1$-Lipschitz statistic (first-order statistics). This regime characterizes when $\Xi$-VI can be replaced by MFVI. The critical scaling term in the threshold is $\sqrt{\sum_{i, j} c_{ij}^2}$, as the other terms are typically well controlled. Roughly, the equivalence between $\Xi$-VI and MFVI is determined by comparing $\lambda_n$ to the $(D^{-1/2})$-scaled Frobenius norm of the Fisher information.

\Cref{eq-thm:theory-hd-2} defines a \textit{Bayes optimal regime}, where $\Xi$-VI asymptotically recovers  $1$-Lipschitz statistic of the exact posterior. If the dimension $D = O(1)$ as $n$ increases and the exact posterior achieves consistency, then $\Xi(\rmq^*_0)$ converges to zero, and \Cref{eq-thm:theory-hd-2} holds for any bounded sequence of $\lambda_n$. When $D$ grows with $n$ but at a slow rate (e.g. $D \lesssim n^{-1/3}$), we may still expect a form of posterior consistency to hold the Bayesian optimal regime to contain non-trivial choices of $\lambda_n$.

To match the computational complexity in \Cref{sect-implementation}, we provide sufficient conditions for $ \lambda_n \succ D$ to be in the mean-field regime.
\begin{corollary} \label{cor:theory-hd-1}
In the setting of \Cref{thm:theory-hd-1}, if $a \lesssim D, b_i \lesssim D, c_{ii} \lesssim D$ for $i \in [D]$ and $c_{ij} \lesssim 1$ for $i \neq j$,
then for $\lambda_n \succ D$, there exists a product distribution $m_{\lambda_n}^*$ such that, for any $1$-Lipschitz function $\psi: \R \mapsto \R$, as $n \to \infty$,
\begin{equation}\label{eq-cor:theory-hd-3}
    \sup_{\bx^{(n)} \in \mathbb{X}^n} \left|\frac{1}{D} \sum_{i = 1}^D \left( \E_{\rmq^*_{\lambda_n}}[\psi(\theta_i)]-\E_{m_{\lambda_n}^*}[\psi(\theta_i)] \right)\right| \overset{P_{\theta_0}}{\to} 0.
\end{equation}
\end{corollary}
The result establishes the asymptotic equivalence between $\rmq_{\lambda_n}^*$ and a product measure for $\lambda \succ D$, which provides the following computational insight: when $\lambda_n$ is large, the $\Xi$-variational posterior can be replaced by the mean-field approximation to the posterior. To meet the assumptions of \Cref{cor:theory-hd-1}, it suffices that 1) the gradient and diagonal Hessian of the log-likelihood scale slower than \(D\) entry-wise and 2) the off-diagonal Hessian is uniformly bounded.

Now we consider the example of high-dimensional linear regression models.

\paragraph{High-Dimensional Linear Model.} We observe $\{(x_i, y_i): 1 \leq i \leq n\}$, $y_i \in \R$, $x_i \in \R^D$. Let $\by = [y_1, \cdots, y_n]\in \mathbb{R}^n$ and $\bX = [x_1, \ldots, x_n]^T \in \R^{n \times D}$. We consider a high-dimensional Bayesian linear regression model where both $n, D$ are tending to infinity:
\begin{equation}\label{eqn-linear-1}
    \by = \bX \theta + \epsilon, \quad \epsilon \sim N(0, \sigma^2 I_n), \quad \theta \sim \pi.
\end{equation}
The $\Xi$-VI problem for Bayesian linear model is given by
\begin{equation}\label{eq:theory-gauss-1}
\rmq^*_{\lambda_n} = \arg\min_{\rmq \in \P_2(\Theta)} \E_\rmq\left[\frac{\|\by-\bX \theta\|^2}{2 \sigma^2}\right] + \KL(\rmq \parallel \pi) + \lambda_n \Xi(\rmq).
\end{equation}

For the matrix $\bB := \sigma^{-2}\bX^\top \bX$, let $\bB_{\mathrm{diag}}$ and
$\bB_{\mathrm{off}}$ denote its diagonal and off-diagonal components, respectively.
Our main result characterizes the asymptotic behavior of the optimizer
$\rmq^*_{\lambda_n}$ of \Cref{eq:theory-gauss-1}.
\begin{theorem}\label{thm:theory-gauss-1}
Suppose that \Cref{assumption:regularity} holds, and let $\Theta = [-1,1]^D$. Assume further that there exist constants
$\kappa_1 \ge 0$ and $\kappa_2 > 0$ such that $\nabla^2 \nu_0 \preceq -\kappa_1 I_D$ and $\bX^\top \bX \succeq \kappa_2 I_D$. The following results hold:

When $\lambda_n \succ \sqrt{\text{tr}(\bB_{\text{off}}^2)}$, there exists a sequence of product distributions $m_{\lambda_n}^*$ such that, as $n \to \infty$,
\begin{equation} \label{eq-thm:theory-gauss-1}
 \sup_{\by^{(n)} \in \R^n}   W_2(\rmq^*_{\lambda_n},  m_{\lambda_n}^*) \overset{P_{\theta_0}}{\to} 0.
\end{equation}
When $\lambda_n \succ \sqrt{D^{-1}\text{tr}(\bB_{\text{off}}^2)}$, there exists a sequence of product distributions $m_{\lambda_n}^*$ such that, for any $1$-Lipschitz function $\psi: \R \mapsto \R$, as $n \to \infty$,
\begin{equation}\label{eq-thm:theory-gauss-2}
    \sup_{\by^{(n)} \in \R^n} \E_{\rmq^*_{\lambda_n}} \left[ \left(\frac{1}{D} \sum_{i = 1}^D \psi(\theta_i)-\frac{1}{D} \sum_{i = 1}^D \E_{m_{\lambda_n}^*}[\psi(\theta_i)] \right)^2\right] \overset{P_{\theta_0}}{\to} 0.
\end{equation}
When $\lambda_n \prec  (\kappa_1 + \kappa_2)\left[\text{tr}\left(\text{Cov}_{\rmq^*_0} \left(\bB_{\text{off}} \theta \right) \right)\right]^{-1}$, as $n \to \infty$,
\begin{equation}\label{eq-thm:theory-gauss-3}
  \sup_{\by^{(n)} \in \R^n}  W_2(\rmq^*_{\lambda_n},  \rmq^*_0) \overset{P_{\theta_0}}{\to} 0.
\end{equation}
When $\lambda_n \prec  D(\kappa_1 + \kappa_2)\left[\text{tr}\left(\text{Cov}_{\rmq^*_0} \left(\bB_{\text{off}} \theta \right) \right)\right]^{-1}$, for any $1$-Lipschitz function $\psi: \R \mapsto \R$, as $n \to \infty$,
\begin{equation}\label{eq-thm:theory-gauss-4}
    \sup_{\by^{(n)} \in \R^n}\E_{\rmq^*_{\lambda_n}} \left[ \left(\frac{1}{D} \sum_{i = 1}^D \psi(\theta_i)-\frac{1}{D} \sum_{i = 1}^D \E_{\rmq^*_0}[\psi(\theta_i)] \right)^2\right] \overset{P_{\theta_0}}{\to} 0.
\end{equation}
\end{theorem}
\Cref{eq-thm:theory-gauss-1} and \Cref{eq-thm:theory-gauss-2} form the \textit{mean-field regimes}. When $\lambda_n$ scales faster than $\sqrt{\text{tr}(\bB_{\text{off}}^2)}$, $\rmq^*_{\lambda_n}$ converges in the Wasserstein metric to a product distribution. This means all moment statistics can be asymptotically transported between $\Xi$-VI and MFVI. When $\lambda_n$ scales faster than $\sqrt{\text{tr}(\bB_{\text{off}}^2)/D}$, we can transport any $1$-Lipschitz statistic between $\Xi$-VI and MFVI. As $\lambda_n$ increases, $\rmq^*_{\lambda_n}$ shares more distributional information with $m_{\lambda_n}^*$.

\Cref{eq-thm:theory-gauss-3} and \Cref{eq-thm:theory-gauss-4} define the \textit{Bayes optimal regimes}. When $\lambda_n$ increases more slowly than $(\kappa_1 + \kappa_2)\left[\text{tr}\left(\text{Cov}_{\rmq^*_0}(\bB_{\text{off}} \theta))\right)\right]^{-1}$, $\rmq^*_{\lambda_n}$ converges to the exact posterior in the Wasserstein metric. By relaxing a factor of $D$, $\rmq^*_{\lambda_n}$ achieves asymptotic Bayes optimality for all $1$-Lipschitz statistics. Since $\text{tr}\left(\text{Cov}_{\rmq^*_0}(\bB_{\text{off}} \theta )\right)\leq \|\bB_{\text{off}}\|_2\, \text{tr}\left(\text{Cov}_{\rmq^*_0}( \theta ) \right)$, the Bayes optimal regime is large when $\|\bB_{\text{off}}\|_2$ is small. When $\bB_{\text{off}} = 0$, any choice of $\lambda_n$ falls within the Bayes optimal regime.

The log-concavity assumptions in \Cref{thm:theory-gauss-1} require the log-prior
$\nu_0 \in \cC^2(\Theta)$ to be $\kappa_1$-strongly concave, and the log-likelihood
$\ell(\bx^{(n)};\cdot)$ to be in $\cC^2(\Theta)$ and $\kappa_2/(2\sigma^2)$-strongly concave. The log-concavity parameters $\kappa_1$ and $\kappa_2$ are not required for the
mean-field regimes but appear in the Bayes-optimal regimes. Larger values of
$\kappa_1$ and $\kappa_2$ enlarge the range of $\lambda_n$ for which Bayes-optimal
behavior is satisfied.

The mean-field regime (\Cref{eq-thm:theory-gauss-1}) and the Bayes optimal regime (\Cref{eq-thm:theory-gauss-3}) can both hold for sufficiently large $\lambda_n$. In that case, the $\Xi$-VI solution $\rmq^*_{\lambda_n}$ can be computed efficiently via MFVI while still closely approximating the exact posterior. For instance, in a linear regression setup where $\mathbf{X}^\top \mathbf{X}$ is diagonal, the exact posterior $\rmq_0^*$ is itself a product measure. Then the upper bound in \eqref{eq-thm:theory-gauss-3} is infinite, whereas the lower bound in \eqref{eq-thm:theory-gauss-1} is zero. In this case, any choice of $\lambda_n$ satisfies both \eqref{eq-thm:theory-gauss-1} and \eqref{eq-thm:theory-gauss-3} simultaneously.

For $1$-Lipschitz statistics, \Cref{eq-thm:theory-gauss-2} and \Cref{eq-thm:theory-gauss-4} both hold if
\begin{equation}\label{eqn:overlap-criterion}
   \sqrt{\text{tr}(\bB_{\text{off}}^2)} \|\bB_{\text{off}}\|_2^2 \text{tr}(\text{Cov}_{\rmq^*_0}(\theta)) \prec  D^{3/2}(\kappa_1 + \kappa_2).
\end{equation}
This criterion is satisfied, for example, when $\text{tr}(\bB_{\text{off}}^2) \prec D$, $\|\bB_{\text{diag}}\|_2 \lesssim 1$, and $\text{tr}(\text{Cov}_{\rmq^*_0}(\theta)) \lesssim D$, which recovers the Bayes optimal condition for MFVI \citep{Mukherjee2022}. But our criterion is more flexible: for example, it is also satisfied when $\text{tr}(\bB_{\text{off}}^2) \lesssim D$, $\|\bB-\bB_{\text{diag}}'\|_2 \lesssim 1$, and $\text{tr}(\text{Cov}_{\rmq^*_0}(\theta)) \prec D$.

When no choice of $\lambda_n$ satisfies the overlap criterion, there is a gap between the mean-field and Bayes optimal regimes. Achieving accurate posterior inference thus requires paying an additional computational cost that scales inversely with $\lambda_n$, as discussed in \Cref{sect-implementation}.

Let the eigenvalues of $\bB_{\text{off}}$ be $\eta_D \geq \ldots \geq \eta_1$. Then $\text{tr}(\bB_{\text{off}}^2) = \sum_{i = 1}^D \eta_i^2$, and the mean-field regime \Cref{eq-thm:theory-gauss-1} corresponds to $\lambda_n \succ \sqrt{\sum_{i = 1}^D \eta_i^2}$. To match the complexity bound in \Cref{sect-implementation}, we provide sufficient conditions for $\lambda_n \succ D$ to be in the mean-field regime \Cref{eq-thm:theory-gauss-1}.
\begin{corollary} \label{cor:theory-gauss-1}
Let $\eta_D \ge \cdots \ge \eta_1$ denote the eigenvalues of $\bB_{\mathrm{off}}$.
In the setting of
\Cref{thm:theory-gauss-1}, if $\sum_{i=1}^D \eta_i^2 \lesssim D^2$ as $n \to \infty$, then for any $\lambda_n \succ D$ there exists
$m_{\lambda_n}^* \in \M(\Theta)$ such that $ \sup_{\by^{(n)} \in \R^n}  W_2(\rmq^*_{\lambda_n},  m_{\lambda_n}^*) \overset{P_{\theta_0}}{\to} 0$.
\end{corollary}
Data preprocessing often involves normalizing features to have unit variances. Thus, the requirement that $\sum_{i = 1}^D \eta_i^2 \lesssim D^2$ is often met in practice.

\subsection{Asymptotic Normality in Finite Dimension}
\label{sect-theory-finite-D}
The second part of the theory deals with posterior consistency and asymptotic normality of
$\Xi$-variational posteriors for finite-dimensional models. The asymptotic normality results state that depending on the limit of $\lambda_n$,
$\Xi$-variational posterior converges in the limit to one of three quantities: the mean-field minimizer of a normal
distribution, the normal distribution itself, or a $\Xi$-variational
normal approximation.

We consider the setting where the observations $X_1,\ldots,X_n \iid P_{\theta_0}$, where
$\theta_0$ lies in the interior of a Borel set $\Theta \subseteq \R^D$. Unlike
\Cref{sect-theory-hd}, here $\Theta$ is allowed to be unbounded. We posit the following assumptions.
\begin{assumption}[Prior Mass]
  \label{assumption: prior}
The function $\nu_0$ is bounded in a neighborhood of $\theta_0$. For some $C > 0$, we have
  \begin{equation*}
    \sup_{\|\theta-\theta_0\|_2
      \leq
      C n^{-1/2}} \|\nabla \nu_0(\theta)\|_2\lesssim \sqrt{n}, \quad
    \text{and} \quad \sup_{\|\theta-\theta_0\|_2\leq  C n^{-1/2}}
    \|\nabla^2 \nu_0(\theta)\|_2\lesssim n.
  \end{equation*}
\end{assumption}

\begin{assumption}[Consistent Testability Assumptions]\label{assumption:uct}
For every $\epsilon > 0$, there exists a sequence of tests $\phi_n$ such that
\begin{equation*}
    \int \phi_n(x)\, p_{\theta_0}(x)\, dx \to 0, \quad
    \sup_{\theta: \|\theta-\theta_0\|_2\geq \epsilon} \int (1-\phi_n(x))\, p_\theta(x)\, dx \to 0.
\end{equation*}
\end{assumption}

\begin{assumption}[Local Asymptotic Normality (LAN) Assumptions]\label{assumption:lan}
For every compact set $K \subset \R^D$, there exist random vectors $\Delta_{n, \theta_0}$ bounded in probability and a nonsingular matrix $V_{\theta_0}$ such that
\begin{equation*}
    \sup_{h \in K} \left| \ell(\bx^{(n)}; \theta_0 + \delta_n h) - \ell(\bx^{(n)}; \theta_0)
    - h^T V_{\theta_0} \Delta_{n, \theta_0}
    + \frac{1}{2} h^T V_{\theta_0} h \right|
    \overset{P_{\theta_0}}{\to} 0,
\end{equation*}
where $\delta_n$ is a $D\times D$ diagonal matrix whose entries converge to
$0$ as $n\to\infty$. For $D = 1$, we commonly take $\delta_n = n^{-1/2}$.
\end{assumption}

The first assumption ensures the prior is light-tailed. It is satisfied by, for example, the flat prior or a sub-Gaussian prior.

The second assumption guarantees the existence of a sequence of uniformly consistent tests for $H_0: \theta = \theta_0$ versus $H_1: \|\theta-\theta_0\|_2\geq \epsilon$ based on the data. This condition is satisfied if there exists a consistent sequence of estimators $T_n$ for $\theta$ and we set $\phi_n(x) := I\{T_n - \theta \geq \epsilon/2\}$, or when the Hellinger distance between $\{p_\theta : \|\theta-\theta_0\|_2\geq \epsilon\}$ and $p_{\theta_0}$ is lower bounded by some positive constant $\delta$ \citep{Ghosal2017}.

The third assumption states that the log-likelihood is locally well-approximated (up to a vanishing error) by that of a normal location model centered at $\theta_0$ under an appropriate rescaling. The rescaling sequence $\delta_n$ is exactly the posterior contraction rate. In standard finite-dimensional, correctly specified models, we typically have $\delta_n = n^{-1/2}$ \citep{Ghosal2017}.

In line with \Cref{assumption:lan}, we consider a change of variable:
\begin{equation} \label{change-of-variable}
h := \delta_n^{-1}(\theta - \theta_0 - \delta_n \Delta_{n, \theta_0}), \quad \text{for } \theta \sim \rmq^*_{\lambda_n}.
\end{equation}
Our first result states that under this change of variable, $\Xi$-variational posterior satisfies a Bernstein-von-Mises phenomenon with a phase transition.
\begin{theorem}[Bernstein von-Mises Theorem] \label{thm:AN-1}
Suppose that \Cref{assumption:regularity,assumption: prior,assumption:uct,assumption:lan} hold, and let
$\theta_0$ be an interior point of a Borel set $\Theta$. Let \(\tilde{\rmq}_{\lambda_n}\) be the distribution of the rate-adjusted parameter $h$ defined in \Cref{change-of-variable}. The distribution \(\tilde{\rmq}_{\lambda_n}\) converges in the Wasserstein metric to a normal distribution under the following three regimes:
\begin{enumerate}
    \item If $\lambda_n \to \infty$, then $    W_2\left(\tilde{\rmq}_{\lambda_n}, N(0,  ((V_{\theta_0})_{\text{diag}})^{-1}) \right) \overset{P_{\theta_0}}{\to} 0$, where \((V_{\theta_0})_{\text{diag}}\) is the diagonal submatrix of \(V_{\theta_0}\).
\item If $\lambda_n \to 0$, then $            W_2\left(\tilde{\rmq}_{\lambda_n}, N(0,  V_{\theta_0}^{-1} ) \right) \overset{P_{\theta_0}}{\to} 0$.
\item If \(\lim_{n \to \infty}\lambda_n = \lambda_\infty\) for some $\lambda_\infty \in (0, \infty)$, then
\begin{equation*}
    W_2\left(\tilde{\rmq}_{\lambda_n}, \argmin_{\rmq \in \P_2(\R^D)} \KL (\rmq \parallel N(0,  V_{\theta_0}^{-1} )) +  \lambda_\infty \Xi(\rmq)\right) \overset{P_{\theta_0}}{\to} 0.
    \end{equation*}
\end{enumerate}
\end{theorem}
The result aligns well with intuition. When $\lambda_n$ diverges, $\rmq^*_{\lambda_n}$ converges to the mean-field approximation. When \(\lambda_n\) approaches zero, the constraint set in the Lagrangian dual problem increases to include the true limiting posterior \(N(0, V_{\theta_0}^{-1})\). When \(\lambda_n\) converges to some finite value \(\lambda_\infty\), the $\Xi$-variational posterior converges to the Gaussian limit of the exact posterior. In the regime where $\lim_{n \to \infty} \lambda_n$ does not exists but $\lambda_n = O(1)$, the $\Xi$-variational posterior converges to a a ``biased'' estimate of the true Gaussian posterior \(N(0, V_{\theta_0}^{-1})\) along a subsequence of $\lambda_n$ that converges as $n \to \infty$.

The Bernstein von-Mises Theorem implies the (weak) posterior consistency for  $\rmq^*_{\lambda_n}$.
\begin{corollary} \label{cor:posterior-consistency}
Under the assumptions of \Cref{thm:AN-1}, $\Xi$-variational posterior is consistent in $[P_{\theta_0}]$-probability in the sense that $  W_2(\rmq^*_{\lambda_n}, \delta_{\theta_0}) \overset{P_{\theta_0}}{\to} 0$.
\end{corollary}
The convergence in \Cref{cor:posterior-consistency} is stated in the Wasserstein metric, which is slightly stronger than the typical metric used in posterior consistency results. The convergence in the Wasserstein metric is equivalent to weak convergence plus the convergence of the second moments (Theorem 5.11, \cite{Santambrogio2015}).  Thus, posterior consistency and the Bernstein--von Mises theorem (\Cref{thm:AN-1}) can be framed in terms of the weak convergence and \(L_2\) convergence for the corresponding measures.

\section{Discussion} \label{sect-discussion}

We introduced $\Xi$-VI, a new way of performing variational inference that extends MFVI through entropic regularization. We characterized the asymptotic normality of $\Xi$-variational posteriors in low-dimensional scenarios and analyzed the tradeoff between computational complexity and statistical fidelity in high-dimensional settings. On both simulated and empirical datasets, we demonstrated its advantages over traditional MFVI. Further, our method explicitly connects VI to entropic optimal transport, using the Sinkhorn algorithm to improve the fidelity of variational approximations.

One question prompted by our work is to understand the fundamental limits of high-dimensional Bayesian models. It is known that many high-dimensional problems exhibit a gap between what is statistically achievable (in a minimax sense) and what is achievable by a polynomial-time algorithm, such as sparse PCA \citep{Wang2016spca} and matrix denoising \citep{Chandrasekaran2013}. However, characterizing a \emph{statistical–computational gap} is a newer topic in probabilistic machine learning.

The theoretical results in \Cref{sect-theory} identify distinct asymptotic regimes that correspond to the exact posterior and the mean-field approximation. The transition between these regimes echoes classical phase transitions in spin glass models \citep{Montanari2022}. While our analysis focuses on a regression setting, similar techniques could be extended to models such as the Ising model or the quadratic interaction model. It would also be interesting to explore the connection between $\Xi$-VI and the rich literature on PAC-Bayes and generalized Bayes learning \citep{Alquier2018,Alquier2021expweighting,husain2022adversarial,Wild2022GVI,Wild2023DEL-VI}, and to investigate how recent results in PAC-Bayes theory \citep{Alquier2024} interact with the observed phase transitions—especially given that the target distribution is a general Gibbs measure. Another promising direction is to extend the analysis in \Cref{sect-theory-hd} beyond compact~$\Theta$ to unbounded parameter spaces, for example using the tools of \cite{Lacker2022}, which require strong log-concavity of the exact posterior.

Another challenge is to solve $\Xi$-VI efficiently for models with high-dimensional parameters. The difficulty lies in implementing the multimarginal Sinkhorn algorithm efficiently for large~$D$ in polynomial time. A stochastic approximation of the Sinkhorn step in \Cref{alg-approx} could help mitigate these computational costs. Future work may focus on developing user-friendly algorithmic tools to enable scalable applications of $\Xi$-VI.

While this paper implemented the examples using EP and BBVI for approximating posterior marginals, advanced mean-field methods such as the TAP approach may be preferable in certain contexts, such as spiked covariance models \citep{Fan2021} and high-dimensional Bayesian linear regression \citep{Celentano2023TAPMF}. Exploring $\Xi$-VI combined with the TAP method is a promising avenue for future research, potentially yielding more accurate approximations of the true posterior.

\section*{Acknowledgments}
David M. Blei was supported by NSF IIS-2127869, NSF DMS-2311108, ONR N00014-24-1-2243, and the Simons Foundation. We thank Sumit Mukherjee, Marcel Nutz, Yixin Wang, Eli Weinstein, Kaizheng Wang, and Chenyang Zhong for helpful discussions. We also thank the three anonymous reviewers for their many insightful comments and pointers that helped to improve the paper.
\clearpage
\bibliography{refs.bib}
\clearpage

\appendix
\section{Implementation and Computational Complexity}
\label{sect-implementation}

\Cref{alg-approx} is a two-stage algorithm to approximate the
$\Xi$-variational posterior. A natural question to ask is how
tractable is it in high-dimensional models? In the first stage,
existing methods for the pseudomarginal computation (such as mean-field variational inference) are known run in polynomially with respect to the parameter dimension \citep{Jiang2023}. But the second stage
is a Sinkhorn algorithm, which is not necessarily scalable \citep{Altschuler2023}. In this section, we discuss sufficient conditions for our algorithm to run in polynomial time.
\subsection{The Sinkhorn algorithm}

We use discrete points to represent distributions.
The set of marginals $\{m_i\}_{i \in [D]}$ is a nonnegative
matrix $\bM = (\bM_1,\ldots,\bM_D) \in \R^{N \times D}$, where each
$\bM_i$ contains $N$ design points.
The EOT potentials $\{\phi_i\}_{i \in [D]}$ are represented by a matrix
$\bF = (\bF_1,\ldots,\bF_D) \in \R^{N \times D}$.
The negative loss $-\ell(\bx^{n}; \theta)$ is represented as a cost tensor
$\bC = (C_{i_1,\ldots,i_D}) \in (\R^N)^{\otimes D}$,
and the variational posterior $\rmq$ is a nonnegative tensor
$\bQ = (Q_{i_1,\ldots,i_D}) \in (\R^N)^{\otimes D}$.

Numerically, the EOT problem~\eqref{eqn-VI-inner} can be posed as a linear program
with $DN$ marginal constraints and $N^{D}$ decision variables:
\begin{equation*}
  \min_{\bQ > 0,\; r_i(\bQ)= r_i(\bM)}
  \langle\bC, \bQ \rangle
  + (\lambda + 1)\,\langle \log \bQ - \log \bM,\; \bQ \rangle.
\end{equation*}
The Sinkhorn algorithm~\eqref{alg-Sinkhorn1} returns the potentials
$\bF^*$ by iteratively performing log-sum-exp updates between
$\bF_1,\ldots,\bF_D$.

\begin{algorithm}[t]
\SetAlgoLined
\KwIn{Cost tensor $\bC$, marginals $\bM$, tolerance $\epsilon$, regularization parameter $\lambda$.}
\textbf{Initialize: } $\bF_i = -\frac{1}{D(\lambda + 1)}\mathbf{1}- \log \bM_i$ for all $i \in [D]$. \\
\While{$\bE > \epsilon$}{
Select the greedy coordinate $j = \argmin_{i \in [D]} \|r_i(\bQ)-\bM_i\|_1$. \\
\For{$i \in [D]$}{
Update
\[
\bF_i =
\begin{cases}
  \bF_i - \log\!\big(r_i(\bQ)\big) + \log(\bM_i), & i=j, \\
  \bF_i, & \text{otherwise}.
\end{cases}
\]
Compute
\[
r_i(\bQ)=
\frac{\sum_{1 \le k_\ell \le N,\;\ell\ne i}
  \exp\!\Big[\sum_{\ell=1}^D \bF_{\ell k_\ell}
  - \frac{1}{\lambda + 1}\bC_{k_1\cdots k_D}\Big]\;
  \bM_{k_1\cdots k_D}}
{\sum_{1 \le k_\ell \le N,\;\forall \ell}
  \exp\!\Big[\sum_{\ell=1}^D \bF_{\ell k_\ell}
  - \frac{1}{\lambda + 1}\bC_{k_1\cdots k_D}\Big]\;
  \bM_{k_1\cdots k_D}}.
\]
}
Set $\bE = \sum_{i=1}^D \|r_i(\bQ)-\bM_i\|_1$.}
\KwOut{An $N \times D$ matrix $\bF$.}
\caption{(Multimarginal) Sinkhorn Algorithm}
\label{alg-Sinkhorn1}
\end{algorithm}

\subsection{The complexity of the Sinkhorn algorithm}

Assuming the cost tensor is uniformly bounded,
the best-known complexity bound for the multimarginal Sinkhorn algorithm is
$O(D^3 N^{D} (\lambda + 1)^{-2})$ \citep{Lin2022}.
Unfortunately, the exponential dependence on $D$ cannot be improved in general
\citep{Kroshnin2019}.
Polynomial-time solvability requires structural assumptions on~$\bC$.
\citet{Altschuler2021} show that if the cost tensor has bounded treewidth or is the sum of a low-rank tensor and a sparse tensor,
then multimarginal EOT is solvable in $\mathrm{poly}(N, D)$ time.

\begin{proposition}[\cite{Altschuler2023}]
\label{prop:Sinkhorn1}
Let $\bC \in (\R^N)^{\otimes D}$ satisfy one of the following:
\begin{enumerate}
  \item $\bC$ has graphical structure with constant junction tree width $\omega$; or
  \item $\bC = \bR + \bS$ where $\bR$ has constant multilinear rank
        and $\bS$ has $\mathrm{poly}(N, D)$ sparsity.
\end{enumerate}
Then for any $\lambda \ge 0$, Algorithm~\ref{alg-Sinkhorn1} terminates in
$\mathrm{poly}(N, D, \bC_{\max}/\epsilon, (\lambda+1)^{-1})$ time.
\end{proposition}

\begin{remark}
The bounded-treewidth assumption ensures polynomial-time solvability of the junction-tree algorithm \citep{Wainwright2008}.
Examples include state-space models, topic models, and sparse linear regression.
\end{remark}

\begin{remark}
The ``low-rank plus sparse'' structural assumption is less commonly used in Bayesian inference.
Roughly, it requires the posterior to be well approximated by a mixture of product distributions up to a $\mathrm{poly}(N, D)$-sparse error term.
\end{remark}

For general graphs $G$ with bounded treewidth, \citet{Fan2022} propose implementing the Sinkhorn algorithm using junction-tree message passing, yielding the following complexity bound:
\begin{corollary}
\label{cor:Sinkhorn-1}
Let $\bC \in (\R^N)^{\otimes D}$ have constant treewidth $\omega$.
Then Algorithm~\ref{alg-Sinkhorn1} implemented with the junction tree method
converges in
\[
O\!\left(D^3 N^{\omega+1} (\lambda + 1)^{-1} \epsilon^{-1}\right)
\]
iterations.
\end{corollary}

This result is adapted from Theorem~4 of \citet{Fan2022}, which shows that computational cost decreases as $\lambda$ increases.
Polynomial dependence on $D$ requires only that the treewidth $\omega(G)$ grows slower than $\log D$, so multimarginal EOT may still be feasible for ``locally tree-like’’ graphs.

Counterintuitively, when $\lambda$ grows faster than $D$, the computational complexity decreases as $D$ increases.
However, as \Cref{cor:theory-hd-1,cor:theory-gauss-1} show in \Cref{sect-theory}, in this regime the variational posterior collapses to the naive mean-field approximation.

\begin{remark}
There are well-known models that violate the assumptions in \Cref{prop:Sinkhorn1}.
For example, an Ising model on a complete $D \times D$ graph has treewidth $D$, and its cost tensor is neither low rank nor sparse.
Multimarginal Sinkhorn becomes NP-hard in this setting \citep{Altschuler2021}.
A practical workaround is to group variables and couple only the groupwise marginals.
For instance, in an Ising model with $100$ variables, one may group the first $50$ and last $50$ variables, yielding two $50$-dimensional marginals.
This still yields a strict improvement over MFVI.
Moreover, when the number of groups is fixed, the Sinkhorn complexity becomes polynomial because it is insensitive to the dimensionality within each group \citep{Altschuler2017}.
\end{remark}
\subsection{Finite-Sample Convergence} \label{sect-theory-mf-conv}
Here we analyze convergence properties of the $\Xi$-variational posterior $\rmq^*_\lambda$ when $\lambda$ tends to $0$ or $\infty$, while keeping $n$ fixed.  Understanding this setting justifies the stability of the algorithmic output after we replace the marginals of $\rmq^*_\lambda$ with a set of pseudomarginals in \Cref{alg-approx}. Moreover, the convergence results of the $\Xi$-variational posterior for both large and small $\lambda$ values are useful from a classical Bayesian perspective that treats the observed data as known \citep{Berger2013}.

We show that $\Xi$-variational posterior converges to the mean-field variational posterior as $\lambda$ tends to infinity, and converges to the exact posterior as $\lambda$ tends to zero. Then we establish a stability property for $\rmq^*_\lambda$ when we replace its marginals with another set of marginals, which helps justify \Cref{alg-approx}.

Let us define a cost function $C_{\lambda}$ over $\lambda \in \bar \R_+$ as $\C_\lambda := \max\limits_{\rmq \in \P_2(\Theta)}  \text{ELBO}(\rmq)-\lambda \Xi(\rmq)$.

\paragraph{Limits as $\lambda \to \infty$ or $\lambda \to 0$.}  We start with the convergence of $\rmq^*_\lambda$ and $\C_\lambda$ as $\lambda$ tends to infinity.

\begin{theorem}\label{thm:mf-convergence}
Assume that $\KL(\rmq \parallel \rmq^*_0) < \infty$ for some $\rmq \in \M(\Theta)$.  For each $\lambda \in \bar \R_+$,  define the set $\Q_\lambda$ as the set of minimizers for the functional $\rmq \mapsto \KL(\rmq \parallel \rmq^*_0) + \lambda \Xi(\rmq)$ in $\P_2(\Theta)$.  If $\rmq^*_\infty \in \Q_\infty$, there exists a sequence $\rmq^*_\lambda \in \Q_\lambda$ such that $\lim_{\lambda \to \infty} W_2(\rmq^*_\infty, \rmq^*_\lambda) = 0$. Furthermore, the $\Xi$-VI cost converges to the mean-field ELBO, i.e. $\lim_{\lambda \to \infty}\left|\C_\lambda-\C_\infty \right| = 0$.
\end{theorem}
The result shows that every mean-field variational posterior is an accumulation point of some sequence of $\Xi$-variational posteriors. This type of result is called "large-time limits" in the optimal transport literature. When the likelihood is quadratic, it is possible to prove an exponential rate of convergence for $C_\lambda$ under more restrictive conditions \citep{Conforti2021}. However, this setting is uninteresting for Bayesian inference and we do not pursue it in this paper.

As $\lambda$ tends to zero, we provide analog results to show that $\Xi$-variational posterior converges to the exact posterior in the Wasserstein metric.
\begin{theorem}\label{thm:posterior-convergence}
Assume that $\Xi(\rmq^*_0) < \infty$ $[P_{\theta_0}]$-almost surely.  For each $\lambda \in \bar \R_+$,  define the set $\Q_\lambda$ as the set of minimizers for the functional $\rmq \mapsto \KL(\rmq \parallel \rmq^*_0) + \lambda \Xi(\rmq)$ in $\P_2(\Theta)$.  If $\rmq^*_\lambda \in \Q_\lambda$ converges as $\lambda \to 0$ in the Wasserstein metric, then $\lim_{\lambda \to 0}    W_2(\rmq^*_0, \rmq^*_\lambda) = 0$.  Furthermore, the $\Xi$-VI cost converges to the true posterior ELBO, i.e. $\lim_{\lambda \to 0}\left|\C_\lambda-\C_0 \right| = 0$.
\end{theorem}

\paragraph{Algorithmic Stability.} Let $m^*_\lambda$ be the product of marginals of $\rmq^*_\lambda$. In \Cref{sect-eot-derivation}, we produce to replace idealized \Cref{alg-full} with a simple, efficient approximate \Cref{alg-approx}. A natural question to ask is whether the solution is stable after we replace $m^*_\lambda$ with pseudomarginals $\tilde m$. To answer this question, we leverage the tools of quantitative stability from OT theory \citep{Eckstein2022}. We make two assumptions: a Lipschitz cost assumption, and transportation cost inequality.

\begin{assumption}[Lipschitz Cost Assumption]\label{assumption:lip-cost}
We assume that there exists a constant $L \geq 0$ and $\phi_i: \Theta_i \to \R$ such that for all $\rmq \in \C(m^*_\lambda)$ and $\tilde \rmq \in \C(\tilde m)$,
\begin{equation} \label{eqn-lip-cost}
\left| \int_\Theta \left(\ell(\bx^{(n)}; \theta)-\sum_{i = 1}^D \phi_i(\theta_i) \right) (\rmq(\theta)-\tilde \rmq(\theta)) d \theta \right| \leq L W_2(\rmq, \tilde \rmq).
\end{equation}
\end{assumption}
This assumption is slightly more general than the Lipschitzness of $\ell(\bx^{(n)}; \cdot)$ minus additive correction factors. As an example, the Gaussian likelihood satisfies \Cref{assumption:lip-cost}  (Lemma 3.5,  \cite{Eckstein2022}).

\begin{assumption}[Transportation Cost Inequality]\label{assumption:transportation-cost}
A product distribution $m$ over $\Theta$ satisfies the transportation cost inequality if there exists a constant $C$ such that
\begin{equation*}
    W_2(\rmq_1, \rmq_2) \leq \sqrt{\KL(\rmq_1 \parallel \rmq_2)}, \quad \text{for all } \rmq_1, \rmq_2 \in \C(m).
\end{equation*}
\end{assumption}
\Cref{assumption:transportation-cost} is standard in high-dimensional statistics \citep{Wainwright2019}. When $\Theta$ is compact, the assumption follows from Pinsker's inequality. Otherwise, this assumption holds when each marginal has a finite exponential moment.

We now state the main stability result which upper bounds the approximation error of \Cref{alg-approx} using the approximation error of the pseudomarginals.
\begin{theorem}[Stability of \Cref{alg-approx}] \label{thm:stability}
Let \Cref{assumption:lip-cost} hold with a Lipschitz constant $L$. Let $m^*_\lambda$ be the marginals of $\Xi$-variational posterior $\rmq_\lambda^*$ and $\tilde m \in \M\left(\Theta\right)$ be another product distribution. Suppose $m^*_\lambda$ satisfies \Cref{assumption:transportation-cost} with a fixed constant $C$. Then for the one-step approximation $\tilde \rmq_\lambda$ defined in \Cref{alg-approx} with pseudomarginals $\tilde m$, the following upper bound holds:
\begin{equation}
    W_2(\rmq^*_\lambda, \tilde \rmq_\lambda) \leq W_2(m^*_\lambda, \tilde m) + 2CL^{\frac{1}{4}} W_2^{\frac{1}{4}}(m^*_\lambda, \tilde m).
\end{equation}
\end{theorem}
The proof uses an OT technique called shadowing. See \Cref{proof-theory-mf-conv} for details.

The result highlights the stability of \Cref{alg-approx}, as the approximation error of $\rmq^*_\lambda$ is only Lipschitz in the approximation error of the pseudomarginals. If $\tilde m$ is close enough to $m^*_\lambda$ in terms of the $W_2$ metric, the output of Algorithm \Cref{alg-approx} is guaranteed to well approximate true variational posterior $\rmq^*_\lambda$.

\begin{corollary}\label{cor:stability-1}
  Assume \Cref{assumption:lip-cost} with Lipschitz constant $L$, and
  \Cref{assumption:transportation-cost} for the pseudomarginals. Then
  the following limits hold:
\begin{enumerate}
\item Let $\rmq^{*^{(\infty)}}_\lambda$ be the optimizer of
  \Cref{eqn-VI-inner} with marginals $\{\rmq^*_{\infty, i}\}_{i \in
    [D]}$. Then $ \lim_{\lambda \to \infty}W_2(\rmq^{*^{(\infty)}}_\lambda,
    \rmq^*_\lambda) = 0$.
\item Let $\rmq^{*^{(0)}}_\lambda$ be the optimizer of
  \Cref{eqn-VI-inner} with marginals
  $\{\rmq^*_{0, i}\}_{i \in [D]}$. Then $
    \lim_{\lambda \to 0}W_2(\rmq^{*^{(0)}}_\lambda, \rmq^*_\lambda)
    = 0$.
\end{enumerate}
\end{corollary}
The Corollary is a consequence of \Cref{thm:mf-convergence},
\Cref{thm:posterior-convergence}, and \Cref{thm:stability}. As
$\lambda$ tends to $0$ or $\infty$, the error of replacing the
idealized \Cref{alg-full} with \Cref{alg-approx} vanishes when we use
exact posterior marginals or mean-field variational posteriors,
respectively. If we plug in a consistent estimate of the exact
posterior marginals (e.g. TAP approximation of a linear model with
i.i.d. Gaussian design \citep{Celentano2023TAPMF}), then
\Cref{alg-approx} asymptotically recovers the exact posterior as
$\lambda$ tends to zero.

\subsection{Full Coordinate Ascent Algorithm} \label{sec-full-CAVI}
In this section, we present a full coordinate descent algorithm to exactly optimize the $\Xi$-VI objective.

\subsubsection{Outer Variational Problem}

We now derive steps to solve the outer variational problem of
\Cref{VI-EOT-decomposition}. Treating $\phi_1, \ldots, \phi_D$ as fixed, we optimize over the marginals $m_i$'s:
\begin{equation}\label{outer-variational-ELBO}
  \min_{m \in \M(\Theta)}
  \E_m\!\left[
    \underbrace{
      (\lambda +1) \left(\sum_{i=1}^D \phi_i(\theta_i)\right)
      \exp\!\left(\sum_{i=1}^D \phi_i(\theta_i) +
                  \frac{1}{\lambda+1}\,\ell(\bx;\theta)\right)
    }_{\text{surrogate loss}}
  \right]
  + \KL(m \parallel \pi).
\end{equation}

\Cref{outer-variational-ELBO} is equivalent to a mean-field VI problem
with a surrogate log-likelihood. To solve
\Cref{outer-variational-ELBO}, we follow a method based on coordinate
ascent variational inference (CAVI)~\citep{Blei2017}.

Denote $\Theta_{-i} := \prod_{j\neq i} \Theta_j$,
$\theta_{-i} := (\theta_{[D]\setminus\{i\}})$, and
\[
m_{-i}^t := \prod_{j < i} m_j^{t+1}(\theta_j) \prod_{j > i} m_j^{t}(\theta_j).
\]
Define $\nu_i^{t+1}(\theta_i)$ as
\begin{equation}
  \label{eq:outer_eta}
  \nu_i^{t+1}(\theta_i)
    := \E_{m_{-i}^t}\!\left[
      (\lambda+1)\left(\sum_{i=1}^D \phi_i^{t+1}(\theta_i)\right)
      \exp\!\left(
        \sum_{i=1}^D \phi_i^{t+1}(\theta_i)
        + \frac{1}{\lambda+1}\,\ell(\bx; \theta_i, \theta_{-i})
      \right)
    \right].
\end{equation}

Given marginals $m^t = (m_1^t, \ldots, m_D^t)$ and potentials
$\phi^{t+1} = (\phi_1^{t+1}, \ldots, \phi_D^{t+1}) \in E(m^t)$,
CAVI iteratively updates each marginal $i\in[D]$ by solving:
\begin{equation*}
m_i^{t+1}
  = \argmin_{m_i \in M(\Theta_i)}
      \E_{m_i}[\nu_i^{t+1}(\theta_i)]
      + \KL(m_i \parallel \pi_i).
\end{equation*}
This yields the explicit update
\begin{equation}\label{cavi-update}
    m_i^{t+1}(\theta_i)
      \propto \pi_i(\theta_i)\,
      \exp\!\big(- \nu_i^{t+1}(\theta_i)\big),
      \qquad \forall\,\theta_i \in \Theta_i.
\end{equation}

\subsubsection{Full Coordinate Ascent Algorithm}

\Cref{alg-full} presents the full coordinate ascent algorithm. It
monitors the change in the ELBO as the criterion for convergence,
which is equivalent (up to a constant) to the KL divergence between the
variational posterior and the exact posterior.

\begin{algorithm}[h]
\SetAlgoLined
\KwIn{Log-likelihood $\ell(\bx;\theta)$, prior $\pi$, tolerance $\epsilon$, regularization parameter $\lambda$.}
\textbf{Initialize:} marginals $m_1^0, \ldots, m_D^0$; EOT potentials $\phi_1^0, \ldots, \phi_D^0$; $t=0$. \\[0.1cm]
\While{ELBO has not converged}{
\For{$i\in[D]$}{
Update $\phi_i^{t+1}(\theta_i)$ using \Cref{sinkhorn-update}.}
\For{$i\in[D]$}{
Update $m_i^{t+1}(\theta_i)$ using \Cref{cavi-update}. \tcp{Challenging step}}
Compute
\[
\rmq^{t+1}(\theta)
  = \exp\!\left(
      \sum_{i=1}^D \phi_i^{t+1}(\theta_i)
      + \frac{1}{\lambda+1}\,\ell(\bx;\theta)
    \right)
    \prod_{i=1}^D m_i^{t+1}(\theta_i).
\]
Compute
\[
\mathrm{ELBO}(\rmq^{t+1})
  = \E_{\rmq^{t+1}}[\ell(\bx;\theta)+\log\pi(\theta)]
    - \E_{\rmq^{t+1}}[\log \rmq^{t+1}(\theta)].
\]
Increment $t = t+1$.
}
\KwOut{$\rmq(\theta)$.}
\caption{Coordinate Ascent Algorithm}
\label{alg-full}
\end{algorithm}

Unfortunately, \Cref{alg-full} is difficult to implement because we cannot
compute the expectations required in \Cref{sinkhorn-update} or
\Cref{eq:outer_eta}. When the $\phi_i$'s are represented implicitly,
there is no practically stable MFVI procedure for such implicit log-likelihoods,
especially in high-dimensional models.

\section{Support Results} \label{sect-support}
\begin{lemma}[Gibbs variational principle]
\label{lemma-gibbs}
For probability measures $\mu$ on $\Theta$,
\begin{equation}
    \log \E_{\mu}[\exp(f(\theta))] = \sup_{\nu \in \P(\Theta)} \left\{\E_\nu[f(\theta)] - \KL(\nu \parallel \mu)  \right\}.
    \end{equation}
\end{lemma}

\begin{lemma} \label{lemma:matrix-2}
    Let $A$ be a symmetric, positive definite matrix. For all $j \in [D]$, it holds that $(A^{-1})_{jj} \geq \frac{1}{A_{jj}}$.
\end{lemma}
\begin{proof}
    Given that $A$ is symmetric and positive definite, there exists another symmetric positive definite matrix $B$ such that $B^{2} = A$. We note that $A_{jj} = e_{j}^{T} A e_{j} = e_{j}^{T} B^{T} B e_{j} = \| B e_{j} \|_2^{2}$ and similarly, $(A^{-1})_{jj} = \| B^{-1} e_{j} \|_2^{2}$.

By the Cauchy–Schwarz inequality, we have
\begin{equation*}
    \langle B e_{j}, B^{-1} e_{j} \rangle^{2} \leq \| B e_{j} \|_2^{2} \| B^{-1} e_{j} \|_2^{2} = A_{jj} (A^{-1})_{jj}.
\end{equation*}
However,
\begin{equation*}
    \langle B e_{j}, B^{-1} e_{j} \rangle = e_{j}^{T} (B^{-1})^{T} B e_{j} = e_{j}^{T} B^{-1} B e_{j} = e_{j}^{T} e_{j} = 1.
\end{equation*}
Therefore, we have $A_{jj} (A^{-1})_{jj} \geq 1$, which simplifies to $(A^{-1})_{jj} \geq \frac{1}{A_{jj}}$. This completes the proof.
\end{proof}

\begin{lemma}
For $\rmq \in \P(\Theta)$, the following variational characterization of its expressivity holds:
\begin{equation*}
    \Xi(\rmq) = \sup_f \E_\rmq[f]-\log \E_{\rmq_i} \E_{\rmq_{-i}}[\exp(f(\theta_i, \theta_{-i})) \mid \theta_i].
\end{equation*}
\begin{proof}
    Apply Donsker–Varadhan lemma. We obtain
\begin{equation*}
     \Xi(\rmq) = \sup_f \E_\rmq[f]-\log \E_{\rmq_i} \E_{\rmq_{-i}}[\exp(f(\theta_i, \theta_{-i}))].
\end{equation*}
Apply Donsker–Varadhan lemma again to $\KL(\rmq(\theta_i, \theta_{-i}) $
\end{proof}
\end{lemma}
\begin{theorem}[Theorem 5.11, \cite{Santambrogio2015}]\label{thm:Wasserstein-Convergence}
In the space \(\P_p(\R^d)\), we have \(W_p(\mu_n, \mu) \to 0\) if and only if \(\mu_n \to \mu\) weakly and
\[
\int |x|^p \, d\mu_n \to \int |x|^p \, d\mu,
\]
where \(p > 0\) is a given exponent.
\end{theorem}
\begin{lemma} \label{lemma:lsc-kl-xi}
Let $\rmq_n$ be a sequence of measures in $\P_2(\Theta)$. If $W_2(\rmq_n, \rmq) \to 0$ for some $\rmq \in \P_2(\Theta)$, then
\begin{equation*}
 \liminf_{n \to \infty} \Xi(\rmq_n) \geq \Xi(\rmq).
\end{equation*}
Let $\rmq_0$ be another measure in $\P_2(\Theta)$. We have
\begin{equation*}
 \liminf_{n \to \infty} \KL (\rmq_n  \parallel  \rmq_0) \geq \KL (\rmq  \parallel  \rmq_0).
\end{equation*}
\end{lemma}
\begin{proof}
The second property follows from the fact that functional $\KL (\cdot \parallel \rmq^*_0)$ is continuous in the Wasserstein metric  (Proposition 7.1, \cite{Santambrogio2015}). For any $\rmq_n \overset{W_2}{\to} \rmq$, \Cref{thm:Wasserstein-Convergence} implies that $\rmq_n$ weakly converge to $\rmq_0$. The convergence $W_2(\rmq_n, \rmq_0) \to 0$  implies the convergence $W_2(\rmq_{n,i}, \rmq_{0, i}) \to 0$ for each $i \in [D]$. Since $\KL$ is lower semicontinuous in both arguments (Theorem 4.8, \cite{Polyanskiy2022}), we get
\begin{equation}
 \liminf_{n \to \infty} \Xi(\rmq_n) =  \liminf_{n \to \infty} \KL\left(\rmq_n \parallel \prod_{i = 1}^D \rmq_{n,i}\right)\geq \KL\left(\rmq_0 \parallel \prod_{i = 1}^D \rmq_{0,i} \right) = \Xi(\rmq),
\end{equation}
where $D$ is fixed with respect to $n$.
\end{proof}

\begin{definition}[Shadow]
Let $p \in [1, \infty]$ and $m, \tilde m$ be product measures within $\P_p(\Theta)$. Assume $\kappa_i \in \C(m_i, \tilde m_i)$ is a coupling that achieves $W_p(m_i, \tilde m_i)$ and let $\kappa_i = m_i \otimes K_i$ represent a disintegration. For a given $\rmq \in \C(m)$, its shadow $\rmq^s \in \C(\tilde m)$ is defined as the second marginal of $\rmq \otimes K \in \P(\Theta \times \Theta)$, where the kernel $K : \Theta \rightarrow \P(\Theta)$ is constructed as a direct sum $K(x) = K_1(x_1) \otimes \ldots \otimes K_D(x_D)$.
\end{definition}

 Given a coupling $\rmq \in \C(m)$, its shadow $\rmq^s$ satisfies the following properties.

\begin{lemma}[Lemma 3.2, \cite{Eckstein2022}]\label{lemma:shadow1}For product distributions $m, \tilde m \in \P_2(\Theta)$ and coupling $\rmq \in \C(m)$, its shadow $\rmq^s \in \C(\tilde m)$ satisfies
\begin{equation*}
W_2(\rmq, \rmq^s) = W_2(m, \tilde m), \quad D_f(\rmq^s \parallel \tilde m) \leq D_f(\rmq \parallel m),
\end{equation*}
where $D_f(\cdot)$ is any $f$-divergence.
\end{lemma}
\begin{theorem}[Theorem 12, \cite{Lin2022}] \label{thm:Lin-2022-12}
Let \(\{\phi^t\}_{t \geq 0}\) be the iterates generated by \Cref{alg-Sinkhorn1}. The number of iterations $t$ required to reach the stopping criterion $E \leq \epsilon'$ is upper bounded by:
\[
t \leq 2 + \frac{2 D^2 \left[\|\bC\|_\infty/(\lambda + 1)-\log\left(\min_{1\leq i \leq N, 1 \leq j \leq D} m_{ij}\right) \right]}{\epsilon'}.
\]
\end{theorem}

\section{Proofs of \Cref{sect-xi-vi}}\label{sect-proof1}
\begin{proof}[Derivation of \Cref{eqn-VI-inner}]
Let $m$ be given. Then to optimize $\rmq^*$, we have
\begin{equation*}
\begin{aligned}
\rmq^*(\theta) &= \argmin\limits_{\rmq \in \C(m)} \E_\rmq[-\ell(\bx; \theta)] + \lambda \KL(\rmq \parallel m) + \KL(\rmq \parallel \pi) \\
&= \argmin\limits_{\rmq \in \C(m)} \E_\rmq[-\ell(\bx; \theta)] + \lambda \KL(\rmq \parallel m) + \KL(\rmq \parallel m) + \KL(m \parallel \pi) \\
&= \argmin\limits_{\rmq \in \C(m)} \E_\rmq[-\ell(\bx; \theta)] + (\lambda +1) \KL(\rmq \parallel m).
\end{aligned}
\end{equation*}
The first line uses the fact that $m$ is a product distribution. The third line drops $\KL(m \parallel \pi)$ as it does not depend on $\rmq$.
\end{proof}
The next result states that the solution to the EOT problem~\eqref{eqn-VI-inner} has a unique representation.
\begin{theorem}[Structure Theorem for Multi-Marginal EOT]
  \label{thm:MEOT-struct}
  Assume that
  \[
  \inf_{\rmq \in \mathcal{C}(m)} \left\{ -\mathbb{E}_{\rmq}[\ell(\bx; \theta)] + (\lambda + 1) \Xi(\rmq) \right\} < \infty
  \quad \text{and} \quad
  \sup_{\theta \in \Theta} \ell(\bx; \theta) < \infty.
  \]

  Then there exists a unique minimizer $\rmq^*$ to the inner variational problem~\eqref{eqn-VI-inner} that is absolutely continuous with respect to $m$ (denoted $\rmq^* \ll m$), and the following hold:

  \begin{enumerate}[label=\textnormal{(\arabic*)}]
    \item There exist measurable functions $\phi_i^*: \Theta_i \to \mathbb{R}$ for $i \in [D]$ such that
    \begin{equation} \label{eqn-EOT-solution-appendix}
      \rmq^*(\theta) = \exp\left( \sum_{i = 1}^D \phi_i^*(\theta_i) + \frac{1}{\lambda + 1} \ell(\bx; \theta) \right) m(\theta),
    \end{equation}
    $m$-almost surely. The collection $\phi^* := (\phi_1^*, \dots, \phi_D^*)$ is referred to as the \emph{EOT potentials}. Each $\phi_i^*$ is $m_i$-almost surely measurable and unique up to an additive constant. Moreover, if $\mathbb{E}_{m_i}[\phi_i^*] \geq 0$, then $\phi_i^* \in L^\infty(m_i)$ for all $i \in [D]$.

    \item Conversely, suppose $\rmq \in \mathcal{C}(m)$ admits a density of the form in \Cref{eqn-EOT-solution-appendix}, $m$-almost surely, for some functions $\phi_i: \Theta_i \to \mathbb{R}$. Then $\rmq$ minimizes the inner variational problem in \Cref{eqn-VI-inner}, and the functions $\phi_i$ are the EOT potentials.
  \end{enumerate}
\end{theorem}
This result first appeared heuristically in \cite{Carlier2022}. For
$D = 2$, the uniform boundedness assumption can be relaxed to
$\ell(\bx; \cdot)$ being integrable (Theorem 4.2,
\cite{Nutz2021Notes}).
\begin{proof}[Proof of \Cref{thm:MEOT-struct}]
Define the auxiliary distribution $\rmq_{\mathrm{aux}} \in \P(\Theta)$ as
\begin{equation} \label{auxiliary-ref-measure-eot}
    \rmq_{\mathrm{aux}}(\theta) = \cZ_n(\lambda)^{-1} \exp\left( \frac{1}{\lambda + 1} \ell(\bx; \theta) \right) m(\theta),
\end{equation}
where $\cZ_n(\lambda)$ is the normalizing constant. Since $\sup_{\theta \in \Theta} \ell(\bx; \theta) < \infty$, $\cZ_n(\lambda) < \infty$, and hence $\rmq_{\mathrm{aux}}$ is well-defined and absolutely continuous with respect to $m$.

Minimizing the objective function in \Cref{eqn-VI-inner} is equivalent to minimizing the KL loss to $\rmq_{\text{aux}}$.
\begin{align*}
&\min\limits_{\rmq \in \C(m)} \E_\rmq\left[-\ell(\bx; \theta)\right] + (\lambda +1) \KL(\rmq \parallel m) \\
&=\min\limits_{\rmq \in \C(m)} \E_\rmq\left[-\frac{1}{\lambda + 1}\ell(\bx; \theta) \right] + \KL(\rmq \parallel m) \\
&= \min\limits_{\rmq \in \C(m)} \int_\Theta \rmq(\theta) \log \frac{\rmq(\theta)}{\exp\left(\frac{1}{\lambda + 1} \ell\left(\bx; \theta \right) \right)  m(\theta)} d \theta \\
&= \min\limits_{\rmq \in \C(m)} \KL(\rmq \parallel \rmq_{\text{aux}})-\log \E_{m} \exp\left(\frac{1}{\lambda + 1} \ell\left(\bx; \theta \right) \right)   \\
&\overset{C}{=} \min\limits_{\rmq \in \C(m)} \KL(\rmq \parallel \rmq_{\text{aux}}). \numberthis  \label{eqn-VI-inner-2}
\end{align*}
Since $\C(m)$ is displacement convex and the KL functional is displacement convex  \citep{Villani2009}, the solution is unique.

Let $\rmq^*$ be the optimizer. Then, by the method of Lagrange multipliers, there exist dual variables $\phi_i^* \in L^\infty(m_i)$ such that
\begin{equation*}
    \rmq^* = \argmin_{\rmq \in \P(\Theta)} \KL(\rmq \| \rmq_{\mathrm{aux}}) + \sum_{i = 1}^D \left( \E_{m_i}[\phi_i^*] - \E_{\rmq_i}[\phi_i^*] \right).
\end{equation*}
This is equivalent to
\begin{equation} \label{eqn-proof-MEOT-struct-1}
    \rmq^*(\theta) \propto \exp\left( \sum_{i=1}^D \phi_i^*(\theta_i) + \frac{1}{\lambda + 1} \ell(\bx; \theta) \right) m(\theta),
\end{equation}
after normalizing. Since the potentials $\phi_i^*$ are uniformly bounded, the resulting normalization constant is finite. By adding the normalizing constant to $\phi_D$, we obtain
\begin{equation*}
    \rmq^*(\theta) = \exp\left( \sum_{i=1}^D \phi_i^*(\theta_i) + \frac{1}{\lambda + 1} \ell(\bx; \theta) \right) m(\theta).
\end{equation*}

"only if" direction: Assume that the optimal coupling $\rmq^*$ is given by
 \begin{equation*}
    \rmq^* = \exp\left(\sum_{i = 1}^D \phi_i(\theta_i) + \frac{1}{\lambda + 1} \ell\left(\bx; \theta\right) \right) m(\theta).
\end{equation*}
where $\phi = (\phi_1, \cdots, \phi_D) \in \prod_{i = 1}^D L^1(m_i)$ are some potential functions.

Plugging the solution in the EOT primal problem, for each $i$ and $[m_i]$-a.s.$\theta_i$,  the potentials satisfy a set of fixed point equations called the \textit{Schr\"odinger system}:
\begin{equation} \label{eqn-proof-MEOT-struct-2}
    \exp(\phi_i(\theta_i)) \int_{\Theta_{-i}} \exp \left(\sum_{j \neq i} \phi_j(\theta_j) + \frac{1}{\lambda + 1} \ell\left(\bx; \theta\right)\right) m_{-i}(\theta_{-i}) d \theta_{-i}  = 1.
\end{equation}
The Schr\"odinger system (\Cref{eqn-proof-MEOT-struct-2}) satisfies the Euler-Lagrange optimality condition for the primal EOT problem \citep{Carlier2022}. Precisely, the EOT potentials solve
\begin{equation*}
\max_{\phi \in \prod_{i = 1}^D L^1(m_i)} \sum_{i = 1}^D \E_{m_i} \phi_i-\E_{m} \left[\exp\left(\sum_{i = 1}^D \phi_i(\theta_i) + \frac{1}{\lambda + 1} \ell\left(\bx; \theta\right) \right)\right],
\end{equation*}
which is the dual problem to the multimarginal EOT problem (\Cref{eqn-VI-inner}).  Since the EOT problem is convex \citep{Nutz2021Notes}, the primal-dual gap closes, which means the probability measure $\rmq$ defined under $\phi$ solves \Cref{eqn-VI-inner}.

To see that $\phi_i \in L^\infty(\Theta_i)$ for $i \in [D]$. Assume that $\E_{m_i}[\phi_i] \geq 0$, which is possible under the Euler–Lagrange condition:
\begin{equation*}
 \sum_{i = 1}^D \E_{m_i} \phi_i  = \min\limits_{\rmq \in \C(m)}  \KL(\rmq \parallel \rmq_{\text{aux}}) \geq 0 \\
\end{equation*}

By \Cref{eqn-proof-MEOT-struct-2}, we apply Jensen's inequality to obtain that
\begin{align*}
    \phi_i(\theta_i) &=-\log \int_{\Theta_{-i}} \exp\left(\sum_{j \neq i} \phi_j(\theta_j) + \frac{1}{\lambda + 1} \ell\left(\bx; \theta\right) \right) dm_{-i}(\theta_{-i}) \\
    &\leq -\E_{m_{-i}} \left[ \sum_{j \neq i} \phi_j(\theta_j) + \frac{1}{\lambda + 1} \ell\left(\bx; \theta\right) \right] \leq -\frac{1}{\lambda + 1}\E_{m_{-i}}\left[ \ell\left(\bx; \theta\right)\right],
\end{align*}
thus $\sup_{\theta_i \in \Theta_i} \phi_i(\theta_i) \leq -\sup_{\theta \in \Theta}|\ell\left(\bx; \theta\right)|/(\lambda + 1)$ for all $i \in [D]$.

For the other direction, since $\sup_{\theta \in \Theta}\bl\left(\bx; \theta\right)   < \infty$, we have
\begin{align*}
    \phi_i(\theta_i) &=-\log \int_{\Theta_{-i}} \exp\left(\sum_{j \neq i} \phi_j(\theta_j) + \frac{1}{\lambda + 1} \ell\left(\bx; \theta\right) \right) dm_{-i}(\theta_{-i}) \\
&\geq-\frac{\sup_{\theta \in \Theta}\bl\left(\bx; \theta\right)}{\lambda + 1}- \log \int_{\Theta_{-i}} \exp\left(\sum_{j \neq i} \phi_j(\theta_j)\right) dm_{-i}(\theta_{-i}).
\end{align*}
Since the right-hand side of the inequality does not depend on $\theta_i$, $\inf_{\theta_i \in \Theta_i}\phi_i(\theta_i) >-\infty$ as long as $\sum_{j \neq i} \phi_j(\theta_j) < \infty$ holds $[m_{-i}]$-almost surely. Since $\sup_{\theta_i \in \Theta_i} \phi_i(\theta_i) \leq -\sup_{\theta \in \Theta} \ell\left(\bx; \theta\right)/(\lambda + 1)$, we have that $\inf_{\theta_i \in \Theta_i}\phi_i(\theta_i) >-\infty$ for all $i \in [D]$.
\end{proof}
\begin{proof}[Proof of \Cref{outer-variational-ELBO}]
We make the following derivation,
\begin{align*}
&\min\limits_{m \in \M(\Theta)} \min\limits_{\rmq \in \C(m)} \E_\rmq[-\ell(\bx; \theta)] + (\lambda +1) \KL(\rmq \parallel m) + \KL(m \parallel \pi) \\
&= \min\limits_{m \in \M(\Theta)} (\lambda +1)\int_\Theta \rmq^*(\theta) \log \frac{\rmq^*(\theta)}{\exp(\frac{1}{\lambda + 1} \ell\left(\bx; \theta \right)) m(\theta)} d \theta +  \KL(m \parallel \pi) \\
&= \min\limits_{m \in \M(\Theta)}  (\lambda +1)\int_\Theta \rmq^*(\theta) \log \frac{ \exp\left( \sum_{i = 1}^D \phi_i(\theta_i) + \frac{1}{\lambda + 1} \ell\left(\bx; \theta\right)\right)}{\exp(\frac{1}{\lambda + 1} \ell\left(\bx; \theta \right))} d \theta + \KL(m \parallel \pi) \\
&= \min\limits_{m \in \M(\Theta)}   \E_m\left[\underbrace{(\lambda +1) \left(\sum_{i = 1}^D \phi_i(\theta_i)  \right) \exp\left( \sum_{i = 1}^D \phi_i(\theta_i) + \frac{1}{\lambda + 1} \ell\left(\bx; \theta\right)\right) }_{\text{surrogate loss}} \right] + \KL(m \parallel \pi).
\end{align*}
\end{proof}
\section{Proofs of \Cref{sect-implementation}} \label{proof-implementation}

\begin{proof}[Proof of \Cref{prop:Sinkhorn1}]
By Theorem~14 of \cite{Lin2022}, \Cref{alg-Sinkhorn1} reaches the stopping
criterion $\bE \le \epsilon$ in $t$ iterations, where $t$ satisfies
\begin{equation*}
    t \le 2 + \frac{2 D^2 \|\bC\|_\infty - \log (\max_{ij} \bM_{ij})}{\epsilon(\lambda + 1)}.
\end{equation*}
This implies
\begin{equation} \label{prop-Sinkhorn1-iterations}
    t \asymp \mathrm{poly}\!\left(D,\; \bC_{\max}/\epsilon,\; (\lambda+1)^{-1}\right).
\end{equation}
\Cref{alg-Sinkhorn1} calls the following oracle $D$ times:
\begin{equation} \label{oracle-Sinkhorn1}
\text{Compute}\quad
r_i(\bQ)
=
\frac{
  \sum_{1 \le k_j \le N,\; j\neq i}
  \exp\!\left[
    \sum_{j=1}^D \bF_{j k_j}
    -\frac{1}{\lambda+1}\,\bC_{k_1 \cdots k_D}
  \right]
  \bM_{k_1 \cdots k_D}
}{
  \sum_{1 \le k_j \le N,\; j\in[D]}
  \exp\!\left[
    \sum_{j=1}^D \bF_{j k_j}
    -\frac{1}{\lambda+1}\,\bC_{k_1 \cdots k_D}
  \right]
  \bM_{k_1 \cdots k_D}
}.
\end{equation}
All other steps are computed in linear time.

By Theorems~5.5 and~7.4 of \cite{Altschuler2023}, this oracle can be
computed in $\mathrm{poly}(N, D)$ time.
Repeating the oracle evaluation $Dt$ times and using
\Cref{prop-Sinkhorn1-iterations}, the algorithm terminates in $\mathrm{poly}\!\left(N,\; D,\; \bC_{\max}/\epsilon,\; (\lambda+1)^{-1}\right)$ time.
\end{proof}

\begin{proof}[Proof of \Cref{cor:Sinkhorn-1}]
We consider Algorithm~1 of \cite{Fan2022}. The algorithm implements the
marginalization step of \Cref{alg-Sinkhorn1} using the sum–product
method. Consider a graph $G = ([D], E, K)$, where $[D]$, $E$, and $K$
denote the nodes, edges, and maximal cliques, respectively.
If the log-likelihood $\ell(\bx^{(n)};\theta)$ factorizes according to $G$,
by the Hammersley–Clifford theorem, $\ell(\bx; \theta)
  = \sum_{\alpha\in K} \ell_\alpha(\theta_\alpha)$ where each $\ell_\alpha$ is defined on $\prod_{j\in \alpha} \Theta_j$.

Define $\bC_{k_\alpha}$ as the tensor of
$\ell_\alpha(\theta_\alpha)$ evaluated at support points
$(\theta_i^{(s)}, i\in\alpha)_{s\in[N]}$.
The full cost tensor decomposes as
\begin{equation*}
    \bC_{k_1,\ldots,k_D}
      = \sum_{\alpha\in K} \bC_{k_\alpha}.
\end{equation*}

Let $t$ be the number of iterations for Algorithm~1 of \cite{Fan2022}
to reach tolerance $\epsilon$.
By Theorem~1 of \cite{Fan2022},
\begin{equation*}
    \E[t]
      = O\!\left(
          D^2
          \max_{\alpha\in K} \|\bC_{k_\alpha}\|_\infty\;
          (\lambda+1)^{-1}\;
          \epsilon^{-1}
        \right).
\end{equation*}

Let $\cT$ be a minimal junction tree for $G$.
Marginalizing over each factor in $\cT$ costs $O(N^{\omega(G)})$, and
message passing costs $O(d(\cT)\,N^{\omega(G)})$, where $d(\cT)$ is the
average leaf distance in $\cT$.

Since $\max_{\alpha\in K}\|\bC_{k_\alpha}\|_\infty$ is uniformly bounded,
the sum–product implementation of the Sinkhorn algorithm
requires $O\!\left(
  d(\cT)\,N^{\omega(G)}\, D^2\, (\lambda+1)^{-1}\, \epsilon^{-1}
\right)$ iterations.
Because $d(\cT)\le D$, the complexity is also $O\!\left(
  N^{\omega(G)}\, D^3\, (\lambda+1)^{-1}\, \epsilon^{-1}
\right)$.
\end{proof}
\section{Proofs of \Cref{sect-theory}}\label{proof-theory}
\subsection*{Proofs of \Cref{sect-theory-finite-D}}

We define the set $\tilde \Theta_n$ as the set of all $h$ defined in \Cref{change-of-variable}, and $\H(\rmq) := \int_\Theta \rmq(\theta) \log \rmq(\theta) d \theta$ as the Boltzmann's $H$-functional \citep{Villani2009}.

\begin{lemma}[Transformation Identities] \label{lemma:identity}
For $h := \delta_n^{-1}(\theta-\theta_0-\delta_n \Delta_{n, \theta_0})$ where $\theta \sim \rmq$, we have
\begin{equation*}
\rmq(\theta) = \left|\text{det}(\delta_n)\right|^{-1} \tilde \rmq(h), \quad \text{and} \quad \rmq_i(\theta) = \delta_{n,ii}^{-1} \tilde \rmq_i(h) \ \text{for} \ i \in [D],
\end{equation*}
Moreover, we have
\begin{equation*}
    \H(\rmq) =   \H(\tilde \rmq) -\log|\text{det}(\delta_n)| , \quad \text{and} \quad \H(\rmq_i) = \H(\tilde \rmq_i)-\log \delta_{n,ii},
\end{equation*}
and
\begin{equation*}
    \Xi(\rmq) = \Xi(\tilde \rmq) + \log|\text{det}(\delta_n)|-\sum_{i = 1}^D \log \delta_{n,ii},
\end{equation*}
and for any distribution $\rmq_1, \rmq_2$ over $\Theta$, we have
\begin{equation*}
    \KL(\rmq_1 \parallel \rmq_2) =   \KL(\tilde \rmq_1 \parallel \tilde \rmq_2),
\end{equation*}
where $\tilde \rmq_1, \tilde \rmq_2$ are densities defined via \Cref{change-of-variable}.
\end{lemma}
\begin{proof}[Proof of \Cref{lemma:identity}]
We obtain the first equality by applying the change-of-variables formula to \Cref{change-of-variable}.
\begin{equation*}
\rmq(\theta) = {|\text{det}(\delta_n)|}^{-1} \tilde \rmq(h),\quad \text{and} \quad \rmq_i(\theta) = \delta_{n,ii}^{-1} \tilde \rmq_i(h), i \in [D],
\end{equation*}
For the second equality, we have
\begin{equation*}
    \H(\tilde \rmq) = \int |\text{det}(\delta_n)| \rmq(\theta)  \log(|\text{det}(\delta_n)| \rmq(\theta)) d h = \int \rmq(\theta)  \log \rmq(\theta) d \theta + \log|\text{det}(\delta_n)|^{-1}.
\end{equation*}
The univariate case follows from this.

For the third equality, we can write
\begin{align*}
\Xi(\tilde \rmq) &= \H(\tilde \rmq) -\sum_{i = 1}^D \H(\tilde \rmq_i) \\
&= \H(\rmq)- \sum_{i = 1}^D \H(\tilde \rmq_i)-\log|\text{det}(\delta_n)| + \sum_{i = 1}^D \log \delta_{n,ii} = \Xi(\rmq) -\log|\text{det}(\delta_n)| + \sum_{i = 1}^D \log \delta_{n,ii} .
\end{align*}

For the fourth equality, we have
\begin{align*}
    \KL(\tilde \rmq_1 \parallel \tilde \rmq_2)  &= \H(\tilde \rmq_1)-\int \log \tilde \rmq_2(h) \tilde \rmq_1(h) dh \\
    &= \H(\rmq_1) + \log|\text{det}(\delta_n)|-\int \rmq_2(\theta) \rmq_1(\theta) d\theta-\log|\text{det}(\delta_n)|  \\
    &=  \KL( \rmq_1 \parallel \rmq_2).
\end{align*}
This concludes the proof.
\end{proof}

To establish the Bernstein von-Mises theorem, we introduce the tool of $\Gamma$-convergence \citep{Braides2014}.

\begin{definition}[\(\Gamma\)-Convergence]
Let \(X\) be a metric space and consider a set of functionals \(F_{\varepsilon} : X \to \R\) indexed by \(\varepsilon > 0\). A limiting functional \(F_0\) exists and is called the \(\Gamma\)-limit of \(F_\varepsilon\) as \(\varepsilon \to 0\), if the following conditions are met:
\begin{enumerate}
    \item \textbf{Liminf Inequality:} For all \(x \in X\) and for every sequence \(x_{\varepsilon} \to x\),
    \[
    F_0(x) \leq \liminf_{\varepsilon \to 0} F_{\varepsilon}(x_{\varepsilon}).
    \]
    \item \textbf{Limsup Inequality / Existence of a Recovery Sequence:} For each \(x \in X\), there exists a sequence \(\bar{x}_{\varepsilon} \to x\) such that
    \[
    F_0(x) \geq \limsup_{\varepsilon \to 0} F_{\varepsilon}(\bar{x}_{\varepsilon}).
    \]
\end{enumerate}
\end{definition}

The first condition requires \(F_0\) to be asymptotically upper bounded by \(F_\varepsilon\). When paired with the second condition, it ensures that \(F_0(x) = \lim\limits_{\varepsilon \to 0} F_{\varepsilon}(\bar{x}_{\varepsilon})\), thereby confirming that the lower bound is tight.

\begin{definition}[Equi-Coerciveness of Functionals]
A sequence of functionals \(F_\varepsilon: X \to \R\) is said to be equi-coercive if for every bounded sequence \(x_\varepsilon\) with \(F_\varepsilon(x_\varepsilon) \leq t\), there exists a subsequence $x_j$ of $x_\epsilon$ and a converging sequence $x_j'$ satisfies \(F_{\varepsilon_j}(x_j') \leq F_{\varepsilon_j}(x_j) + o(1)\).
\end{definition}

Equi-coerciveness ensures the existence of a precompact minimizing sequence for \(F_\varepsilon\), which helps establish the convergence \(x_\varepsilon \to x\).

\begin{theorem}[Fundamental Theorem of \(\Gamma\)-Convergence] \label{thm:gamma-cvg}
Let \(X\) be a metric space and \(F_\varepsilon\) an equi-coercive sequence of functionals. If \(F = \Gamma\)-\(\lim\limits_{\varepsilon \to 0} F_\varepsilon\), then
\begin{equation*}
    \arg\min_{x \in X} F = \lim\limits_{\varepsilon \to 0} \arg\min_{x \in X} F_\varepsilon.
\end{equation*}
\end{theorem}
This theorem implies that if minimizers \(x_\varepsilon\) for all \(F_\varepsilon\) exist, the sequence converges, potentially along a subsequence, to a minimizer of \(F\). We note that the converse is not necessarily true; there may exist minimizers for \(F\) which are not limits of minimizers for \(F_\varepsilon\).

Note that when $\delta_n = \lambda_n^{1/2} \delta_n$, we have $|\text{det}(\delta_n)|  = \lambda_n^{D/2}|\text{det}(\delta_n)|$ and $\delta_{n,ii} = \lambda_n^{1/2} \delta_{n,ii}$.

We can explicitly characterize the transformed variational posterior:
\begin{equation} \label{tranformed-variational-posterior}
    \tilde \rmq_{\lambda_n}(h) := |\text{det}(\delta_n)|  \rmq^*_{\lambda_n}(\theta_0 + \delta_n h + \delta_n \Delta_{n, \theta_0}),
\end{equation}
where $\rmq^*_\lambda$ is the original $\Xi$-variational posterior.
\begin{lemma} \label{lemma:transformed-VI-functional}
Under Definition \Cref{tranformed-variational-posterior}, the distribution $\tilde \rmq_\lambda$ solves the following variational problem
\begin{equation*}
    \tilde \rmq_{\lambda} =\argmin_{\rmq \in \P_2(\Theta)} \KL(\rmq \parallel \tilde \rmq_0) + \lambda_n \Xi(\rmq).
\end{equation*}
\end{lemma}
This Lemma is a direct consequence of
the transformation identities (\Cref{lemma:identity}) and \Cref{def:theory-setup}, thus the proof is omitted.
\begin{proof}[Proof of \Cref{thm:AN-1}]
WLOG, we assume that $\Theta = \R^D$. Otherwise, we use the same proof by adding an indicator of the minimizing set to the sequence of functionals.

\underline{\textbf{Regime 1:}} $\lambda_n \to \infty$.
It suffices to show
\begin{equation*}
        F_n(\rmq) :=  \KL(\rmq \parallel \tilde \rmq_0) + \lambda_n \Xi(\rmq),
\end{equation*}
$\Gamma$-converge to
\begin{equation*}
    F_0(\rmq) :=  \KL(\rmq \parallel N(0, V_{\theta_0}^{-1})) +\infty \Xi(\rmq),
\end{equation*}
in $[P_{\theta_0}]$-probability as $n \to \infty$.

By \Cref{thm:gamma-cvg}, $\Gamma$ convergence implies $W_2(\tilde \rmq_{\lambda_n}, \argmin_{\rmq \in \P_2(\Theta)} F_0(\rmq)) \overset{P_{\theta_0}}{\to} 0$, where $\rmq_0$ is the minimizer of $F_0$.

To prove the $\Gamma$-convergence, we  rewrite $F_n$.
\begin{align*}
    F_n(\rmq) &:=  \KL (\rmq \parallel \tilde \rmq_0) + \lambda_n \Xi(\rmq) \\
    &=\E_\rmq\left[-\ell\left(\bx^{(n)}; \theta_0 + \delta_n h + \delta_n \Delta_{n, \theta_0} \right)\right] + \KL(\rmq \parallel  \tilde \pi) +  \log|\text{det}(\delta_n)| + \lambda_n\Xi(\rmq) \\
    &+\int \pi(\theta_0 + \delta_n h + \delta_n \Delta_{n, \theta_0})\ell(\bx^{(n)}; \theta_0 + \delta_n h + \delta_n \Delta_{n, \theta_0}) d h . \\
    &= -\ell(\bx^{(n)}; \theta_0) + \E_\rmq\left[\frac{1}{2} h^T V_{\theta_0} h\right] +  \H(\rmq)- \E_\rmq\left[\log \pi(\theta_0 + \delta_n h + \delta_n \Delta_{n, \theta_0})\right] + \lambda_n \Xi(\rmq)\\
    &+ \log \int \pi(\theta_0 + \delta_n h + \delta_n \Delta_{n, \theta_0})\ell(\bx^{(n)}; \theta_0 + \delta_n h + \delta_n \Delta_{n, \theta_0}) d h + o_P(1).
\end{align*}

Applying LAN expansion and Laplace approximation to the log-normalizer, we have
\begin{align*}
  &\log \int \pi(\theta_0 + \delta_n h + \delta_n \Delta_{n, \theta_0})\ell(\bx^{(n)}; \theta_0 + \delta_n h + \delta_n \Delta_{n, \theta_0}) d h \\
  & =  \frac{D}{2} \log 2 \pi-\frac{1}{2} \log \text{det}(V_{\theta_0})  + \log \pi(\theta_0) + \ell(\bx^{(n)};\theta_0) + o_P(1).
\end{align*}
After cancellation, we have
\begin{align*}
    F_n(\rmq) &= \E_\rmq\left[\frac{1}{2} h^T V_{\theta_0} h\right] +  \H(\rmq)+ \lambda_n \Xi(\rmq) + (\frac{D}{2} \log 2 \pi-\frac{1}{2} \log \text{det}(V_{\theta_0})) \\
    & -\left\{\E_\rmq\left[\log \pi(\theta_0 + \delta_n h + \delta_n \Delta_{n, \theta_0})\right]-\log \pi(\theta_0) \right\} +  o_P(1).
\end{align*}
Using \Cref{assumption: prior} to bound the prior tail via Taylor expansion, we have an expression for $F_n$
\begin{equation} \label{bvm-regime1-Fn}
\begin{aligned}
F_n(\rmq) &= \E_\rmq\left[\frac{1}{2} h^T V_{\theta_0} h\right] + \H(\rmq)+ \lambda_n \Xi(\rmq) + \frac{D}{2} \log 2 \pi-\frac{1}{2} \log \text{det}(V_{\theta_0}) +  o_P(1). \\
&=  \KL(\rmq \parallel N(0, V_{\theta_0}^{-1})) + \lambda_n \Xi(\rmq) + o_P(1).
\end{aligned}
\end{equation}
Now we rewrite $F_0(\rmq)$.
\begin{align*}
F_0(\rmq)&:=  \KL(\rmq \parallel N(0, V_{\theta_0}^{-1})) +\infty \Xi(\rmq) \\
&= \E_\rmq\left[\frac{1}{2} h^T V_{\theta_0} h\right] + \H(\rmq)+ \infty \Xi(\rmq) + \frac{D}{2} \log 2 \pi-\frac{1}{2} \log \text{det}(V_{\theta_0}). \\
&= F_n(\rmq) +  \infty \Xi(\rmq) + o_P(1).
\end{align*}

Now we prove the $\Gamma$ convergence.

First, we verify the liminf inequality. Let $\rmq_n \overset{W_2}{\to} \rmq$. When $\rmq$ is not mean-field, we have:
\begin{align*}
    \liminf_{n \to \infty} F_n(\rmq_n) &\geq \liminf_{n \to \infty} \left \{ \KL(\rmq_n \parallel N(0, V_{\theta_0}^{-1})) + \lambda_n \Xi(\rmq_n)  \right\}- \epsilon \\
    &\geq  \liminf_{n \to \infty} \KL(\rmq_n \parallel N(0, V_{\theta_0}^{-1})) + \liminf_{n \to \infty} \lambda_n \liminf_{n \to \infty} \Xi(\rmq_n)-\epsilon \\
    &\geq \KL(\rmq \parallel N(0, V_{\theta_0}^{-1})) + \infty \Xi(\rmq)-\epsilon = \infty \geq F_0(\rmq).
\end{align*}
The second inequality follows from the definition of liminf. The third line is due to \Cref{lemma:lsc-kl-xi}, which states that the KL functional and $\Xi$ functional are lower semicontinuous.

\begin{align*}
 \liminf_{n \to \infty} F_n(\rmq_n) &\geq \liminf_{n \to \infty} \left \{ \KL(\rmq_n \parallel N(0, V_{\theta_0}^{-1})) + \lambda_n \Xi(\rmq_n)  \right\}- \epsilon \\
 &\geq \liminf_{n \to \infty} \KL(\rmq_n \parallel N(0, V_{\theta_0}^{-1}))-\epsilon \\
 &\geq \KL(\rmq \parallel N(0, V_{\theta_0}^{-1}))-\epsilon.
\end{align*}

Since this holds for all $\epsilon$, we verified that $ \liminf_{n \to \infty} F_n(\rmq_n) \geq F_0(\rmq)$.

Next, we show the existence of a recovery sequence. When $\rmq$ is not mean-field, $F_0(\rmq) = + \infty$, and the limsup inequality is automatically satisfied. When $\rmq$ is mean-field, choose $\rmq_n := \rmq$, then:
\begin{equation*}
    \limsup_{n \to \infty} F_n(\rmq_n) = \limsup_{n \to \infty}\KL(\rmq \parallel N(0, V_{\theta_0}^{-1})) + o_P(1) \leq F_0(\rmq).
\end{equation*}
Thus, $F_0$ is the $\Gamma$-limit of the sequence $F_n$.

Next we prove that the sequence $F_n$ is eqi-coercive. Take $n_j \to \infty$ and $\rmq_{n_j}$ such that $F_{n_j}(\rmq_{n_j}) \leq t$ for all $j$.  Then  $\lambda_{n_j} \Xi(\rmq_{n_j})$ is bounded as $\lambda_{n_j} \to \infty$, thus $\Xi(\rmq_{n_j}) = o(1)$.  Using this and \Cref{bvm-regime1-Fn}, we have
\begin{equation*}
   \KL\left(\rmq_{n_j} \parallel N(0, V_{\theta_0}^{-1}) \right) \leq t + 1, \quad \text{for sufficiently large } j.
\end{equation*}
Since $\KL\left(. \parallel N(0, V_{\theta_0}^{-1}) \right)$ is a Wasserstein (geodesically) convex functional, it is coercive by Lemma 2.4.8 of \cite{Ambrosio2005}. This implies that the set $\{\rmq \in \P_2(\Theta) \mid \KL\left(\rmq \parallel \tilde \rmq_0 \right) \leq t + 1\}$ is compact under the Wasserstein metric, thus $\rmq_{n_j}$ has a subsequence $\rmq_{n_j}'$ that converges to $\rmq^*$ in the Wasserstein metric of and $\KL\left(\rmq^* \parallel \tilde \rmq_0 \right) \leq t + 1$.  Thus we have $F_{n_j}(\rmq_{n_j}') \leq F_{n_j}(\rmq_{n_j}') + o(1)$ by \Cref{bvm-regime1-Fn} where $\rmq_{n_j}'$ is a converging subsequence of $\rmq_{n_j}$.  This verifies the equi-coercivity of $F_n$.

Lastly, we note that $F_0$ attains its minimum at $N(\Delta_{n, \theta_0},  V_{\theta_0}^{'-1})$ where $V_{\theta_0}^{'-1}$ is the MFVI covariance. As a result of \Cref{thm:gamma-cvg}, we conclude that the desired convergence takes place:
\begin{equation*}
   \KL\left(\rmq_{n_j} \parallel N(0, V_{\theta_0}^{-1}) \right) \leq t + 1, \quad \text{for sufficiently large } j.
\end{equation*}

\underline{\textbf{Regime 2:}} $\lambda_n \to 0$.

In this regime, we show that the functionals
\begin{equation*}
    F_n(\rmq) := \KL(\rmq \parallel \tilde \rmq_0) + \lambda_n \Xi(\rmq),
\end{equation*}
$\Gamma$-converge to
\begin{equation*}
    F_0(\rmq) :=  \KL(\rmq \parallel N(0, V_{\theta_0}^{-1})),
\end{equation*}
in $[P_{\theta_0}]$-probability as $n \to \infty$.

Given that $F_n$ is defined analogous to Regime 1, we will skip the derivation:
\begin{align*}
    F_n(\rmq) &= \E_\rmq\left[\frac{1}{2} h^T V_{\theta_0} h\right] + \H(\rmq)+ \lambda_n \Xi(\rmq) + \frac{D}{2} \log 2 \pi-\frac{1}{2} \log \text{det}(V_{\theta_0}) +  o_P(1). \\
    &=  \KL(\rmq \parallel N(0, V_{\theta_0}^{-1})) + \lambda_n \Xi(\rmq) + o_P(1).
\end{align*}

Now we prove the $\Gamma$ convergence.  First, we verify the liminf inequality. Let $\rmq_n \overset{W_2}{\to} \rmq$. We have:
\begin{align*}
    \liminf_{n \to \infty} F_n(\rmq_n) &\geq \liminf_{n \to \infty} \left \{ \KL(\rmq_n \parallel N(0, V_{\theta_0}^{-1})) + \lambda_n \Xi(\rmq_n)  \right\}- \epsilon \\
    &\geq  \liminf_{n \to \infty} \KL(\rmq_n \parallel N(0, V_{\theta_0}^{-1})) + \liminf_{n \to \infty} \lambda_n \liminf_{n \to \infty} \Xi(\rmq_n)-\epsilon \\
    &\geq \KL(\rmq \parallel N(0, V_{\theta_0}^{-1}))-\epsilon.
\end{align*}
Since this holds for all $\epsilon$, we verified that $ \liminf_{n \to \infty} F_n(\rmq_n) \geq F_0(\rmq)$.

For the recovery sequence, we take $\rmq_n := \rmq$. Since $\rmq$ is absolutely continuous with respect to the product of its marginals, $\Xi(\rmq)$ is finite.  Then we have:
\begin{equation*}
    \limsup_{n \to \infty} F_n(\rmq_n) = \limsup_{n \to \infty}\KL(\rmq \parallel N(0, V_{\theta_0}^{-1})) + o_P(1) \leq F_0(\rmq).
\end{equation*}

The equicoercivity of $F_n$ follows from the argument in regime 1. By \Cref{thm:gamma-cvg}, we conclude with the desired convergence:
\begin{equation*}
    W_2(\tilde \rmq_{\lambda_n}, N(0,  V_{\theta_0}^{'-1})) \to 0.
\end{equation*}
\underline{\textbf{Regime 3:}} $\lambda_n \to \lambda_\infty \in (0, \infty)$.

In this regime, we show that the functionals
\begin{equation*}
        F_n(\rmq) := \KL(\rmq \parallel \tilde \rmq_{0}) + \lambda_n \Xi(\rmq),
\end{equation*}
$\Gamma$-converge to
\begin{equation*}
    F_0(\rmq) :=  \KL(\rmq \parallel N(0, V_{\theta_0}^{-1})) + \lambda_\infty \Xi(\rmq),
\end{equation*}
in $[P_{\theta_0}]$-probability as $n \to \infty$.

Recall that
\begin{align*}
    F_n(\rmq) &= \E_\rmq\left[\frac{1}{2} h^T V_{\theta_0} h\right] + \H(\rmq)+ \lambda_n \Xi(\rmq) + \frac{D}{2} \log 2 \pi-\frac{1}{2} \log \text{det}(V_{\theta_0}) +  o_P(1). \\
    &=  \KL(\rmq \parallel N(0, V_{\theta_0}^{-1})) + \lambda_n \Xi(\rmq) + o_P(1).
\end{align*}

Now we prove the $\Gamma$ convergence.  First, we verify the liminf inequality. Let $\rmq_n \overset{W_2}{\to} \rmq$. We have:
\begin{align*}
    \liminf_{n \to \infty} F_n(\rmq_n) &\geq \liminf_{n \to \infty} \left \{ \KL(\rmq_n \parallel N(0, V_{\theta_0}^{-1})) + \lambda_n \Xi(\rmq_n)  \right\}- \epsilon \\
    &\geq  \liminf_{n \to \infty} \KL(\rmq_n \parallel N(0, V_{\theta_0}^{-1})) + \liminf_{n \to \infty} \lambda_n \liminf_{n \to \infty} \Xi(\rmq_n)-\epsilon \\
    &\geq \KL(\rmq \parallel N(0, V_{\theta_0}^{-1}))  + \lambda_\infty \Xi(\rmq)-\epsilon.
\end{align*}
The second inequality follows from the definition of liminf, and the last inequality is due to \Cref{lemma:lsc-kl-xi}, which states that the KL functional and $\Xi$ functional are lower semicontinuous.

For the recovery sequence, we take $\rmq_n := \rmq$. As long as $\Xi(\rmq)$ is finite, we have:
\begin{equation*}
    \limsup_{n \to \infty} F_n(\rmq_n) = \limsup_{n \to \infty}\KL(\rmq \parallel N(0, V_{\theta_0}^{-1})) + \lambda_n \Xi(\rmq) + o_P(1) = F_0(\rmq).
\end{equation*}
The equicoercivity of $F_n$ follows from the argument in regime 1. By \Cref{thm:gamma-cvg}, we have the convergence:
\begin{equation*}
    W_2(\tilde \rmq_{\lambda_n},  \argmin_{\rmq \in \P_2(\Theta)} F_0(\rmq)) \to 0.
\end{equation*}
\end{proof}
\begin{proof}[Proof of \Cref{cor:posterior-consistency}]
    Recall the definition of Wasserstein distance:
\begin{equation*}
   W_2(\rmp, \rmq) =  (\inf_{\pi \in \C(\rmp, \rmq)} \E_{\pi}[\|X-Y\|^2])^{1/2}.
\end{equation*}
Given the change-of-variables definition (\Cref{change-of-variable}), we have
\begin{align*}
   W_2(\tilde \rmq_{\lambda_n}, N(\mu, \Sigma)) &=  (\inf_{\pi \in \C(\tilde \rmq_{\lambda_n}, N(\mu, \Sigma))} \E_{\pi}[\|h-h'\|^2])^{1/2} \\
   &= |\text{det}(\delta_n)|^{-1} (\inf_{\pi \in \C(\rmq^*_{\lambda_n}, N(\delta_n \mu + \theta_0 + \delta_n \Delta_{n, \theta_0}, \delta_n^T \Sigma \delta_n ))} \E_{\pi}[\|\theta-\theta'\|^2])^{1/2} \\
   &=  |\text{det}(\delta_n)|^{-1} W_2(\rmq^*_{\lambda_n}, N(\delta_n \mu + \theta_0 + \delta_n \Delta_{n, \theta_0}, \delta_n^T \Sigma \delta_n )).
\end{align*}
If $W_2(\tilde \rmq_{\lambda_n}, N(\mu, \Sigma))$ tends to $0$, then $W_2(\rmq^*_{\lambda_n}, N(\delta_n \mu + \theta_0 + \delta_n \Delta_{n, \theta_0}, \delta_n^T \Sigma \delta_n ))$ tends to $0$.  Since $N(\delta_n \mu + \theta_0 + \delta_n \Delta_{n, \theta_0}, \delta_n^T \Sigma \delta_n )$ weakly converge to $\delta_{\theta_0}$, it converges to $\delta_{\theta_0}$ in Wasserstein metric.  By \Cref{thm:AN-1}, we have $\rmq^*_{\lambda_n}$ converges in Wasserstein metric to $\delta_{\theta_0}$, as desired.
\end{proof}

\subsection*{Proofs of \Cref{sect-theory-hd}}\label{proof-theory-hd}
We first prove a useful proposition.
\begin{proposition}[Optimality to fixed point] \label{prop:fixed-point}
 Let $m^*_\lambda(\theta) = \prod_{i = 1}^D m^*_{\lambda, i}(\theta_i)$ be the product of optimal marginals, and $\phi^*_\lambda$ be the optimal EOT potentials. Then $m^*_\lambda$ and $\phi^*_\lambda$ satisfy the fixed point equations:
\begin{equation}
\begin{aligned}
m^*_{\lambda, i}(\theta_i) &= Z_i^{-1} \exp(- (\lambda + 1) \phi^*_{\lambda, i}(\theta_i)) \pi_i(\theta_i), \quad \text{and } \quad \\
\phi_{\lambda, i}^*(\theta_i) &=-\log \int_{\Theta_{-i}} \exp\left(\frac{1}{\lambda + 1} \ell\left(\bx^{(n)}; \theta\right)- \lambda \sum_{j \neq i}\phi^*_{\lambda, j}(\theta_j) \right) \prod_{j \neq i} \pi_j(\theta_j) d \theta_{-i} \\
&+ \sum_{j \neq i} \log \int_{\Theta_j} \exp\left(- (\lambda + 1) \phi^*_{\lambda, j}(\theta_j) \right) \pi_j(\theta_j) d\theta_j,
\end{aligned}
\end{equation}
where $Z_i$'s are the normalizing constants.
\end{proposition}
\begin{proof}[Proof of \Cref{prop:fixed-point}]
Define $f_\lambda(\theta)$ as follows,
\begin{equation}
    f_\lambda(\theta) :=  -(\lambda +1) \left(\sum_{i = 1}^D \phi^*_{\lambda, i}(\theta_i)  \right) \exp\left( \sum_{i = 1}^D \phi^*_{\lambda, i}(\theta_i) + \frac{1}{\lambda + 1} \ell\left(\bx^{(n)}; \theta\right)\right) + \sum_{i = 1}^D \log \pi_i(\theta_i).
\end{equation}
 By \Cref{thm:MEOT-struct} and the uniform boundedness of $ \ell\left(\bx^{(n)}; \cdot \right)$, the function $f_\lambda$ is integrable with respect to $m^*_\lambda$.  From the derivation in \Cref{sect-eot-derivation}, the distribution $m^*_\lambda$ attain the minimum,
\begin{equation} \label{eqn:proof-optimality-fixed-point1}
\min\limits_{m \in \M(\Theta)} -\E_m \left[f_\lambda(\theta) \right] + \H(m).
\end{equation}
Define $\hat f_{\lambda, i}(\theta_i) = \E_{m^*_\lambda}[f_\lambda(\theta) \mid \theta_i]$ for $[m^*_{\lambda,i}]$-a.s. $\theta_i \in \Theta_i$. Since $m \in \M(\Theta)$, $\H(m) = \sum_{i = 1}^D \H(m_i)$ and by the tower property, we have
\begin{equation}
    m^*_{\lambda, i}(\theta_i) := \argmin_{m_i \in \M(\Theta_i)} \left(\E_{m_i}[- \hat f_{\lambda, i}(\theta_i)] + \H(m_i) \right).
\end{equation}

By the Gibbs variational principle (\Cref{lemma-gibbs}), the minimum is uniquely attained by
\begin{equation}
    m^*_{\lambda, i}(\theta_i) \propto \exp(\hat f_{\lambda, i}(\theta_i)).
\end{equation}

Recall that the optimal EOT potentials satisfy the Schr\"odinger system:
\begin{equation}
  \phi_{\lambda, i}^*(\theta_i) =-\log \int_{\Theta_{-i}} \exp\left(\sum_{j \neq i}\phi^*_{\lambda, j}(\theta_j) + \frac{1}{\lambda + 1} \ell\left(\bx^{(n)}; \theta\right) \right) \prod_{j \neq i} m^*_{\lambda, j}(\theta_j) d \theta_{-i}.
\end{equation}

This allows us to simplify $\hat f_{\lambda, i}(\theta_i)$:
\begin{align*}
&\hat f_{\lambda, i}(\theta_i) = \E_{m^*_\lambda}[f_\lambda(\theta) \mid \theta_i] \\
&=  \E_{m^*_\lambda} \left[ -(\lambda +1) \left(\sum_{i = 1}^D \phi^*_{\lambda, i}(\theta_i)  \right) \exp\left( \sum_{i = 1}^D \phi^*_{\lambda, i}(\theta_i) + \frac{1}{\lambda + 1} \ell\left(\bx^{(n)}; \theta\right)\right) + \sum_{i = 1}^D \log \pi_i(\theta_i) \mid \theta_i\right] \\
&= -(\lambda +1) \frac{\E_{m^*_{\lambda, -i}}  \left[ \left(\sum_{i = 1}^D \phi^*_{\lambda, i}(\theta_i)  \right) \exp\left( \sum_{j \neq i} \phi^*_{\lambda, j}(\theta_j) + \frac{1}{\lambda + 1} \ell\left(\bx^{(n)}; \theta\right)\right)\right]}{\exp\left(-\phi^*_{\lambda, i}(\theta_i) \right)}\\
&+ \E_{m^*_{\lambda, -i}} \left[\sum_{j \neq i} \log \pi_j(\theta_j)\right]  + \log \pi_i(\theta_i)\\
&= -(\lambda+1)\phi^*_{\lambda, i}(\theta_i)  + \log \pi(\theta_i)-(\lambda+1) \E_{\hat h_\lambda}\left[\sum_{j \neq i} \phi^*_{\lambda, j}(\theta_j) \right] + C.
\end{align*}
where $\hat h_\lambda(\theta_{-i}) \propto \exp\left( \sum_{j \neq i} \phi^*_{\lambda, j}(\theta_j) + \frac{1}{\lambda + 1} \ell\left(\bx^{(n)}; \theta\right)\right) \prod_{j \neq i} m^*_{\lambda, i}(\theta_i)$.  Since $\hat h_\lambda(\theta_{-i}) \propto \rmq^*_\lambda(\theta_{-i}, \theta_i)$, we have for all $\theta_i$,
\begin{equation}
     \E_{\hat h_{\lambda}(\theta_{-i})}\left[\sum_{j \neq i} \phi^*_{\lambda, j}(\theta_j) \right] = \E_{\rmq^*_\lambda(\theta_{-i}, \theta_i)}\left[\sum_{j \neq i} \phi^*_{\lambda, j}(\theta_j) \right] = \sum_{j \neq i}\E_{m^*_{\lambda, i}} \left[ \phi^*_{\lambda, j}(\theta_j)\right].
\end{equation}
The last equality uses the fact that $m^*_{\lambda, i}$ is the $i^{\text{th}}$ marginal of $\rmq^*_\lambda$.

Since $\sum_{j \neq i} \E_{m^*_{\lambda, j}}\left[ - (\lambda + 1) \phi^*_{\lambda, j}(\theta_j) + \log \pi_j(\theta_j)\right]$ does not depend on $\theta_i$,  we obtain
\begin{equation} \label{eqn:proof-prop-fixed-point-m}
    m^*_{\lambda, i}(\theta_i) \propto \exp(- (\lambda + 1) \phi^*_{\lambda, i}(\theta_i)) \pi_i(\theta_i).
\end{equation}

Using \Cref{eqn:proof-prop-fixed-point-m}, we conclude
\begin{align*}
\phi_{\lambda, i}^*(\theta_i) &=-\log \int_{\Theta_{-i}} \exp\left(\frac{1}{\lambda + 1} \ell\left(\bx^{(n)}; \theta\right)- \lambda \sum_{j \neq i}\phi^*_{\lambda, j}(\theta_j) \right) \prod_{j \neq i} \pi_j(\theta_j) d \theta_{-i} \\
&+ \sum_{j \neq i} \log \int_{\Theta_j} \exp\left(- (\lambda + 1) \phi^*_{\lambda, j}(\theta_j) \right) \pi_j(\theta_j) d\theta_j.
\end{align*}
\end{proof}

\begin{proof}[Proof of \Cref{thm:theory-hd-1}]
Since the parameter space $\Theta$ is compact and $\ell(\bx^{(n)}; \cdot) \in \CC^2(\Theta)$, both the gradient $\nabla \ell(\bx^{(n)}; \cdot)$ and the Hessian $\nabla^2 \ell(\bx^{(n)}; \cdot)$ are uniformly bounded over $\Theta$.

Let $u_i := \frac{1}{2} \inf_{\theta \in \Theta} [\nabla^2 \ell(\bx^{(n)}; \theta)]_{ii}$, $v_i$ is chosen such that $\sup_{\theta \in \Theta}|[\nabla \ell(\bx^{(n)}; \theta)]_i -v_i-2 u_i \theta_i| = b_i$, and $w$ is chosen such that $\sup_{\theta \in \Theta} |\ell(\bx^{(n)}; \theta)-w-\sum_{i} v_i \theta_i-\sum_{i} u_i \theta_i^2| = a$. We define a new log-likelihood $\tilde \ell(\bx^{(n)}; \theta)$ that shift $ \ell(\bx^{(n)}; \theta)$ by a quadratic function:
\begin{align*}
\tilde \ell(\bx^{(n)}; \theta) &:= \ell(\bx^{(n)}; \theta)-w-\sum_{i = 1}^D v_i \theta_i-\sum_{i = 1}^D u_i \theta_i^2.
\end{align*}
Some calculation yields that
\begin{align*}
    \sup_{\theta \in \Theta}|\tilde \ell(\bx^{(n)}; \theta) | &= a, \quad \sup_{\theta \in \Theta}|[\nabla \tilde \ell(\bx^{(n)}; \theta)]_i| = b_i, \quad  \sup_{\theta \in \Theta}|[\nabla^2 \tilde \ell(\bx^{(n)}; \theta)]_{ii}| = c_i.
\end{align*}
Given the optimal $m^*_{\lambda_n}$, the inner variational problem is
\begin{equation*}
    \rmq^*_{\lambda_n} = \argmin_{\rmq \in \C(m^*_{\lambda_n})} \left\{-\E_\rmq[\tilde \ell(\bx^{(n)}; \theta) ]+ (\lambda_n + 1) \KL(\rmq \parallel m^*_{\lambda_n}) \right\},
\end{equation*}
where we replace $\ell(\bx^{(n)}; \theta)$ with $\tilde \ell(\bx^{(n)}; \theta)$ since adding a tensorized function $-(w + \sum_{i = 1}^D v_i \theta_i + \sum_{i = 1}^D u_i \theta_i^2)$ to the cost does not change the optimal EOT coupling.

By \Cref{thm:MEOT-struct}, we can write $\rmq^*_{\lambda_n}$ as follows,
\begin{equation*}
    \rmq^*_{\lambda_n}(\theta) = \exp\left(\frac{1}{\lambda_n + 1} \tilde \ell(\bx^{(n)}; \theta)  + \sum_{i = 1}^D \phi^*_{\lambda_n, i}(\theta_i) \right) \prod_{i = 1}^D m^*_{\lambda_n, i}(\theta)
\end{equation*}
where $m^*_{\lambda_n, i}$'s are the marginals of $\rmq^*_{\lambda_n}$ and $\phi^*_{\lambda_n, i}$'s are the EOT potentials.

Define another product distribution $\tilde m_{\lambda_n}(\theta) \propto  \exp\left( \sum_{i = 1}^D \phi^*_{\lambda_n, i}(\theta_i) \right) m^*_{\lambda_n, i}(\theta)$. We can rewrite $\rmq^*_{\lambda_n}$ as the product of a tempered likelihood and a $\tilde m_{\lambda_n}$.
\begin{equation}
     \rmq^*_{\lambda_n}(\theta) =  \frac{1}{\cZ_D(\lambda_n)}\exp \left(\frac{1}{\lambda_n + 1} \tilde \ell(\bx^{(n)}; \theta)  \right) \tilde m_{\lambda_n}(\theta),
\end{equation}
where the normalizing constant is given by
\begin{equation}
    \cZ_D(\lambda_n) := \int_\Theta \exp \left(\frac{1}{\lambda_n + 1} \tilde \ell(\bx^{(n)}; \theta)  \right) \tilde m_{\lambda_n}(\theta) d \theta.
\end{equation}
First, we want to show that
\begin{equation}\label{proof:thm:theory-hd-1-eqn-0}
\lim\limits_{n \to \infty} \frac{1}{D}\left[\log \cZ_D(\lambda_n)-\sup_{m \in \M(\Theta), m \ll \tilde m_{\lambda_n}} \left\{ \frac{1}{\lambda_n +1}\E_m[\tilde \ell(\bx^{(n)}; \theta) ] -\sum_{i = 1}^D \KL(m_i \parallel \tilde m_{\lambda_n, i}) \right\}\right] = 0.
\end{equation}
Let $\|f \|_\infty$ denote the supremum norm of a function $f$. Fix some $\epsilon > 0$. Let $\cS_{\lambda_n}(\epsilon) \subset \Theta$ be a finite set such that for any $\theta \in \Theta$, there exists $s \in \cS_{\lambda_n}(\epsilon)$ satisfying
\begin{equation}\label{proof:thm:theory-hd-1-eqn-1}
\sum_{i = 1}^D \|\frac{1}{\lambda_n + 1}\partial_i \tilde \ell(\bx^{(n)}; \theta) -s_i\|_\infty^2 \leq \epsilon^2 D.
\end{equation}

Denote by $|\cS_{\lambda_n}(\epsilon)|$ the cardinality of $\cS_{\lambda_n}(\epsilon)$.  Theorem 1 of \cite{Yan2020} implies that
\begin{equation}\label{proof:thm:theory-hd-1-eqn-2}
\begin{aligned}
&\log \cZ_D(\lambda_n)-\sup_{m \in \M(\Theta), m \ll \tilde m_{\lambda_n}} \left[ \frac{1}{\lambda_n +1}\E_m[\tilde \ell(\bx^{(n)}; \theta) ] -\sum_{i = 1}^D \KL(m_i \parallel \tilde m_{\lambda_n, i}) \right] \leq \\
&4 \left( \frac{4}{(\lambda_n + 1)^2}\left(a \sum_{i = 1}^D c_{ii} + \sum_{i =1}^D b_i^2\right) + \frac{8}{(\lambda_n + 1)^2} \sum_{i = 1}^D \sum_{j = 1}^D b_i c_{ij} + \frac{16}{(\lambda_n +1 )^{3/2}}\left(a \sum_{i = 1}^D \sum_{j = 1}^D c_{ij}^2 + \sum_{i = 1}^D \sum_{j = 1}^D b_i b_j c_{ij}\right) \right)^{1/2} \\
&+ 4\left(\frac{1}{(\lambda_n + 1)^2} \sum_{i =1}^D b_i^2 + \epsilon^2 D \right)^{1/2} \left( \frac{8}{\lambda_n + 1} \left(\sum_{i = 1}^D c_{ii}^2\right)^{1/2} + 4 D^{1/2} \epsilon\right) + \frac{4}{\lambda_n + 1}\sum_{i = 1}^D c_{ii} + 2 D \epsilon \\
&+ \log 2 + \log |\cS_{\lambda_n}(\epsilon)|.
\end{aligned}
\end{equation}
Consider $\lambda_n \succ D^{-1/2}\max\left(\sqrt{a \sum_{i = 1}^D c_{ii}}, \sqrt{\sum_{i =1}^D b_i^2}, \sqrt{\sum_{i = 1}^D \sum_{j = 1}^D c_{ij}^2 },  D^{1/2}\right)$. Then,
\begin{align*}
&\frac{\sum_{i = 1}^D \sum_{j = 1}^D b_i c_{ij}}{(\lambda_n +1)^2} \leq\frac{\sqrt{\sum_{i =1}^D b_i^2} \sqrt{\sum_{i = 1}^D \sum_{j = 1}^D c_{ij}^2 }}{(\lambda_n +1)^2}  =  o(D), \\
&\frac{\sum_{i = 1}^D \sum_{j = 1}^D b_i b_j c_{ij}}{(\lambda_n + 1)^3} \leq  \frac{\sqrt{\sum_{i = 1}^D \sum_{j = 1}^D c_{ij}^2 } \sum_{i =1}^D b_i^2}{(\lambda_n + 1)^3} = o(D^{3/2}).
\end{align*}
by the Cauchy–Schwarz inequality. With the other terms being $o(D)$, we have
\begin{equation}\label{proof:thm:theory-hd-1-eqn-4}
\begin{aligned}
&\log \cZ_D(\lambda_n)-\sup_{m \in \M(\Theta), m \ll \tilde m_{\lambda_n}} \left[ \frac{1}{\lambda_n +1}\E_m[\tilde \ell(\bx^{(n)}; \theta) ] -\sum_{i = 1}^D \KL(m_i \parallel \tilde m_{\lambda_n, i}) \right] \\
&\leq o(D) + 2 D \epsilon +  \log 2 + \log |\cS_{\lambda_n}(\epsilon)|.
\end{aligned}
\end{equation}

To upper bound $|\cS_{\lambda_n}(\epsilon)|$, we can construct an $\epsilon$-covering by covering $[-\frac{b_i}{\lambda_n + 1}, \frac{b_i}{\lambda_n + 1}]$ with balls of size $2\epsilon$.  We consider a candidate set $\tilde \cS_{\lambda_n}(\epsilon)$ as the product of these coverings. Since $|\tilde \cS_{\lambda_n}(\epsilon)| =  \frac{\prod_{i = 1}^D b_i}{(\lambda_n + 1)^D \epsilon^D}$, we have
\begin{equation*}
         \log |\cS_{\lambda_n}(\epsilon)| \leq  \sum_{i = 1}^D \log b_i-D \log (\lambda_n + 1)-D \log \epsilon.
\end{equation*}
Define $\bar b := \sum_{i =1}^D b_i$. Since $\lambda_n \succ D^{-1/2} \sqrt{\sum_{i =1}^D b_i^2}$, we have $D \lambda_n \succ D^{1/2} \sqrt{\sum_{i =1}^D b_i^2} \geq D \bar b$. By Jensen's inequality,
\begin{equation*}
    D \log(\lambda_n + 1) \succ D \log(\bar b + 1) \geq \sum_{i = 1}^D \log(b_i).
\end{equation*}
To complete the bound of $\log |\cS_{\lambda_n}(\epsilon)|$, we choose a specific sequence $\epsilon_n := \sqrt{\frac{\bar b + 1}{\lambda_n + 1}}$.  The inequality above shows that $\epsilon = o(1)$.

Thus,
\begin{equation}
\log |\cS_{\lambda_n}(\epsilon_n)| \leq  D \log(\bar b + 1)-D \log (\lambda_n + 1)-D \log \epsilon_n = \frac{1}{2} D \log \left(\frac{\bar b + 1}{\lambda_n + 1} \right) \to -\infty.
\end{equation}

Plugging the definition of $\epsilon_n$ into \Cref{proof:thm:theory-hd-1-eqn-4}, we get
\begin{equation} \label{eqn-hd-1-oD}
\log \cZ_D(\lambda_n)-\sup_{m \in \M(\Theta), m \ll \tilde m_{\lambda_n}} \left[ \frac{1}{\lambda_n +1}\E_m[\tilde \ell(\bx^{(n)}; \theta) ] -\sum_{i = 1}^D \KL(m_i \parallel \tilde m_{\lambda_n, i}) \right] = o(D).
\end{equation}

For any $m \in \M(\Theta)$, we have
\begin{align*}
    \KL(m \parallel \rmq^*_{\lambda_n}) &= \int_\Theta m(\theta) \left[ \log \cZ_D(\lambda_n)-\frac{1}{\lambda_n + 1} \tilde \ell(\bx^{(n)}; \theta)  + \log \frac{m(\theta)}{\prod_{i = 1}^D \tilde m_{\lambda_n, i}(\theta_i)} \right] d \theta.  \\
    &=  \log \cZ_D(\lambda_n)-\frac{1}{\lambda_n +1}\E_m[\tilde \ell(\bx^{(n)}; \theta) ]  + \sum_{i = 1}^D \KL(m_i \parallel \tilde m_{\lambda_n, i}(\theta_i)).
\end{align*}
\Cref{eqn-hd-1-oD} implies that for $m_{\lambda_n}^* \in \arg \min_{m \in \M(\Theta)}  \KL(m \parallel \rmq^*_{\lambda_n})$, we have
\begin{equation} \label{proof:thm:theory-hd-1-eqn-5}
    \KL\left(m_{\lambda_n}^* \parallel \rmq^*_{\lambda_n} \right) = o(D).
\end{equation}
 For any $1$-Lipschitz function $f$ under the $L_1$ norm, consider the random variable $f(\theta)$, where $\theta \sim \rmq^*_{\lambda_n}$. This variable satisfies the inequality $\log \E_{\rmq^*_{\lambda_n}}[\exp(\langle t, f(\theta)-\E_{\rmq^*_{\lambda_n}}[f(\theta)]\rangle)] \leq 2 D \|t\|_2^2$, which is derived from the assumption that $\Theta = [-1,1]^D$. Thus, $\rmq^*_{\lambda_n}$ is $(4D)$-subGaussian.  By the $T_1$-transportation inequality (Theorem 4.8, \cite{VanHandel2014}), for any $m \in \M(\Theta)$, the following upper bound holds:
\begin{equation}
W_1(m, \rmq^*_{\lambda_n}) \leq \sqrt{8 D \cdot \KL(m \parallel \rmq^*_{\lambda_n})} = o(D),
\end{equation}
where $W_1$ is the $1-$Wasserstein distance.

Let $m_{\lambda_n}^*$ denote the minimizer of the left hand side \Cref{proof:thm:theory-hd-1-eqn-5}.  Consider a function $\psi$ that is $1$-Lipschitz on $\R$. The function $\theta \mapsto \sum_{i = 1}^D \psi(\theta_i)$ is also $1$-Lipschitz with respect to the $L_1$ norm. This follows from the inequality:
\begin{equation}
|\sum_{i = 1}^D \psi(\theta_i)-\sum_{i = 1}^D \psi(\theta_i')| \leq \sum_{i = 1}^D |\psi(\theta_i)-\psi(\theta_i')| \leq \sum_{i = 1}^D |\theta_i-\theta'_i| \leq \|\theta-\theta'\|_1.
\end{equation}
Applying Kantorovich duality, we obtain the bound:
\begin{equation}
\sup_{\psi \in \text{Lip}(\R)} \left| \sum_{i = 1}^D \left( \E_{\rmq^*_{\lambda_n}}[\psi(\theta_i)]-\E_{m_{\lambda_n}^*}[\psi(\theta_i)] \right)\right| \leq W_1(\rmq^*_{\lambda_n}, m_{\lambda_n}^*) = o(D).
\end{equation}
Since the bound in \Cref{proof:thm:theory-hd-1-eqn-2} does not depend on the value of $\bx^{(n)}$, we have
\begin{equation}\label{proof:thm:theory-hd-1-eqn-6}
\sup_{\bx^{(n)} \in \mathbb{X}^n} \left|\frac{1}{D} \sum_{i = 1}^D \left( \E_{\rmq^*_{\lambda_n}}[\psi(\theta_i)]-\E_{m_{\lambda_n}^*}[\psi(\theta_i)] \right)\right|\leq \frac{1}{D}W_1(\rmq^*_{\lambda_n}, m_{\lambda_n}^*)  \overset{P_{\theta_0}}{\to} 0.
\end{equation}

Consider the regime $\lambda_n \prec D \Xi^{-1}(\rmq^*_0)$.  Recall that the the $\Xi$-VI has the Lagrangian formulation as $\min_{\Xi(\rmq) \leq r(\lambda_n)} \KL(\rmq \parallel \rmq^*_0)$ for some constant $r(\lambda_n)$ depending on $\lambda_n$. If $\Xi(\rmq^*_0) \leq r(\lambda_n)$, then $\rmq^*_{\lambda_n} = \rmq^*_0$, which implies $\Xi(\rmq^*_0) \geq \Xi(\rmq^*_{\lambda_n})$ for all $\lambda_n$. For fixed $n$, we have
\begin{align*}
    \KL(\rmq^*_{\lambda_n} \parallel \rmq^*_0)-\KL(\rmq^*_0 \parallel \rmq^*_0)   \leq \lambda_n (\Xi(\rmq^*_0)-\Xi(\rmq^*_{\lambda_n})) \leq \lambda_n \Xi(\rmq^*_0) = o(D).
\end{align*}
By the $T_1$-transportation inequality and Kantorovich duality, we have
\begin{equation}
\sup_{\bx^{(n)} \in \mathbb{X}^n} \left|\frac{1}{D} \sum_{i = 1}^D \left( \E_{\rmq^*_{\lambda_n}}[\psi(\theta_i)]-\E_{\rmq_0^*}[\psi(\theta_i)] \right)\right|\leq \frac{1}{D}W_1\left(\rmq^*_{\lambda_n}, \rmq_0^* \right)  \lesssim \sqrt{D \cdot  \KL(\rmq^*_{\lambda_n} \parallel \rmq^*_0)} \overset{P_{\theta_0}}{\to} 0.
\end{equation}
\end{proof}

\begin{proof}[Proof of \Cref{cor:theory-hd-1}]
Under the assumptions,
\begin{equation*}
    \sqrt{\sum_{i = 1}^D a c_{ii}/D} \lesssim D, \quad \sqrt{\sum_{i =1}^D b_i^2/D} \lesssim D, \quad \sqrt{\sum_{i, j} c_{ij}^2} \lesssim D.
\end{equation*}
When we plug these terms in the upper bounds~\eqref{thm:theory-hd-1}, \Cref{eq-thm:theory-hd-1} follows as the desired result.
\end{proof}
For the linear model, denote $\bw:=  \sigma^{-2}\bX^T \by$ and $d_i = [\bB_{\text{diag}}]_{ii}$, where $[\bB_{\text{diag}}]_{ii}$ is the $i^{th}$ diagonal entry  of matrix $\bB_{\text{diag}}$. The next result shows $\Xi$-VI respects log-concavity of the exact posterior.
\begin{lemma} \label{lemma:theory-gauss-1}
Let the assumptions of \Cref{thm:theory-gauss-1} hold. For \(\lambda_n \in \bar \R_+\), the solution \(\rmq^*_{\lambda_n}\) to \Cref{def:theory-setup} is $(\kappa_1 + \kappa_2)$-log-concave. Moreover, for each $i$, the optimal EOT potential $\phi^*_{\lambda_n, i}$ is $\kappa_2/(\lambda_n +1)$-convex and marginal $m^*_{\lambda_n, i}$ is $(\kappa_1 + \kappa_2)$-log-concave.
\end{lemma}

\begin{proof}[Proof of \Cref{lemma:theory-gauss-1}]
We first prove existence. By Lagrangian duality,  $\Xi$-VI (\Cref{def:theory-setup}) is equivalent to $\min_{\Xi(\rmq) \leq r(\lambda_n)} \KL(\rmq \parallel \rmq^*_0)$. An optimizer of the latter problem exists because $\Xi(\cdot)$ has weakly closed sublevel set in $\P_2(\Theta)$ and because $\KL(\cdot \parallel \rmq^*_0)$ has weakly compact sub-level sets.

Recall the $\Xi$-variational posterior be represented in term of optimal marginals $m^*_{\lambda_n}$ and optimal EOT potentials $\phi^*_{\lambda_n}$:
\begin{equation} \label{proof:lemma-theory-gauss-1-1}
\rmq^*_{\lambda_n} (\theta) = \exp\left(\sum_{i = 1}^D \phi^*_{\lambda_n,i}(\theta_i) + \frac{1}{\lambda_n + 1} \ell (\bx^{(n)}; \theta) \right) m^*_{\lambda_n}(\theta).
\end{equation}
By \Cref{prop:fixed-point}, $m^*_{\lambda_n}$ and $\phi^*_{\lambda_n}$ satisfy the following fixed point equations:
\begin{equation}\label{proof:lemma-theory-gauss-1-2}
\begin{aligned}
m^*_{\lambda_n, i}(\theta_i) &= Z_i^{-1} \exp(- (\lambda_n + 1) \phi^*_{\lambda_n, i}(\theta_i)) \pi_i(\theta_i), \quad \text{and } \quad \\
\hat  \phi_{\lambda_n, i}(\theta_i) &=-\log \int_{\Theta_{-i}} \exp\left(\frac{1}{\lambda_n + 1} \ell\left(\bx^{(n)}; \theta\right)- \lambda_n \sum_{j \neq i}\phi^*_{\lambda_n, j}(\theta_j) \right) \prod_{j \neq i} \pi_j(\theta_j) d \theta_{-i} \\
&+ \sum_{j \neq i} \log \int_{\Theta_j} \exp\left(- (\lambda_n + 1) \phi^*_{\lambda, j}(\theta_j) \right) \pi_j(\theta_j) d\theta_j.
\end{aligned}
\end{equation}
 Using equations \Cref{proof:lemma-theory-gauss-1-2} to replace $m^*_{\lambda_n}$ in \Cref{proof:lemma-theory-gauss-1-1}, the variational posterior $\rmq^*_{\lambda_n}$ satisfies
\begin{equation}\label{proof:lemma-theory-gauss-1-3}
  \rmq^*_{\lambda_n} (\theta) \propto \exp\left( \frac{1}{\lambda_n + 1} \ell (\bx^{(n)}; \theta) - \lambda_n \sum_{i = 1}^D \phi^*_{\lambda_n,i}(\theta_i)\right) \pi(\theta).
\end{equation}
We now establish the log-concavity of $\rmq^*_{\lambda_n}$. Applying \Cref{proof:lemma-theory-gauss-1-2} to \Cref{proof:lemma-theory-gauss-1-3}, we get
\begin{equation}\label{proof:lemma-theory-gauss-1-4}
  \rmq^*_{\lambda_n} (\theta) \propto \exp\left( \frac{1}{\lambda_n + 1} \ell (\bx^{(n)}; \theta) + \lambda_n \sum_{i = 1}^D  \log \E_{\pi_{-i}} \exp\left(\frac{1}{\lambda_n + 1} \ell\left(\bx^{(n)}; \theta\right)- \lambda_n \sum_{j \neq i}\phi^*_{\lambda_n, j}(\theta_j) \right) \right) \pi(\theta).
\end{equation}
For $\alpha \in [0,1]$ and $\theta_i^0,\theta_i^1  \in \Theta_i$, we have
\begin{align*}
&-\hat  \phi_{\lambda_n, i}(\alpha \theta_i^0 + (1 - \alpha)\theta_i^1 )  = \log \E_{\pi_{-i}} \left[ \exp \left(\frac{1}{\lambda_n + 1} \ell\left(\bx^{(n)};\alpha \theta_i^0 + (1 - \alpha)\theta_i^1, \theta_{-i}\right)- \lambda_n \sum_{j \neq i}\phi^*_{\lambda_n, j}(\theta_j) \right) \right] + C.
\end{align*}
The log-likelihood is $\kappa_2$-concave, thus
\begin{equation*}
\ell\left(\bx^{(n)};\alpha \theta_i^0 + (1 - \alpha)\theta_i^1, \theta_{-i}\right) \geq \alpha \ell\left(\bx^{(n)};\theta_i^0 , \theta_{-i}\right) + (1 - \alpha) \ell\left(\bx^{(n)};\theta_i^1,  \theta_{-i}\right) + \frac{\kappa_2 \alpha (1 - \alpha)}{2} (\theta_i^0 - \theta_i^1)^2.
\end{equation*}
By the Pr\'ekopa–Leindler inequality (Theorem 19.16, \cite{Villani2009}), we have
\begin{align*}
&\E_{\pi_{-i}} \left[ \exp \left(\frac{\alpha \ell\left(\bx^{(n)}; \theta_i^0 , \theta_{-i}\right) + (1 - \alpha) \ell\left(\bx^{(n)}; \theta_i^1, \theta_{-i}\right)}{\lambda_n + 1} - \lambda_n \sum_{j \neq i}\phi^*_{\lambda_n, j}(\theta_j) \right) \right] \\
&\geq \E_{\pi_{-i}} \left[ \exp \left(\frac{ \ell\left(\bx^{(n)}; \theta_i^0 , \theta_{-i}\right)}{\lambda_n + 1} - \lambda_n \sum_{j \neq i}\phi^*_{\lambda_n, j}(\theta_j) \right) \right]^{\alpha} \E_{\pi_{-i}} \left[ \exp \left(\frac{ \ell\left(\bx^{(n)}; \theta_i^1, \theta_{-i}\right)}{\lambda_n + 1} - \lambda_n \sum_{j \neq i}\phi^*_{\lambda_n, j}(\theta_j) \right) \right]^{1 - \alpha}.
\end{align*}
Since the logarithmic function is concave, we conclude
\begin{align*}
&-\hat  \phi_{\lambda_n, i}(\alpha \theta_i^0 + (1 - \alpha)\theta_i^1 ) \\
&\underset{\text{const}}{\geq} \log \E_{\pi_{-i}} \left[ \exp \left(\frac{\alpha \ell\left(\bx^{(n)}; \theta_i^0 , \theta_{-i}\right) + (1 - \alpha) \ell\left(\bx^{(n)}; \theta_i^1, \theta_{-i}\right)}{\lambda_n + 1} - \lambda_n \sum_{j \neq i}\phi^*_{\lambda_n, j}(\theta_j) \right) \right] + \frac{\kappa_2 \alpha (1 - \alpha)}{2(\lambda_n +1)} (\theta_i^0 - \theta_i^1)^2 \\
&\geq - \alpha \phi_{\lambda_n, i}(\theta_i^0) -(1 - \alpha) \phi_{\lambda_n, i}(\theta_i^1) + \frac{\kappa_2 \alpha (1 - \alpha)}{2(\lambda_n +1)} (\theta_i^0 - \theta_i^1)^2.
\end{align*}
Thus, the function $ -\hat  \phi_{\lambda_n, i}(\cdot)$ is $\kappa_2/(\lambda_n+1)$-concave.  By the fixed point representation~\eqref{proof:lemma-theory-gauss-1-2}, $m^*_{\lambda_n}$ is $(\kappa_2 + \kappa_1)$-log-concave.  Using the representation~\eqref{proof:lemma-theory-gauss-1-3}, we conclude that the distribution $\rmq^*_{\lambda_n}$ is $(\kappa_2 + \kappa_1)$-log-concave.
\end{proof}

We introduce some notations to streamline the two subsequent proofs.
\begin{definition}[Nonlinear quadratic tilt] \label{def:theory-gauss-tilt}
Let $\mu$ be a probability measure on $\R$. For $(\phi, \gamma) \in L_1(\R) \in (0, \infty)$, set
\begin{equation} \label{nonlinear-tilt-c}
    c_{\mu}(\phi, \gamma) := \log \left[ \int \exp\left( - \phi(\theta) -\frac{\gamma}{2} \theta^2 \right) d \mu(\theta)\right],
\end{equation}
and define the probability distribution $\mu_{\phi, \gamma}$ on $\R$ by setting
\begin{equation}\label{nonlinear-tilt}
    \mu_{\phi, \gamma} (\theta) := \exp\left(-\phi(\theta) -\frac{\gamma}{2} \theta^2-c_\mu(\phi, \gamma) \right) \mu(\theta), \quad \forall \theta \in \R.
\end{equation}
\end{definition}
For any probability measure $\mu$, we have $c_\mu(\phi, \gamma) < \infty$ for any $(\phi, \gamma) \in L_1(\R) \in (0, \infty)$. Given the base measure $\mu$,  the tilted measure $\mu_{\phi, \gamma}(\theta)$ has an exponential family density that has $(\phi(\theta), \theta^2)$ as the sufficient statistics.  We call $\mu_{\phi, \gamma}$ a \textit{nonlinear quadratic tilt} of $\mu$.

Using \Cref{thm:MEOT-struct} and \Cref{prop:fixed-point}, we have
\begin{equation}\label{eq:theory-gauss-5}
 \rmq^*_{\lambda_n}(\theta) =  \frac{1}{\cZ_D(\lambda_n)} \exp\left(-\frac{\theta^T \bB_{\text{off}} \theta}{2 (\lambda_n + 1)} + \bw^T \theta + \sum_{i = 1}^D c_{\pi_i}\left(\lambda_n \phi^*_{\lambda_n, i}, \frac{d_i}{\lambda_n +1} \right)\right) \prod_{i = 1}^D \pi_{i, \lambda_n \phi^*_{\lambda_n, i}, \frac{d_i}{\lambda_n +1}}(\theta_i),
\end{equation}
Here $ c_{\pi_i}\left(\lambda_n \phi^*_{\lambda_n, i}, \frac{d_i}{\lambda_n +1}\right)$ is defined in \Cref{nonlinear-tilt-c}, and $\pi_{i, \lambda_n \phi^*_{\lambda_n, i}, \frac{d_i}{\lambda_n +1}}$ is the nonlinear quadratic tilt of $\pi_i$ with parameters $(\lambda_n \phi^*_{\lambda_n, i}, \frac{d_i}{\lambda_n +1})$. The constant $\cZ_D(\lambda_n)$ is defined as:
\begin{equation*}
    \cZ_D(\lambda_n) := \int_\Theta \exp\left(-\frac{\theta^T \bB_{\text{off}} \theta }{2 (\lambda_n + 1)} + \bw^T \theta  \right)\prod_{i = 1}^D \pi_{i, \lambda_n \phi^*_{\lambda_n, i}, \frac{d_i}{\lambda_n +1}}(\theta_i) d\theta.
\end{equation*}

When $\lambda_n = 0$,  $\cZ_D(0)$ is the normalizing constant of the exact posterior. When $\lambda_n > 0$, we can view $\cZ_D(\lambda_n)$ as an approximation to $\cZ_D(0)$.

The log-concavity of  $\rmq^*_{\lambda_n}$ implies an upper bound of $\Xi(\rmq^*_{\lambda_n})$ using the covariance matrix, the design matrix, and the regularization parameter.

\begin{lemma}\label{lemma:theory-gauss-2}
Let the assumptions of \Cref{thm:theory-gauss-1} hold. The solution \(\rmq^*_{\lambda_n}\) to \Cref{eq:theory-gauss-1} satisfies
\begin{align*}
    \Xi(\rmq^*_{\lambda_n}) \leq  \frac{\text{tr}\left(\text{Cov}_{\rmq^*_{\lambda_n}} \left(\bB_{\text{off}} \theta \right) \right)}{4(\kappa_1 + \kappa_2)(\lambda_n  + 1)^2}.
\end{align*}
\end{lemma}

\begin{proof}[Proof of \Cref{lemma:theory-gauss-2}]
Any constant shift in $\rmq^*_{\lambda_n}$ is preserved by its marginal distribution $m^*_{\lambda_n}$. Since the KL divergence is invariant to constant shift, $\Xi(\rmq^*_{\lambda_n})$ is the same if we shift $\rmq^*_{\lambda_n}$ by a constant. WLOG, we can assume that $\E_{\rmq^*_{\lambda_n}}\left[\theta \right]= 0$.

Let $\tilde m_i(\theta_i) \propto \exp \left(- \frac{d_i}{2} \theta_i^2 +  \bw_i \theta_i \right) \pi(\theta_i)$, and $\tilde m(\theta) = \prod_{i = 1}^D \tilde m_i(\theta_i)$. By the variational representation of mutual information, we have
\begin{equation*}
     \Xi(\rmq^*_{\lambda_n}) = \KL(\rmq^*_{\lambda_n} \parallel m^*_{\lambda_n}) \leq  \KL(\rmq^*_{\lambda_n} \parallel \tilde m).
\end{equation*}

By \Cref{lemma:theory-gauss-1}, $\phi^*_{\lambda_n, i}$ is $\kappa_2/(\lambda_n+1)$-convex. Since $\pi$ is $\kappa_1$-log-concave, $\tilde m$ is $(\kappa_1 + \kappa_2)$-log-concave. By the log–Sobolev inequality, we have:
\begin{equation*}
 \Xi(\rmq^*_{\lambda_n}) \leq  \KL(\rmq^*_{\lambda_n} \parallel \tilde m) \leq \frac{1}{\kappa_1 +\kappa_2} \int_\Theta \left\|\nabla_\theta \left( \frac{\theta^T \bB_{\text{diag}} \theta - \theta^T \bB \theta}{2(\lambda_n + 1)}   \right) \right\|_2^2 \rmq^*_{\lambda_n}(\theta) d \theta = \frac{\E_{\rmq^*_{\lambda_n}}\left[\left\|\bB_{\text{off}} \theta \right\|_2^2 \right]}{4(\kappa_1 + \kappa_2)(\lambda_n + 1)^2} .
\end{equation*}

Under the assumed constraint $\E_{\rmq^*_{\lambda_n}}[\theta] = 0$, we conclude with the desired inequality:
\begin{equation*}
    \Xi(\rmq^*_{\lambda_n})  \leq \frac{\text{tr}\left(\text{Cov}_{\rmq^*_{\lambda_n}} \left(\bB_{\text{off}} \theta \right) \right)}{4(\kappa_1 + \kappa_2)(\lambda_n  + 1)^2}.
\end{equation*}
\end{proof}

\begin{proof}[Proof of \Cref{thm:theory-gauss-1}]

Define $\tilde \bB_{\text{off}} := \frac{\bB_{\text{off}}}{\lambda_n + 1}$ and $\tilde \pi_i := \pi_{i, \lambda_n \phi^*_{\lambda_n, i}, \frac{d_i}{\lambda_n +1}}(\theta_i)$.

We can write
\begin{align*}
    \rmq^*_{\lambda_n}(\theta) =  \frac{1}{\cZ_D(\lambda_n)} \exp\left(-\frac{\theta^T \bB_{\text{off}} \theta}{2(\lambda_n + 1)} + \bw^T \theta + \sum_{i = 1}^D c_{\pi_i}\left(\lambda_n \phi^*_{\lambda_n, i}, \frac{d_i}{\lambda_n +1} \right)\right) \prod_{i = 1}^D  \tilde \pi_i(\theta_i).
\end{align*}
By \Cref{lemma:theory-gauss-1}, $\rmq^*_{\lambda_n}$ is a $(\kappa_1 + \kappa_2)$-log-concave. By Theorem 1 of \cite{Lacker2022}, we have:
\begin{equation} \label{proof:thm:theory-gauss-1-eqn-1}
\log \cZ_D(\lambda_n)-\sup_{m \in \M(\Theta)} \left[- \frac{1}{2} \E_m[\theta]^T \tilde \bB_{\text{off}} \E_m[\theta] +\bw^T \E_m[\theta] -\sum_{i = 1}^D \KL(m_i \parallel  \tilde \pi_i )\right] \leq \frac{\sum_{j = 1}^D \sum_{i = 1}^D [\bB_{\text{off}}]_{ij}^2}{(\kappa_1+ \kappa_2)^2 (\lambda_n +1)^2} ,
\end{equation}
where $\E_m[\theta]$ is the mean vector of $m$.

For any $m \in \M(\Theta)$, we have
\begin{align*}
\KL(m \parallel \rmq^*_{\lambda_n}) &= \int_\Theta m(\theta) \left[ \log \cZ_D(\lambda_n) + \frac{1}{2}\theta^T \tilde \bB_{\text{off}} \theta-\bw^T \theta  + \log \frac{m(\theta)}{\prod_{i = 1}^D  \tilde \pi_i(\theta_i)} \right] d \theta.  \\
&=  \log \cZ_D(\lambda_n) + \frac{1}{2} \E_m[\theta]^T \tilde \bB_{\text{off}} \E_m[\theta]-\bw^T \E_m[\theta]  + \sum_{i = 1}^D \KL(m_i \parallel \tilde \pi_i).
\end{align*}
We invoke the upper bound on the log normalizer \Cref{proof:thm:theory-gauss-1-eqn-1}:
\begin{equation} \label{proof:thm:theory-gauss-1-eqn-2}
    \inf_{m \in \M(\Theta)}  \KL(m \parallel \rmq^*_{\lambda_n}) \leq \frac{\sum_{j = 1}^D \sum_{i = 1}^D [\bB_{\text{off}}]_{ij}^2}{(\kappa_1+ \kappa_2)^2 (\lambda_n +1)^2} .
\end{equation}
By the $T_2$-transportation inequality (Theorem 1 and 2, \cite{Otto2000}), we upper bound the Wasserstein metric with the square root of KL divergence:
\begin{equation} \label{proof:thm:theory-gauss-1-eqn-3}
    \inf_{m \in \M(\Theta)} W_2(\rmq^*_{\lambda_n}, m) \leq \sqrt{\frac{2}{\kappa_2 + \kappa_1}   \inf_{m \in \M(\Theta)}  \KL(m \parallel \rmq^*_{\lambda_n}) } \leq \sqrt{\frac{2\sum_{j = 1}^D \sum_{i = 1}^D [\bB_{\text{off}}]_{ij}^2}{(\kappa_1+ \kappa_2)^3 (\lambda_n + 1)^2} }.
\end{equation}
For $\lambda_n \succ \sqrt{\text{tr}(\bB_{\text{off}}^2)}$, we have $ \inf_{m \in \M(\Theta)} W_2(\rmq^*_{\lambda_n}, m) \overset{P_{\theta_0}}{\to} 0$.

Consider the second regime $\lambda_n \succ \sqrt{\text{tr}(\bB_{\text{off}}^2)/D}$. By the triangle inequality, we have
\begin{equation}
\begin{aligned}
&\sup_{\by \in \R^n} \inf_{m \in \M(\Theta)} \E_{\rmq^*_{\lambda_n}} \left[ \left(\frac{1}{D} \sum_{i = 1}^D \psi(\theta_i)-\frac{1}{D} \sum_{i = 1}^D \E_m[\psi(\theta_i)] \right)^2\right]^{1/2}\\
&\leq  \sup_{\by \in \R^n}  \E_{\rmq^*_{\lambda_n}} \left[ \left(\frac{1}{D} \sum_{i = 1}^D \psi(\theta_i)-\frac{1}{D} \sum_{i = 1}^D \E_{\rmq^*_{\lambda_n}}[\psi(\theta_i)] \right)^2\right]^{1/2} + \inf_{m \in \M(\Theta)}\left|\frac{1}{D} \sum_{i = 1}^D \left( \E_{\rmq^*_{\lambda_n}}[\psi(\theta_i)]-\E_m[\psi(\theta_i)] \right)\right| .
\end{aligned}
\end{equation}
Since $\psi$ is $1$-Lipschitz, we apply Kantorovich duality to bound the second term.
\begin{align*}
& \left( \inf_{m \in \M(\Theta)}\left|\frac{1}{D} \sum_{i = 1}^D \left( \E_{\rmq^*_{\lambda_n}}[\psi(\theta_i)]-\E_m[\psi(\theta_i)] \right)\right| \right)^2 \\
&\leq \inf_{m_1, \cdots m_D}\frac{1}{D} \sum_{i = 1}^D W_1^2 (\rmq^*_{\lambda_n, i}, m_i) \leq \inf_{m_1, \cdots m_D}\frac{1}{D} \sum_{i = 1}^D W_2^2 (\rmq^*_{\lambda_n, i}, m_i).
\end{align*}
By the subadditivity inequality of Wasserstein distance and \Cref{proof:thm:theory-gauss-1-eqn-3},  we have
\begin{equation}
   \inf_{m_1, \cdots m_D}  \sum_{i = 1}^D W_2^2 (\rmq^*_{\lambda_n, i}, m_i) \leq \inf_{m \in \M(\Theta)}  W_2^2(\rmq^*_{\lambda_n}, m) \leq \frac{2 \sum_{j = 1}^D \sum_{i = 1}^D [\bB_{\text{off}}]_{ij}^2}{(\kappa_1+ \kappa_2)^3(\lambda_n + 1)^2}
\end{equation}
Thus,
\begin{equation}\label{proof:thm:theory-gauss-1-eqn-4}
     \inf_{m \in \M(\Theta)}\left|\frac{1}{D} \sum_{i = 1}^D \left( \E_{\rmq^*_{\lambda_n}}[\psi(\theta_i)]-\E_m[\psi(\theta_i)] \right)\right| \leq \sqrt{\frac{2 \sum_{j = 1}^D \sum_{i = 1}^D [\bB_{\text{off}}]_{ij}^2}{D(\kappa_1+ \kappa_2)^3(\lambda_n + 1)^2} }.
\end{equation}
The Lipschitzness implies $\|\nabla \psi\|_2 \leq 1$.  To bound the first term, we apply Poincar\'e inequality to the function $x \mapsto \frac{1}{D} \sum_{i = 1}^D \psi(x_i)$.
\begin{equation}\label{proof:thm:theory-gauss-1-eqn-5}
\begin{aligned}
    &\sup_{\by \in \R^n}  \E_{\rmq^*_{\lambda_n}} \left[\left(\frac{1}{D} \sum_{i = 1}^D \psi(\theta_i)-\frac{1}{D} \sum_{i = 1}^D \E_{\rmq^*_{\lambda_n}}[\psi(\theta_i)] \right)^2\right] \leq  \sup_{\by \in \R^n} \text{Var}_{\rmq^*_{\lambda_n}} \left(  \frac{1}{D} \sum_{i = 1}^D \psi(\theta_i)\right) \\
    &\leq \frac{1}{(\kappa_1 + \kappa_2) D^2} \sum_{i = 1}^D \E_{\rmq^*_{\lambda_n}}  \left[\|\nabla \psi\|_2^2 \right] \leq \frac{1}{D(\kappa_1 + \kappa_2)}.
\end{aligned}
\end{equation}
Combining bounds \Cref{proof:thm:theory-gauss-1-eqn-4} and \Cref{proof:thm:theory-gauss-1-eqn-5}, we have
\begin{equation}
\begin{aligned}
&\sup_{\by \in \R^n} \inf_{m \in \M(\Theta)} \E_{\rmq^*_{\lambda_n}} \left[ \left(\frac{1}{D} \sum_{i = 1}^D \psi(\theta_i)-\frac{1}{D} \sum_{i = 1}^D \E_m[\psi(\theta_i)] \right)^2\right] \\
& \leq \frac{(\kappa_1 + \kappa_2)^2 (\lambda_n + 1)^2 + 2 \sum_{j = 1}^D \sum_{i = 1}^D [\bB_{\text{off}}]_{ij}^2}{D(\kappa_1 + \kappa_2)^3 (\lambda_n + 1)^2}
\end{aligned}
\end{equation}
For $\lambda_n \succ \sqrt{\text{tr}(\bB_{\text{off}}^2)/D}$, the bounds implies the \Cref{eq-thm:theory-gauss-2}.

Consider the third regime $\lambda_n \prec (\kappa_1 + \kappa_2)\left[\text{tr}\left(\text{Cov}_{\rmq^*_0} \left(\bB_{\text{off}} \theta \right) \right)\right]^{-1}$.  Recall that $\Xi$-VI has a dual problem of the form $\min_{\Xi(\rmq) \leq r(\lambda_n)} \KL(\rmq \parallel \rmq^*_0)$ for some constant $r(\lambda_n)$ depending on $\lambda_n$. If $\Xi(\rmq^*_0) \leq r(\lambda_n)$, then $\rmq^*_{\lambda_n} = \rmq^*_0$, hence $\Xi(\rmq^*_0) \geq \Xi(\rmq^*_{\lambda_n})$. For $t < \lambda_n$ and fixed $n$, we apply \Cref{lemma:theory-gauss-2} to obtain an upper bound,
\begin{align*}
    \KL(\rmq^*_{\lambda_n} \parallel \rmq^*_0)-\KL(\rmq^*_t \parallel \rmq^*_0)   \leq \lambda_n (\Xi(\rmq^*_t)-\Xi(\rmq^*_{\lambda_n})) \leq \lambda_n \Xi(\rmq^*_0) \overset{P_{\theta_0}}{\to} 0.
\end{align*}
Finally, consider the fourth regime $\lambda_n \prec D(\kappa_1 + \kappa_2)\left[\text{tr}\left(\text{Cov}_{\rmq^*_0} \left(\bB_{\text{off}} \theta \right) \right)\right]^{-1}$. We follow an analogous derivation as the third regime:
\begin{align*}
    \KL(\rmq^*_{\lambda_n} \parallel \rmq^*_0)-\KL(\rmq^*_t \parallel \rmq^*_0)   \leq \lambda_n (\Xi(\rmq^*_t)-\Xi(\rmq^*_{\lambda_n})) \leq \lambda_n \Xi(\rmq^*_0) = o(D).
\end{align*}
Since $\rmq^*_0$ is $(\kappa_1 + \kappa_2)$-log-concave, we invoke the $T_2$-transportation inequality:
\begin{equation*}
     W_2(\rmq^*_{\lambda_n}, \rmq^*_0) \leq \sqrt{\frac{2}{\kappa_2 + \kappa_1}  \KL(\rmq^*_{\lambda_n} \parallel \rmq^*_0)} .
\end{equation*}
The proofs for the third and fourth regimes are the same as the first two regimes, where we plug in the upper bounds for the KL divergence to upper-bound the Wasserstein distance. We skip repeating the details.
\end{proof}

\begin{proof}[Proof of \Cref{cor:theory-gauss-1}]
Given $\text{tr}(\bB_{\text{off}}^2) = \sum_{i = 1}^D \eta_i^2$, \Cref{thm:theory-gauss-1} ensures that the convergence of $W_2(\rmq^*_{\lambda_n},  m_{\lambda_n}^*)$ holds for $\lambda_n \succ \sqrt{\sum_{i = 1}^D \eta_i^2}$. Since $\sum_{i = 1}^D \eta_i^2 \lesssim D^2$, we have $W_2(\rmq^*_{\lambda_n},  m_{\lambda_n}^*)$ converges in probability to zero, for any choice of $\lambda_n \succ D$.
\end{proof}

\subsection*{Proofs of \Cref{sect-theory-mf-conv}}\label{proof-theory-mf-conv}
We first state an auxiliary lemma to \Cref{thm:mf-convergence}.
\begin{lemma} \label{lemma:mf-convergence-1}
Let $\rmq^*_\lambda$ be the $\Xi$-variational posterior. Then $\text{ELBO}(\rmq^*_\lambda)$ and $\C_\lambda$ are monotonically decreasing function of $\lambda$.
\end{lemma}

\begin{proof}[Proof of \Cref{lemma:mf-convergence-1}]
    Since $\rmq^*_\lambda$ is a maximizer of $\text{ELBO}(\rmq)-\lambda \Xi(\rmq)$, we have
    \begin{equation*}
        \C_{\lambda} = \text{ELBO}(\rmq^*_{\lambda})-\lambda \Xi(\rmq^*_{\lambda}).
    \end{equation*}
    For $\lambda_1 < \lambda_2$, we have
    \begin{equation*}
        \C_{\lambda_1}  = \text{ELBO}(\rmq^*_{\lambda_1})-\lambda_1 \Xi(\rmq^*_{\lambda_1}) \geq \text{ELBO}(\rmq^*_{\lambda_2})-\lambda_1 \Xi(\rmq^*_{\lambda_2})  \geq \text{ELBO}(\rmq^*_{\lambda_2})-\lambda_2 \Xi(\rmq^*_{\lambda_2}) = \C_{\lambda_2}.
    \end{equation*}
By Lagrangian duality, we have $\text{ELBO}(\rmq^*_\lambda) =  \max\limits_{\Xi(\rmq) \leq t(\lambda)}  \text{ELBO}(\rmq)$ for $t(\lambda)$ monotonically decreasing in $\lambda$.

For $\lambda_1 < \lambda_2$, $t(\lambda_1) \geq t(\lambda_2)$ hence
    \begin{equation*}
   \text{ELBO}(\rmq^*_{\lambda_1})  =\max\limits_{\Xi(\rmq) \leq t(\lambda_1)}  \text{ELBO}(\rmq) \geq \max\limits_{\Xi(\rmq) \leq t(\lambda_2)}  \text{ELBO}(\rmq) =  \text{ELBO}(\rmq^*_{\lambda_2}).
\end{equation*}
\end{proof}

\begin{proof}[Proof of \Cref{thm:mf-convergence}]
Let $(P_2(\Theta), W_2)$ be the metric space.  We want to show that the functionals
\begin{equation*}
    F_\lambda(\rmq) :=  \KL (\rmq \parallel \rmq^*_0) + \lambda \Xi(\rmq).
\end{equation*}
$\Gamma$-converge to
\begin{equation*}
    F_\infty(\rmq) := \KL (\rmq \parallel \rmq^*_0) + \infty \Xi(\rmq),
\end{equation*}
as $\lambda \to \infty$.

To verify $\Gamma$ convergence, we make use of the property that the KL divergence functional $\KL(\cdot \parallel \rmq^*_0)$ and $\Xi(.)$ functional are lower semicontinuous (l.s.c.) in Wasserstein metric. This is provided in \Cref{lemma:lsc-kl-xi}.

Let $\rmq \in\P_2(\Theta)$ and $W_2(\rmq_\lambda,  \rmq) \to 0$.  If $\rmq$ is a product measure, then
\begin{align*}
F_\infty(\rmq) &= \KL (\rmq \parallel \rmq^*_0) \leq \liminf_{\lambda \to \infty} \KL (\rmq_\lambda \parallel \rmq^*_0) \leq \liminf_{\lambda \to \infty} F_\lambda(\rmq_\lambda).
\end{align*}
The first inequality holds because $\KL (. \parallel \rmq^*_0)$ is l.s.c.

If $\rmq$ is not a product measure, we have $\liminf_{n \to \infty} \Xi(\rmq_n) \geq \Xi(\rmq) > 0$ by the lower semicontinuity of $\Xi$. Since the KL term is nonnegative, we have
\begin{equation*}
    F_\infty(\rmq) = \infty = \liminf_{\lambda \to \infty} F_\lambda(\rmq_\lambda).
\end{equation*}
Thus the liminf inequality is verified.

Next we show the existence of a recovery sequence.  For any $\rmq \in\P_2(\Theta)$, we take $\rmq_\lambda =  \rmq$. If $\rmq$ is a product measure, then
\begin{equation*}
        F_\infty(\rmq) = \KL(\rmq \parallel \rmq^*_0) \geq  \KL(\rmq \parallel \rmq^*_0).
\end{equation*}
Otherwise,
\begin{equation*}
        F_\infty(\rmq)  = \infty \geq \limsup_{\lambda \to \infty} F_\lambda(\rmq_\lambda).
\end{equation*}
This verifies the limsup inequality.  Combining the liminf and limsup inequalities, we obtain that $F_\infty = \Gamma-\text{lim}_{\lambda \to \infty} F_\lambda$.

Next we prove that the sequence $F_\lambda$ is eqi-coercive. Take $\lambda_j \to \infty$ and $\rmq_{\lambda_j}$ such that $F_{\lambda_j}(\rmq_{\lambda_j}) \leq t$ for all $j$.  Then $\Xi(\rmq_{\lambda_j}) = o(1)$ because $\lambda_j \Xi(\rmq_{\lambda_j})$ is bounded as $\lambda_j \to \infty$.  Moreover, $\KL(\rmq_{\lambda_j} \parallel \rmq^*_0)$ is upper bounded by $t$. Since $\KL(. \parallel \rmq^*_0)$ is Wasserstein (geodesically) convex, it is coercive by Lemma 2.4.8 of \cite{Ambrosio2005}. Thus, there exists a converging sequence $\rmq_{\lambda_j}'$ such that $\KL(\rmq_{\lambda_j}' \parallel \rmq^*_0) \leq \KL(\rmq_{\lambda_j} \parallel \rmq^*_0) + o(1)$.  Since $\Xi(\rmq_{\lambda_j}) = o(1)$, we obtain that $F_{\lambda_j}(\rmq_j') \leq F_{\lambda_j}(\rmq_j) + o(1)$.  This verifies the equi-coercivity of $F_\lambda$.

Finally, by the fundamental theorem of $\Gamma$ convergence (\Cref{thm:gamma-cvg}), we conclude that
\begin{equation*}
    W_2(\rmq^*_\infty, \rmq^*_\lambda) \to 0, \quad \text{as} \quad \lambda \to \infty,
\end{equation*}
and
\begin{equation*}
   \left|\C_\lambda-\C_\infty \right| \to 0, \quad \text{as} \quad \lambda \to \infty.
\end{equation*}
By Corollary 2.1 of \cite{Braides2014}, every minimizer of $F_\infty$ is the limit of some converging minimizing sequences of $F_\lambda$. For any $\rmq^*_\infty \in \Q_\infty$, this implies the existence of a sequence $\rmq^*_\lambda \in \Q_\lambda$ such that
\begin{equation*}
    W_2(\rmq^*_\infty, \rmq^*_\lambda) \to 0, \quad \text{as} \quad \lambda \to \infty.
\end{equation*}
\end{proof}

\begin{proof}[Proof of \Cref{thm:posterior-convergence}]
We define $\P_2'(\Theta)$ as $P_2'(\Theta) = \{\rmq \in \P_2(\Theta): \Xi(\rmq) < \infty\}$.  The space $(\P_2'(\Theta), W_2)$ is a metric space. We want to show that the sequence of functionals
\begin{equation*}
    F_\lambda(\rmq) :=  \KL (\rmq \parallel \rmq^*_0) + \lambda \Xi(\rmq).
\end{equation*}
$\Gamma$-converge to
\begin{equation*}
    F_0(\rmq) := \KL (\rmq \parallel \rmq^*_0),
\end{equation*}
as $\lambda \to 0$.  Both $F_\lambda(\rmq)$ and $F_0(\rmq)$ are defined on $(\P_2'(\Theta), W_2)$.

We make use of \Cref{lemma:lsc-kl-xi} which shows that the KL divergence functional $\KL(\cdot \parallel \rmq^*_0)$ and $\Xi(.)$ functional are lower semicontinuous (l.s.c.) in Wasserstein metric.

Let $\rmq \in\P_2'(\Theta)$ and $W_2(\rmq_\lambda,  \rmq) \to 0$.  We have
\begin{align*}
F_0(\rmq) &= \KL (\rmq \parallel \rmq^*_0) \leq \liminf_{\lambda \to 0} \KL (\rmq_\lambda \parallel \rmq^*_0) \leq \liminf_{\lambda \to 0} F_\lambda(\rmq_\lambda).
\end{align*}
The first inequality holds because $\KL (. \parallel \rmq^*_0)$ is l.s.c.  The second inequality holds because $\Xi(\cdot)$ is nonnegative.

Next we show that the existence of a recovery sequence.  For any $\rmq \in\P_2(\Theta)$, we take $\rmq_\lambda = \rmq$. Since $\Xi(\rmq) < \infty$, we have
\begin{equation*}
        F_0(\rmq) = \KL(\rmq \parallel \rmq^*_0) \geq \limsup_{\lambda \to 0} \KL(\rmq \parallel \rmq^*_0) + \lambda \Xi(\rmq).
\end{equation*}
This verifies the limsup inequality.  Combining the liminf and limsup inequalities, we obtain that $F = \Gamma-\text{lim}_{\lambda \to \infty} F_\lambda$.

We proceed to establish equi-coercivity of the sequence \( F_\lambda \). Consider a sequence \( \lambda_j \to 0 \) and \( \rmq_{\lambda_j} \in\P_2'(\Theta) \) for which \( F_{\lambda_j}(\rmq_{\lambda_j}) \leq t \) holds for all \( j \). Given that \( \Xi(\rmq_{\lambda_j}) \geq 0 \), it follows that \( \KL(\rmq \parallel \rmq^*_0) \leq t \). Owing to the geodesic convexity of the Kullback-Leibler divergence \( \KL(\cdot \parallel \rmq^*_0) \) in the Wasserstein space, Lemma 2.4.8 from \cite{Ambrosio2005} ensures that it is coercive, implying that the set \( \{\rmq \in\P_2'(\Theta) \mid \KL(\rmq \parallel \rmq^*_0) \leq t\} \) is compact in the metric space \( (\P_2'(\Theta), W_2) \). Sequential compactness guarantees the existence of a convergent subsequence of \( \rmq_{\lambda_j} \), which converges to some \( \rmq_0 \) in \(\P_2'(\Theta) \). Since \( \KL(\cdot \parallel \rmq^*_0) \) is lower semicontinuous (l.s.c.), we conclude that:

\begin{align*}
    F_{\lambda_j}(\rmq_0) &= \KL(\rmq_0 \parallel \rmq^*_0) + \lambda_j \Xi(\rmq_0) \leq \KL(\rmq_{\lambda_j} \parallel \rmq^*_0) + \lambda_j \Xi(\rmq_0) \\
    &\leq F_{\lambda_j}(\rmq_{\lambda_j}) + \lambda_j \Xi(\rmq_0) = F_{\lambda_j}(\rmq_{\lambda_j}) + o(1),
\end{align*}

Finally, by the fundamental theorem of $\Gamma$ convergence, we conclude that
\begin{equation*}
    W_2(\rmq^*_0, \rmq^*_\lambda) \to 0, \quad \text{as} \quad \lambda \to 0,
\end{equation*}
where $\rmq_0$ is a minimizer of $F_0$,  and
\begin{equation*}
   \left|\C_\lambda-\C_0 \right| \to  0, \quad \text{as} \quad \lambda \to 0.
\end{equation*}
 Since $\rmq^*_0$ is the unique minimizer of $F_0$, we conclude that $\rmq_0  = \rmq^*_0$.

To prove the convergence of optimal cost, we note that
\begin{equation*}
    F_\lambda(\rmq^*_\lambda) \leq F_\lambda(\rmq^*_0) = F_0(\rmq^*_0) + \lambda \Xi(\rmq^*_0).
\end{equation*}
Thus,
\begin{equation*}
    |\C_\lambda-\C_0| = |F_\lambda(\rmq^*_\lambda)-F_0(\rmq^*_0)| \leq \lambda \Xi(\rmq^*_0).
\end{equation*}
\end{proof}

Define a functional $\Phi_\lambda$ that combines the objective functional of the inner variational objective problem and \Cref{assumption:lip-cost}:
\begin{equation} \label{def:stability1}
    \Phi_\lambda(\rmq) :=  \E_\rmq\left[-\ell(\bx^{(n)}; \theta) +  \sum_{i = 1}^D \phi_i(\theta_i)  \right] + (\lambda + 1)\Xi(\rmq),
\end{equation}
where $\phi_i: \Theta_i \mapsto \R$ are the one-dimensional function in the Lipschitz cost assumption (\Cref{assumption:lip-cost}).  Since $\phi_i$ are tensorized, minimizing $\Phi_\lambda$ over $\C(m)$ is equivalent to solving the inner variational problem over $\C(m)$.

For proving \Cref{thm:stability}, we introduce a Pythagorean theorem for the inner variational problem.
\begin{lemma}  \label{lemma:stability1}
Let $ \rmq_\lambda \in \C(m)$ be a optimizer of $\Phi_\lambda$ over $\C(m)$. Then
\[
\KL(\rmq, \rmq_\lambda) \leq \Phi_\lambda(\rmq)-\Phi_\lambda(\rmq_\lambda), \quad \text{for all }  \rmq \in  \C(m).
\]
\end{lemma}
\begin{proof}[Proof of \Cref{lemma:stability1}]
We recall definition of the auxiliary measure $\rmq_{\text{aux}}$ in the proof of \Cref{thm:MEOT-struct}, $\rmq_{\text{aux}}(\theta) = \alpha^{-1} \exp\left( \frac{\ell(\bx^{(n)}; \theta)-\sum_{i = 1}^D \phi_i(\theta_i)}{\lambda + 1} \right) m(\theta)$, where $\alpha$ is the normalizing constant. Then
\begin{equation} \label{eqn:lemma:stability-1}
    \Phi_\lambda(\rmq) = \KL(\rmq \parallel \rmq_{\text{aux}})-\log \alpha,
\end{equation}
so that the entropic optimal transport problem is equivalent to minimizing $\KL(\cdot \parallel \rmq_{\text{aux}})$. In particular, $\rmq_\lambda = \arg\min_{\C(m)}\KL(\rmq \parallel \rmq_{\text{aux}})$ and the Pythagorean theorem for relative entropy (Theorem 2.2, \citep{Csiszar1975}) yields
\begin{equation*}
    \KL(\rmq \parallel \rmq_{\text{aux}}) \geq \KL(\rmq_\lambda \parallel \rmq_{\text{aux}}) + \KL(\rmq \parallel \rmq_\lambda) \quad \text{for all }  \rmq \in \C(m).
\end{equation*}
In view of \Cref{eqn:lemma:stability-1}, the desired claim holds.
\end{proof}
The next Lemma is also auxiliary to the proof of \Cref{thm:stability}.
\begin{lemma} \label{lemma:stability2}
    Let $\rmq^*_\lambda \in \C(m^*)$ be a optimizer of $\Phi_\lambda$ over $\C(m^*)$, and $\rmq_\lambda^s \in \C(\tilde m)$ be its shadow. Then
    \begin{equation*}
        \left|\Phi_\lambda(\rmq^*_\lambda)-\Phi_\lambda(\rmq_\lambda^s) \right| \leq L W_2(\rmq^*_\lambda, \rmq_\lambda^s).
    \end{equation*}
\end{lemma}
\begin{proof}[Proof of \Cref{lemma:stability2}]
Using the Lipschitz cost assumption and \Cref{lemma:shadow1}, we have
\begin{align*}
    \Phi_\lambda(\rmq^*_\lambda) &= \E_{\rmq^*_\lambda}\left[-\ell(\bx^{(n)}; \theta) +  \sum_{i = 1}^D \phi_i(\theta_i)  \right] + (\lambda + 1)\Xi(\rmq^*_\lambda). \\
    &\geq \E_{\rmq_\lambda^s}\left[-\ell(\bx^{(n)}; \theta) +  \sum_{i = 1}^D \phi_i(\theta_i)  \right]-L W_2(\rmq^*_\lambda, \rmq_\lambda^s) + (\lambda + 1)\Xi(\rmq_\lambda^s) \\
    &= \Phi_\lambda(\rmq_\lambda^s)-L W_2(\rmq^*_\lambda, \rmq_\lambda^s).
\end{align*}
The claim follows by a symmetric argument.
\end{proof}

\begin{proof}[Proof of \Cref{thm:stability}]
 Consider the optimizers $\tilde \rmq_\lambda \in \C(\tilde m)$ and $\rmq^*_\lambda \in \C(m^*_\lambda)$. Let $\rmq^s_\lambda \in \C(\tilde m)$ be the shadow of $\rmq^*_\lambda$. By \Cref{lemma:shadow1} and the Lipschitz cost assumption, we have:
\begin{align*}
    \Phi_\lambda(\rmq_\lambda^s)-\Phi_\lambda(\rmq^*_\lambda) &\leq \int_\Theta \left(\ell(\bx^{(n)}; \theta)-\sum_{i = 1}^D \phi_i(\theta_i) \right)(\rmq_\lambda^s(\theta)-\rmq^*_\lambda(\theta)) d \theta \\
    &\leq L W_2(\rmq_\lambda^s,  \rmq^*_\lambda) \leq L W_2(m^*_\lambda, \tilde m).
\end{align*}

\Cref{lemma:stability2} implies $\Phi_\lambda(\tilde \rmq_\lambda)-\Phi_\lambda(\rmq^*_\lambda) \leq LW_2(m^*_\lambda, \tilde m)$. Adding the inequalities shows:
\begin{equation*}
    |\Phi_\lambda(\tilde \rmq_\lambda) - \Phi_\lambda(\rmq_\lambda^s)| \leq 2 L W_2(m^*_\lambda, \tilde m).
\end{equation*}
By \Cref{lemma:stability1}, we have that $\KL(\tilde{\pi}, \pi^{*}) \leq 2L W_2(m^*_\lambda, \tilde m)$, and the transport inequality assumption implies:
\begin{equation*}
    W_\rho(\rmq_\lambda^s, \tilde \rmq_\lambda) \leq C_\rho(2L W_2(m^*_\lambda, \tilde m))^{\frac{1}{2\rho}}.
\end{equation*}

By \Cref{lemma:shadow1}, we get $W_2(\rmq^*_\lambda, \rmq_\lambda^s)  = W_2(m^*_\lambda, \tilde m)$. We conclude the proof via the triangle inequality,
\begin{equation*}
    W_2(\rmq^*_\lambda, \tilde \rmq_\lambda) \leq W_2(\rmq^*_\lambda, \rmq_\lambda^s) + W_2(\rmq_\lambda^s, \tilde \rmq_\lambda) \leq W_2(m^*_\lambda, \tilde m) + C_\rmq (2L W_2(m^*_\lambda, \tilde m))^{\frac{1}{2 \rmq}}.
\end{equation*}
\end{proof}

\section*{Details of \Cref{example:multivariate-Gaussian}} \label{sect-proof-Gaussian}
\begin{proof}[Proof of \Cref{prop:mGaussian}]
Let the true precision
matrix be $\Lambda_0$. Let $f(\mu, \Sigma)$ be the objective function \Cref{eqn-mGaussian-1}
parameterized by the variational mean $\mu$ and covariance $\Sigma$. A
direct calculation shows
\begin{equation*}
f(\mu, \Sigma) = \frac{1}{2}\left[  (\mu_0-\mu)^T \Lambda_0
       (\mu_0-\mu) + \lambda \sum_{K = 1}^D \log
       \Sigma_{KK} + \text{tr} \left\{ \Lambda_0
       \Sigma \right\}-(\lambda + 1) \log |\Sigma|\right].
\end{equation*}
First, we confirm that the optimal $\mu^*$ is equal to the true
$\mu_0$.  The first-order optimality condition of $f$ with respect to $\mu$ yields
\begin{equation*}
    \partial_\mu f(\mu, \Sigma)  = 0 \implies \Lambda_0(\mu - \mu_0) = 0.
\end{equation*}
Since $\Sigma_0$ is full-rank, its inverse $\Lambda_0$ is full-rank hence the equality above yields $\mu^* = \mu_0$.

Now we turn to $\Sigma^*$. The first-order optimality yields the following characterization of the optimal precision matrix
$\Lambda^*$:
\begin{equation}
  \label{eqn-mGaussian-2}
\partial_{\Sigma} f(\mu, \Sigma) = 0  \implies \lambda (\Sigma_{\text{diag}}^{*})^{-1} + \Lambda_0 - (\lambda + 1) \Lambda= 0 \implies \Lambda^*  =  \frac{1}{\lambda + 1} \Lambda_0  +\frac{\lambda}{\lambda + 1} (\Sigma_{\text{diag}}^{*})^{-1},
\end{equation}
where $\Sigma_{\text{diag}}^*, \Sigma_{\text{off}}^*$ denote the
diagonal and off-diagonal minor of $\Sigma^*$, respectively.
By \Cref{eqn-mGaussian-2}, we have
\begin{equation*}
\begin{aligned}
 \Lambda^*_{\text{diag}} = \frac{1}{\lambda + 1} \Lambda_{0,diag} + \frac{\lambda}{\lambda + 1} (\Sigma^*_{\text{diag}})^{-1} &\implies (\lambda + 1)\Lambda^*_{\text{diag}}-\lambda  (\Sigma^*_{\text{diag}})^{-1} = \Lambda_{0, \text{diag}}
\end{aligned}
\end{equation*}
By \Cref{lemma:matrix-2}, we have the inequality
\begin{equation}\label{proof-mGaussian-1}
  \Sigma^{*^{-1}}_{\text{diag}}\leq \Lambda^*_{\text{diag}}  \leq \Lambda_{0, \text{diag}}.
\end{equation}
By the Woodbury identity (Eq.(156) in \cite{petersen2008matrix}) with $C = I_D$, we have
\begin{align*}
        \Sigma^*  &= \left[\frac{1}{\lambda + 1} \Lambda_0 + \frac{\lambda}{\lambda + 1} \Sigma^{*^{-1}}_{\text{diag}} \right]^{-1} = \frac{\lambda +1}{\lambda}  \Sigma^{*}_{\text{diag}}-\left[\frac{\lambda}{\lambda + 1} \Sigma^{*^{-1}}_{\text{diag}} + \frac{\lambda^2}{\lambda +1 }  \Sigma^{*^{-1}}_{\text{diag}}\Sigma_0  \Sigma^{*^{-1}}_{\text{diag}} \right]^{-1}.
\end{align*}
Taking the diagonal elements on both sides, we have
\begin{align*}
 \frac{1}{\lambda}  \Sigma^{*}_{\text{diag}} &= \left(\left[\frac{\lambda}{\lambda + 1} \Sigma^{*^{-1}}_{\text{diag}} + \frac{\lambda^2}{\lambda +1 }   \Sigma^{*^{-1}}_{\text{diag}}\Sigma_0 \Sigma^{*^{-1}}_{\text{diag}}  \right]^{-1}\right)_{\text{diag}} \\
 &= \left(\left[\underbrace{\frac{\lambda}{\lambda + 1} \Sigma^{*^{-1}}_{\text{diag}} + \frac{\lambda^2}{\lambda +1 }  \Sigma^{*^{-1}}_{\text{diag}} \Sigma_{0, \text{diag}} \Sigma^{*^{-1}}_{\text{diag}}}_{A} + \underbrace{\frac{\lambda^2}{\lambda +1 }  \Sigma^{*^{-1}}_{\text{diag}} \Sigma_{0, \text{off}}  \Sigma^{*^{-1}}_{\text{diag}}}_{B} \right]^{-1}\right)_{\text{diag}}.
\end{align*}
Note that $B$ is a matrix with zero diagonal entries.
By \Cref{lemma:matrix-2}, we have
\begin{align*}
     \frac{1}{\lambda}  \Sigma^{*}_{\text{diag}} = \left([A + B]^{-1}\right)_{\text{diag}} \geq [A + B]_{\text{diag}}^{-1}= A^{-1}.
\end{align*}
This implies that
\begin{equation*}
   \frac{1}{\lambda}  \Sigma^{*}_{\text{diag}} \geq  \Sigma^{*}_{\text{diag}}\left[ \frac{\lambda}{\lambda + 1} \Sigma^{*}_{\text{diag}} + \frac{\lambda^2}{\lambda +1 }  \Sigma_{0, \text{diag}} \right]^{-1} \Sigma^{*}_{\text{diag}},
\end{equation*}
which after simplification yields
\begin{equation} \label{proof-mGaussian-2}
     \Sigma^{*}_{\text{diag}} \leq \Sigma_{0, \text{diag}}.
\end{equation}
By Hua's identity, we have
\begin{align*}
        \Sigma^*  &= \left[\frac{1}{\lambda + 1} \Lambda_0 + \frac{\lambda}{\lambda + 1} \Sigma^{*^{-1}}_{\text{diag}} \right]^{-1} = (\lambda + 1) \Sigma_0-\left[\frac{1}{\lambda + 1} \Lambda_0 + \frac{1}{\lambda(\lambda + 1)} \Lambda_0  \Sigma^{*}_{\text{diag}} \Lambda_0 \right]^{-1}.
\end{align*}
It follows that
\begin{align*}
    \Sigma^* -\Sigma_0 &= \lambda \Sigma_0-\left[\frac{1}{\lambda + 1} \Lambda_0 + \frac{1}{\lambda(\lambda + 1)} \Lambda_0  \Sigma^{*}_{\text{diag}} \Lambda_0 \right]^{-1}.
\end{align*}
By \Cref{proof-mGaussian-2}, the matrix $\Sigma^* -\Sigma_0$ is negative semidefinite. By \Cref{proof-mGaussian-1}, we have $\Sigma^{*}_{\text{diag}} \geq \Lambda_{0, \text{diag}}^{-1}$. Then
\begin{align*}
    \|\Sigma^* -\Sigma_0\| &\leq \left\|\lambda \Sigma_0-\left[\frac{1}{\lambda + 1} \Lambda_0 + \frac{1}{\lambda(\lambda + 1)} \Lambda_0   \Lambda_{0, \text{diag}}^{-1} \Lambda_0 \right]^{-1} \right\|.
\end{align*}

Since $ \Sigma^{*}_{\text{diag}} \leq \Sigma_{0, \text{diag}}$, we obtain a lower bound with analogous techniques.
\begin{align*}
 \|\Sigma^* -\Sigma_0\| &\geq \left\|\lambda \Sigma_0-\left[\frac{1}{\lambda + 1} \Lambda_0 + \frac{1}{\lambda(\lambda + 1)} \Lambda_0    \Sigma_{0, \text{diag}} \Lambda_0 \right]^{-1} \right\|.
\end{align*}
This lower bound holds when the matrix on the right hand side is negative semidefinite.  To see that, we have
\begin{align*}
    \lambda \Sigma_0-\left[\frac{1}{\lambda + 1} \Lambda_0 + \frac{1}{\lambda(\lambda + 1)} \Lambda_0    \Sigma_{0, \text{diag}} \Lambda_0 \right]^{-1} &= -\Sigma_0  + \left[\frac{1}{\lambda + 1} \Lambda_0 + \frac{\lambda}{\lambda + 1}   \Sigma_{0, \text{diag}}^{-1}\right]^{-1},
\end{align*}
Since $\Lambda_{0, \text{diag}} \geq \Sigma_{0, \text{diag}}^{-1}$, we have
\begin{equation*}
    \left(-\Sigma_0  + \left[\frac{1}{\lambda + 1} \Lambda_0 + \frac{\lambda}{\lambda + 1}   \Sigma_{0, \text{diag}}^{-1}\right]^{-1}\right)_{\text{diag}} \leq   -\Sigma_{0, \text{diag}} +  \Sigma_{0, \text{diag}} = 0.
\end{equation*}
This completes the proof.
\end{proof}
Next, we provide the explicit formula for the $\Xi$-VI solution when the exact posterior is bivariate Gaussian.
\begin{proposition}
  \label{prop:mGaussian2}
  Let the exact posterior $\rmq_0^*$ be a bivariate Gaussian distribution with mean $\mu_0$ and precision matrix
  $\Lambda_0 = \begin{pmatrix}
    a_0 & b_0 \\
    b_0 & c_0
  \end{pmatrix}$.
  Then the $\Xi$-variational solution in \Cref{eqn-mGaussian-1} is a bivariate Gaussian distribution with mean $\mu_0$ and the following precision and covariance matrices:
  \begin{align*}
    \Lambda^* &= \begin{pmatrix}
      \frac{a_0}{2} + \sqrt{\frac{a_0^2}{4}-\frac{\lambda}{(\lambda + 1)^2} \frac{a_0 b_0^2}{c_0}} & \frac{1}{\lambda+1} b_0 \\
      \frac{1}{\lambda+1} b_0 &  \frac{c_0}{2} + \sqrt{\frac{c_0^2}{4}-\frac{\lambda}{(\lambda + 1)^2} \frac{c_0 b_0^2}{a_0}}
    \end{pmatrix}, \\
    \Sigma^* &= \frac{1}{a_0 c_0-b_0^2}\begin{pmatrix}
                 c_0 &-\frac{b_0}{\psi(\lambda)}\\
                 - \frac{b_0}{\psi(\lambda)}& a_0
               \end{pmatrix},
  \end{align*}
  where $\psi(\lambda) = \frac{1}{2} \left(\lambda + 1 +
    \sqrt{\left(\lambda -\frac{2 b_0^2-a_0 c_0}{a_0 c_0}\right)^2 +
      \frac{4b_0^2}{a_0^2 c_0^2}\left(a_0 c_0-b_0^2\right)} \right)$.
\end{proposition}
Compared to the exact covariance, the variational covariance matrix is adjusted by a factor depending on the regularizer $\lambda$. The adjusting function $\psi: [0, \infty) \mapsto [0, \infty)$ is strictly increasing. Thus, as $\lambda$ increases, we have element-wise strictly decreasing convergence to the mean-field covariance, i.e.
$ \lim_{\lambda \to \infty} \Sigma^* = \begin{pmatrix} a_0^{-1} & 0 \\
  0 & c_0^{-1}
  \end{pmatrix}$.
\begin{proof}[Proof of \Cref{prop:mGaussian2}]
Denote $\Lambda^* = \begin{pmatrix}
    a & b \\
    b & c
\end{pmatrix}$. The inverse is
\begin{equation*}
    \Sigma^* = \frac{1}{ac-b^2} \begin{pmatrix}
    c & -b \\
    -b & a
\end{pmatrix}.
\end{equation*}
As shown in \Cref{example:multivariate-Gaussian}, we have
\begin{equation*}
      \begin{pmatrix}
    a & b \\
    b & c
\end{pmatrix}  = \frac{1}{\lambda + 1} \begin{pmatrix}
    a_0 & b_0 \\
    b_0 & c_0
\end{pmatrix} + \frac{\lambda}{\lambda + 1}  \begin{pmatrix}
    a-\frac{b^2}{c} & 0 \\
    0 & d-\frac{b^2}{c}
\end{pmatrix}.
\end{equation*}
This implies
\begin{equation*}
\begin{pmatrix}
a + \lambda \frac{b^2}{c} & (\lambda +1)b \\
(\lambda +1)b & c + \lambda \frac{b^2}{a}
\end{pmatrix}  = \begin{pmatrix}
    a_0 & b_0 \\
    b_0 & c_0
\end{pmatrix}.
\end{equation*}
This translates to a system of equations
\begin{align*}
    (\lambda +1)b &= b_0 \\
    a + \lambda \frac{b^2}{c} &= a_0 \\
   c + \lambda \frac{b^2}{a} &= c_0.
\end{align*}
The first Equation yields $b = \frac{b_0}{\lambda + 1}$.  The other two equations yield
\begin{equation*}
    \frac{a_0 c}{c_0 a} = \frac{ac + b^2}{ac + b^2}  = 1 \implies c = \frac{c_0}{a_0} a .
\end{equation*}
Substituting $c$ gives us
\begin{equation*}
     a + \lambda \frac{a_0 b^2}{c_0 a} = a_0.
\end{equation*}
which yields $a =\frac{a_0}{2} \pm \sqrt{\frac{a_0^2}{4}-\frac{\lambda}{(\lambda + 1)^2} \frac{a_0 b_0^2}{c_0}}$.

Similarly, substituting $a$ with $c$ yields  $c =  \frac{c_0}{2} \pm  \sqrt{\frac{c_0^2}{4}-\frac{\lambda}{(\lambda + 1)^2} \frac{c_0 b_0^2}{a_0}}$.

Finally, use the fact that $\Lambda^* = \Lambda_0$ when $\lambda = 0$ to obtain the solution set
\begin{equation*}
   a =\frac{a_0}{2} + \sqrt{\frac{a_0^2}{4}-\frac{\lambda}{(\lambda + 1)^2} \frac{a_0 b_0^2}{c_0}}, \quad b = \frac{b_0}{\lambda + 1}, \quad c =  \frac{c_0}{2} + \sqrt{\frac{c_0^2}{4}-\frac{\lambda}{(\lambda + 1)^2} \frac{c_0 b_0^2}{a_0}}.
\end{equation*}
To obtain the covariance matrix, note that
\begin{equation*}
    |\Lambda^*| = ac-b^2 = \frac{a_0 c_0}{2} \left(1 + \sqrt{1-\frac{\lambda}{(\lambda +1)^2} \frac{4b_0^2}{a_0 c_0}} \right)-\frac{1}{\lambda + 1} b_0^2.
\end{equation*}
By the matrix inversion formula,
\begin{align*}
\Sigma^* = \frac{1}{|\Lambda^*|} \begin{pmatrix}
c & -b \\
-b& a
\end{pmatrix} &=  \begin{pmatrix}
\frac{\frac{c_0}{2} + \sqrt{\frac{c_0^2}{4}-\frac{\lambda}{(\lambda + 1)^2} \frac{c_0 b_0^2}{a_0}}}{\frac{a_0 c_0}{2} \left(1 + \sqrt{1-\frac{\lambda}{(\lambda +1)^2} \frac{4b_0^2}{a_0 c_0}} \right)-\frac{1}{\lambda + 1} b_0^2} & -\frac{b_0}{(\lambda + 1)\left( \frac{a_0 c_0}{2} \left(1 + \sqrt{1-\frac{\lambda}{(\lambda +1)^2} \frac{4b_0^2}{a_0 c_0}} \right)-\frac{1}{\lambda + 1} b_0^2\right)} \\
-\frac{b_0}{(\lambda + 1)\left( \frac{a_0 c_0}{2} \left(1 + \sqrt{1-\frac{\lambda}{(\lambda +1)^2} \frac{4b_0^2}{a_0 c_0}} \right)-\frac{1}{\lambda + 1} b_0^2\right)} &  \frac{\frac{a_0}{2} + \sqrt{\frac{a_0^2}{4}-\frac{\lambda}{(\lambda + 1)^2} \frac{a_0 b_0^2}{c_0}}}{\frac{a_0 c_0}{2} \left(1 + \sqrt{1-\frac{\lambda}{(\lambda +1)^2} \frac{4b_0^2}{a_0 c_0}} \right)-\frac{1}{\lambda + 1} b_0^2}  \\
\end{pmatrix} \\
&= \begin{pmatrix}
\frac{c_0}{a_0 c_0-\frac{2}{\lambda + 1 + \sqrt{(\lambda + 1)^2-\lambda \frac{4b_0^2}{a_0 c_0}}} b_0^2}  &-\frac{b_0}{\frac{a_0 c_0}{2}\left(\lambda + 1 + \sqrt{(\lambda +1)^2-\lambda \frac{4b_0^2}{a_0 c_0}}\right)-b_0^2} \\
- \frac{b_0}{\frac{a_0 c_0}{2}\left(\lambda + 1 + \sqrt{(\lambda +1)^2-\lambda \frac{4b_0^2}{a_0 c_0}}\right)-b_0^2} &  \frac{a_0}{a_0 c_0-\frac{2}{\lambda + 1 + \sqrt{(\lambda + 1)^2-\lambda \frac{4b_0^2}{a_0 c_0}}} b_0^2}
\end{pmatrix} \\
&= \frac{1}{a_0 c_0-\psi^{-1}(\lambda) b_0^2}\begin{pmatrix} c_0 &-\frac{b_0}{\psi(\lambda)}\\
- \frac{b_0}{\psi(\lambda)}& a_0
\end{pmatrix}.
\end{align*}
where $\psi(\lambda) = \frac{1}{2} \left(\lambda + 1 + \sqrt{\left(\lambda -\frac{2 b_0^2-a_0 c_0}{a_0 c_0}\right)^2 + \frac{4b_0^2}{a_0^2 c_0^2}\left(a_0 c_0-b_0^2\right)} \right)$.
\end{proof}
\begin{proof}[Proof of \Cref{cor:stability-1}]
We use the well known property of the $2$-Wasserstein metric that for any $\rmq_0, \rmq_1 \in \P_2(\Theta)$,
\begin{equation}\label{proof:cor:stability-1-eqn-1}
    W_2^2(\prod_{i = 1}^D \rmq_{0, i},\prod_{i = 1}^D \rmq_{1, i} )  \leq \sum_{i =1}^D W_2^2 ( \rmq_{0, i}, \rmq_{1,i}) \leq W_2^2 (\rmq_0, \rmq_1).
\end{equation}
By \Cref{thm:mf-convergence}, $W_2 (\rmq^*_\lambda, \rmq^*_\infty) \to 0$ as $\lambda \to \infty$, hence $W_2(m^*_\lambda, \rmq^*_\infty) \to 0$ by the identity \Cref{proof:cor:stability-1-eqn-1} which implies $W_2(\rmq^{*^{(\infty)}}_\lambda, \rmq^*_\lambda) \to 0$ by \Cref{thm:stability}. An analogous derivation holds for $\lambda \to 0$.
\end{proof}

\section{Additional Simulation Results} \label{sect-additional-simulation}
\begin{table}
\centering
\begin{tabular}{c c c c c}
\toprule
$N$
& $\frac{\|\E_{\hat{\rmq}_{10,N}}(\theta)-\E_{\rmq_0^*}(\theta)\|_2}{\|\E_{\rmq_0^*}(\theta)\|_2}$
& $\frac{\|\mathrm{Cov}_{\hat{\rmq}_{10,N}}(\theta)-\mathrm{Cov}_{\rmq_0^*}(\theta)\|_F}{\|\mathrm{Cov}_{\rmq_0^*}(\theta)\|_F}$
& $W_2(\hat{\rmq}_{10,N},\,\rmq_0^*)$
& Runtime (sec)
\\
\midrule
5  & 0.069 {[0.059,\,0.078]}
   & 0.813 {[0.767,\,0.884]}
   & 0.321 {[0.307,\,0.336]}
   & 0.017 {[0.017,\,0.017]}
\\
10 & 0.049 {[0.041,\,0.052]}
   & 0.672 {[0.633,\,0.724]}
   & 0.243 {[0.231,\,0.263]}
   & 0.222 {[0.220,\,0.225]}
\\
15 & 0.040 {[0.035,\,0.044]}
   & 0.639 {[0.611,\,0.664]}
   & 0.221 {[0.215,\,0.231]}
   & 1.076 {[1.074,\,1.081]}
\\
20 & 0.031 {[0.029,\,0.035]}
   & 0.588 {[0.570,\,0.605]}
   & 0.196 {[0.191,\,0.204]}
   & 3.458 {[3.428,\,3.494]}
\\
25 & 0.032 {[0.027,\,0.035]}
   & 0.583 {[0.559,\,0.604]}
   & 0.192 {[0.184,\,0.201]}
   & 8.528 {[8.475,\,8.602]}
\\
30 & 0.026 {[0.024,\,0.032]}
   & 0.582 {[0.538,\,0.596]}
   & 0.188 {[0.176,\,0.195]}
   & 17.848 {[17.723,\,18.021]}
\\
35 & 0.026 {[0.024,\,0.028]}
   & 0.545 {[0.538,\,0.565]}
   & 0.178 {[0.173,\,0.182]}
   & 32.286 {[32.016,\,32.531]}
\\
\bottomrule
\end{tabular}

\caption{Median and interquantile range of the approximation errors of the $\Xi$-VI posterior $\hat{\rmq}_{10,N}$ for varying Monte Carlo support sizes $N$ in Laplace linear regression.
For each $N$, we run \Cref{alg-approx} with $\lambda=10$ and compare $\hat{\rmq}_{10,N}$ to the exact posterior $\rmq_0^*$ obtained via MCMC.
Across $20$ experimental replicates, we report the median and interquantile range of the relative mean error, relative covariance error, the $W_2$ distance, and the runtime.
Although larger $N$ improves posterior accuracy, the computational cost increases sharply and the accuracy gains beyond $N \approx 20$ are modest, thus $N=20$ is a reasonable choice for accurate and scalable inference in this problem.}
\label{tab:xi-vi-sensitivity}
\end{table}

\begin{figure}[t]
  \centering
  \includegraphics[width=0.8\textwidth]{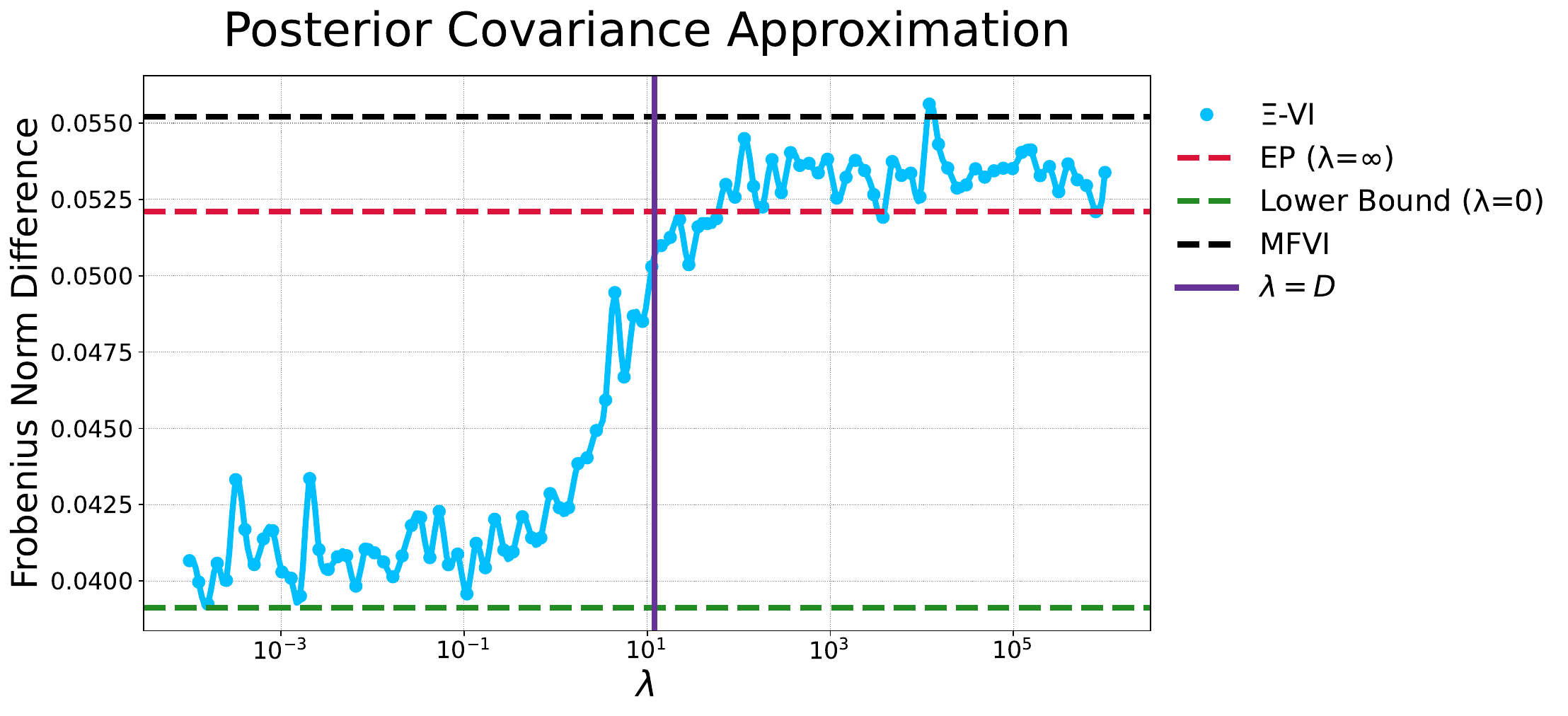}
  \caption{Approximation errors for posterior covariance for Laplace linear regression. The experiment implements \Cref{alg-approx} with expectation propagation as the first step. Errors are quantified using the Frobenius norm and contrasted across a spectrum of $\lambda$ values, including the theoretical lower bound at \(\lambda = 0\) and the diagonal EP approximation at \(\lambda = \infty\). The vertical line marks the regularization parameter \(\lambda = D\).}
\label{fig:laplace_cov1}
\end{figure}

\newpage
\section{Additional Analysis of the 8-Schools Model} \label{sect-additional-school}
In this section, we present additional results for the eight school example in \Cref{example:Eight-school}.

 Here the goal is to compare treatment effects between schools. We compute the posterior credible intervals for the differences in treatment effects $\theta_i - \theta_j$ between schools $i$ and $j$. \Cref{tab:credible_intervals_random_pairs} shows credible intervals for $\theta_i - \theta_j$ across ten randomly chosen school pairs, calculated under the exact posterior, MFVI, and $\Xi$-VI with $\lambda \in \{0, 1, 10, 1000\}$. The results show $\Xi$-VI, especially with lower $\lambda$ values, yields intervals that more accurately reflect those derived from the exact posterior, while MFVI produces the most inaccurate intervals.

\begin{table}[t]
\centering
\caption{$95\%$ posterior credible intervals for \(\theta_i-\theta_j\) for 10 randomly selected pairs of schools.}
\label{tab:credible_intervals_random_pairs}
\begin{tabular}{|c|c|c|c|c|c|}
\hline
\textbf{Method} & $\theta_2-\theta_5$ & $\theta_6-\theta_7$ & $\theta_2-\theta_4$ & $\theta_4-\theta_8$ & $\theta_1-\theta_2$ \\
\hline
MFVI & [-12.08, 18.48] & [-20.12, 11.37] & [-13.99, 15.79] & [-15.37, 14.97] & [-13.47, 17.10] \\
True & [-8.50, 14.90] & [-17.74, 7.30] & [-11.28, 12.40] & [-13.02, 12.52] & [-9.21, 16.55] \\
$\lambda=0$ & [-7.81, 13.77] & [-16.17, 6.49] & [-10.42, 12.06] & [-11.19, 12.70] & [-8.12, 15.81] \\
$\lambda = 1$ & [-9.02, 13.54] & [-16.71, 7.44] & [-11.57, 11.87] & [-11.73, 13.33] & [-9.09, 15.97] \\
$\lambda = 10$ & [-10.43, 13.01] & [-15.77, 8.75] & [-12.57, 12.82] & [-12.21, 14.50] & [-10.54, 15.41] \\
$\lambda = 1000$ & [-10.51, 12.93] & [-15.86, 9.53] & [-12.55, 13.19] & [-12.71, 14.47] & [-10.79, 15.64] \\
\hline
\end{tabular}

\begin{tabular}{|c|c|c|c|c|c|}
\hline
\textbf{Method} & $\theta_2-\theta_8$ & $\theta_3-\theta_8$ & $\theta_5-\theta_6$ & $\theta_2-\theta_7$ & $\theta_3-\theta_4$ \\
\hline
MFVI & [-14.83, 17.07] & [-17.71, 13.88] & [-16.77, 14.19] & [-17.30, 13.41] & [-17.62, 13.49] \\
True & [-12.09, 12.73] & [-16.08, 11.05] & [-12.62, 10.31] & [-14.97, 9.20] & [-14.79, 10.45] \\
$\lambda=0$ & [-10.15, 13.41] & [-14.54, 11.56] & [-11.12, 9.70] & [-13.66, 8.31] & [-14.33, 10.41] \\
$\lambda = 1$ & [-11.26, 13.57] & [-14.31, 12.75] & [-11.49, 10.27] & [-15.28, 8.95] & [-14.95, 11.25] \\
$\lambda = 10$ & [-11.58, 14.47] & [-14.69, 13.74] & [-11.15, 10.96] & [-15.09, 10.55] & [-15.18, 12.35] \\
$\lambda = 1000$ & [-12.14, 14.34] & [-14.48, 14.32] & [-12.00, 11.09] & [-14.75, 10.63] & [-15.21, 13.12] \\
\hline
\end{tabular}
\end{table}

\end{document}